\documentclass[letterpaper]{article}
\usepackage[margin=1in,dvips]{geometry}
\usepackage{microtype}
\usepackage{graphicx,psfrag,amsmath,amsthm,amssymb}
\usepackage[numbers,sort&compress]{natbib}
\usepackage{url,color,booktabs}
\usepackage{algorithm}
\usepackage{multirow}
\usepackage[misc]{ifsym}
\usepackage{enumitem}
\usepackage{appendix}
\usepackage[bookmarksnumbered=true,bookmarksopen=true]{hyperref}

\definecolor{DarkRed}{rgb}{0.545,0,0}
\hypersetup{
  breaklinks=true, colorlinks=true, linkcolor=black,
  citecolor=black, urlcolor=DarkRed, 
  pdftitle={Model compression as constrained optimization, with application to neural nets. Part V: combining compressions},
  pdfauthor={Miguel {\'A}.\ Carreira-Perpi{\~n}{\'a}n and Yerlan Idelbayev}
}

\def \ifempty#1{\def\temp{#1} \ifx\temp\empty }

\newcommand{\U}{\ensuremath{\mathbf{U}}}
\newcommand{\V}{\ensuremath{\mathbf{V}}}
\newcommand{\W}{\ensuremath{\mathbf{W}}}

\renewcommand{\b}{\ensuremath{\mathbf{b}}}

\newcommand{\w}{\ensuremath{\mathbf{w}}}
\newcommand{\x}{\ensuremath{\mathbf{x}}}
\newcommand{\y}{\ensuremath{\mathbf{y}}}
\newcommand{\z}{\ensuremath{\mathbf{z}}}
\newcommand{\0}{\ensuremath{\mathbf{0}}}
\newcommand{\1}{\ensuremath{\mathbf{1}}}

\newcommand{\blambda}{\ensuremath{\boldsymbol{\lambda}}}

\newcommand{\btheta}{\ensuremath{\boldsymbol{\theta}}}

\newcommand{\bDelta}{\ensuremath{\boldsymbol{\Delta}}}

\newcommand{\bbR}{\ensuremath{\mathbb{R}}}

\newcommand{\calC}{\ensuremath{\mathcal{C}}}

\newcommand{\calN}{\ensuremath{\mathcal{N}}}
\newcommand{\calO}{\ensuremath{\mathcal{O}}}

\newcommand{\ceil}[1]{\lceil#1\rceil}

\newcommand{\norm}[2][]{%
  \ifempty{#1} {\left\lVert#2\right\rVert} \else {#1\lVert#2#1\rVert} \fi}

\newcommand{\caja}[4][1]{{%
    \renewcommand{\arraystretch}{#1}%
    \begin{tabular}[#2]{@{}#3@{}}%
      #4%
    \end{tabular}%
    }}

{%
\begin{list}{#1}{
\vspace{-\topsep}
\vspace{-\partopsep}
\setlength{\itemindent}{0cm}
\setlength{\rightmargin}{0cm}
\setlength{\listparindent}{0cm}
\settowidth{\labelwidth}{#1}
\setlength{\leftmargin}{\labelwidth}
\addtolength{\leftmargin}{\labelsep}
\setlength{\itemsep}{0cm}
}%
}%
{%
\end{list}
\vspace{-\topsep}
\vspace{-\partopsep}
}

{\begin{enumerate}%
}%
{\end{enumerate}}

\DeclareMathOperator*{\argmin}{arg\,min}

\newcommand{\rankop}{\operatorname{rank}}
\newcommand{\rank}[1]{\ensuremath{\rankop\left(#1\right)}}

\theoremstyle{plain}
\newtheorem{thm}{Theorem}[section]

\newtheorem*{lemma*}{Lemma}

\newtheorem*{prop*}{Proposition}

\theoremstyle{definition}

\newtheorem*{defn*}{Definition}

\newtheorem*{exmp*}{Example}

\newtheorem*{conj*}{Conjecture}

\theoremstyle{remark}

\newtheorem*{rmk*}{Remark}

\DeclareMathOperator{\bits}{\texttt{bits}}
\DeclareMathOperator{\layers}{\texttt{layers}}
\DeclareMathOperator{\params}{\texttt{weights}}

\newcommand{\na}{\text{--}}
\newcommand\num{\scalebox{0.8}{\raisebox{0.4ex}{\#}}}

\graphicspath{{grf/}}

\AtBeginDvi{}

\title{Model compression as constrained optimization, \\ with application to neural nets. \\Part V: combining compressions}
\author{
  Miguel {\'A}.\ Carreira-Perpi{\~n}{\'a}n \hspace{5ex} Yerlan Idelbayev\\
  Dept.\ of Computer Science \& Engineering, University of California, Merced \\
  {\url{http://eecs.ucmerced.edu}} \\
  {\url{https://github.com/UCMerced-ML/LC-model-compression}}
}

\date{July 9, 2021}

\begin{document}

\maketitle

\begin{abstract}
Model compression is generally performed by using quantization, low-rank approximation or pruning, for which various algorithms have been researched in recent years. One fundamental question is: what types of compression work better for a given model? Or even better: can we improve by combining compressions in a suitable way? We formulate this generally as a problem of optimizing the loss but where the weights are constrained to equal an additive combination of separately compressed parts; and we give an algorithm to learn the corresponding parts' parameters. Experimentally with deep neural nets, we observe that 1) we can find significantly better models in the error-compression space, indicating that different compression types have complementary benefits, and 2) the best type of combination depends exquisitely on the type of neural net. For example, we can compress ResNets and AlexNet using only 1 bit per weight without error degradation at the cost of adding a few floating point weights. However, VGG nets can be better compressed by combining low-rank with a few floating point weights.

\end{abstract}

\begin{figure*}[t]
  \centering
  \begin{tabular}{@{}c@{\hspace{0.05\linewidth}}c@{\hspace{0.05\linewidth}}c@{}}
    reference & low-rank (L) & pruning (P) \\
    \psfrag{W}[][]{\W}
    \includegraphics[width=0.30\linewidth]{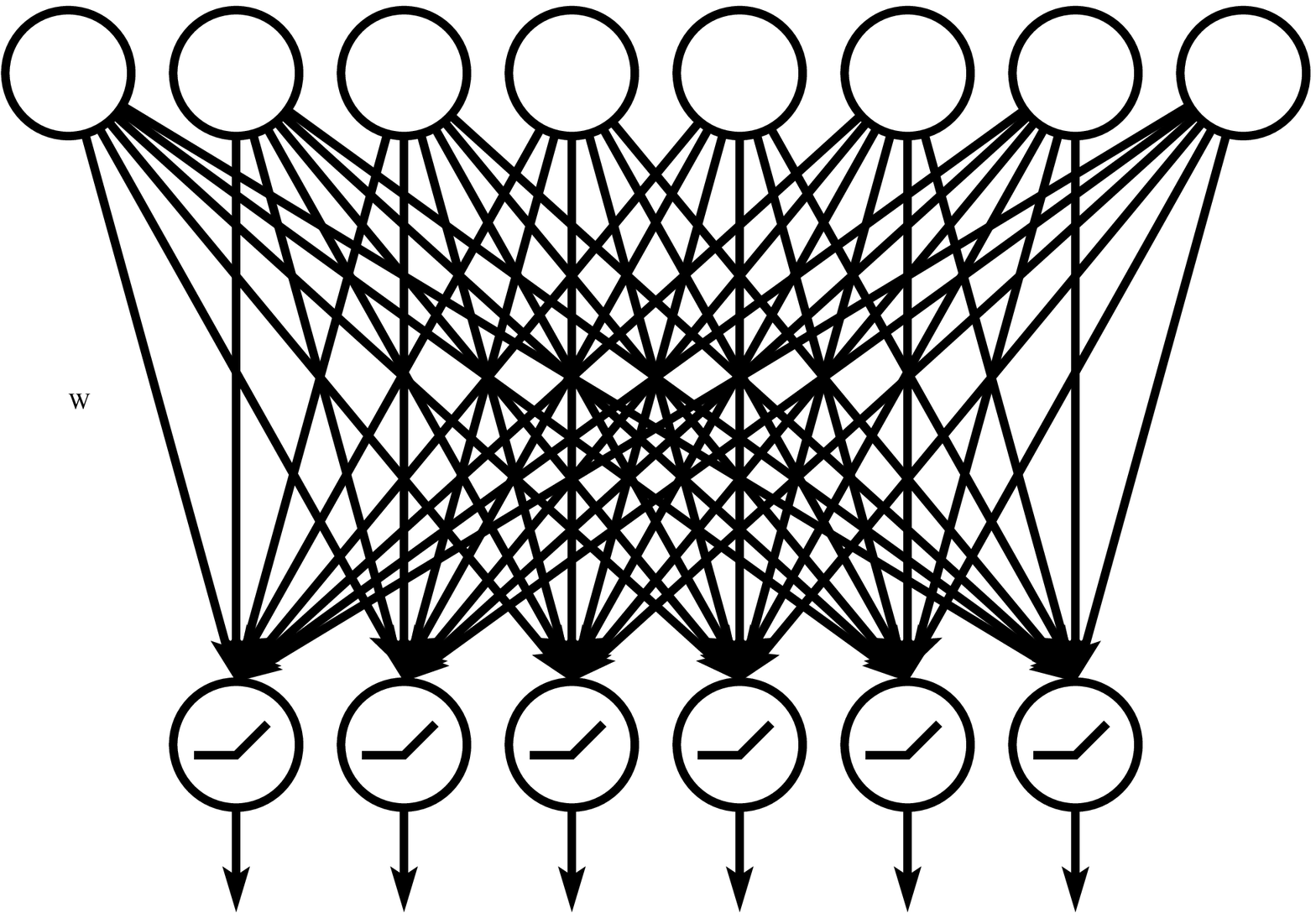} &
    \psfrag{W}[l][r]{$\W_1$}
    \includegraphics[width=0.30\linewidth]{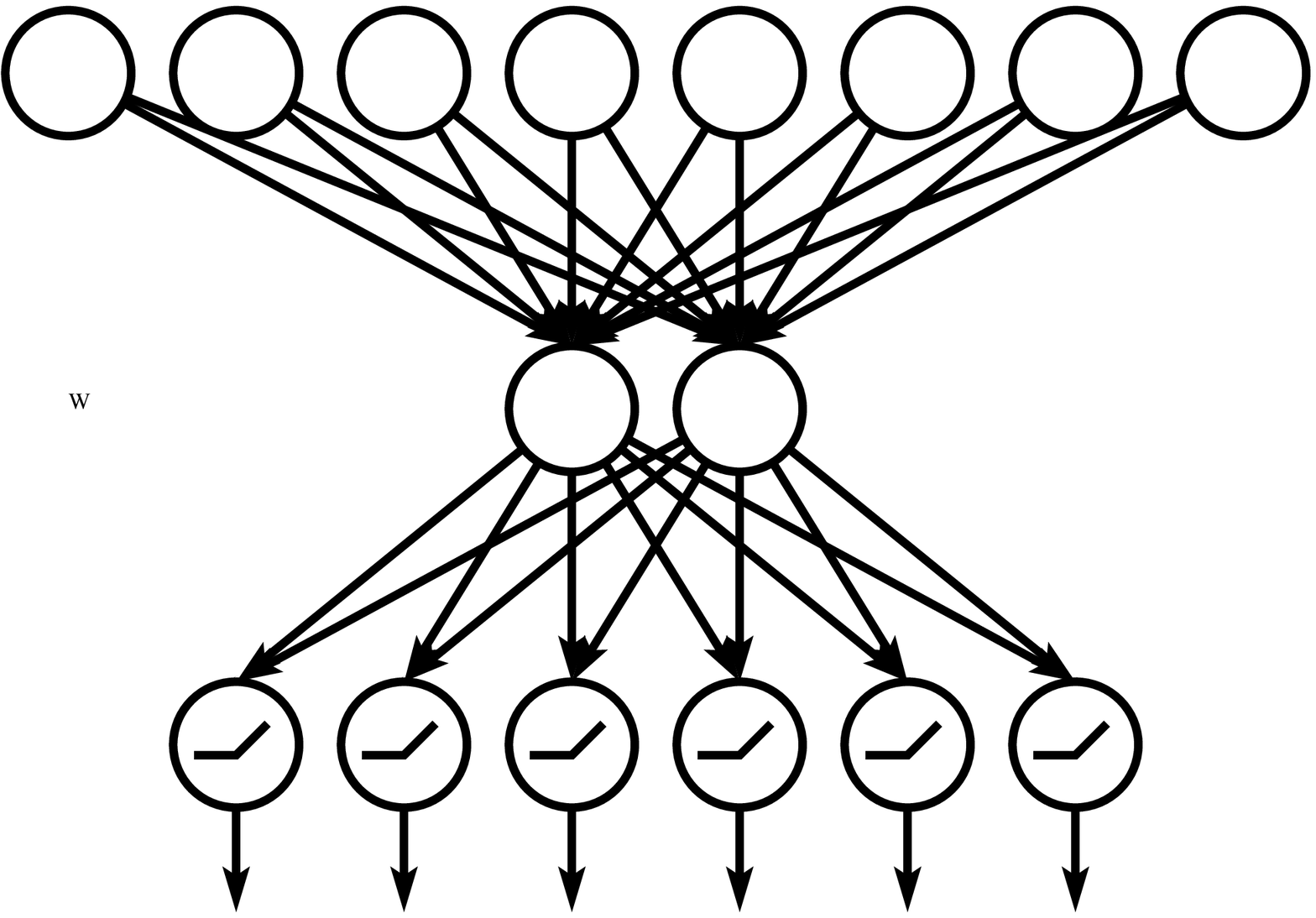} &
    \psfrag{W}[r][r]{$\W_2$\hspace{-0.8ex}}
    \includegraphics[width=0.30\linewidth]{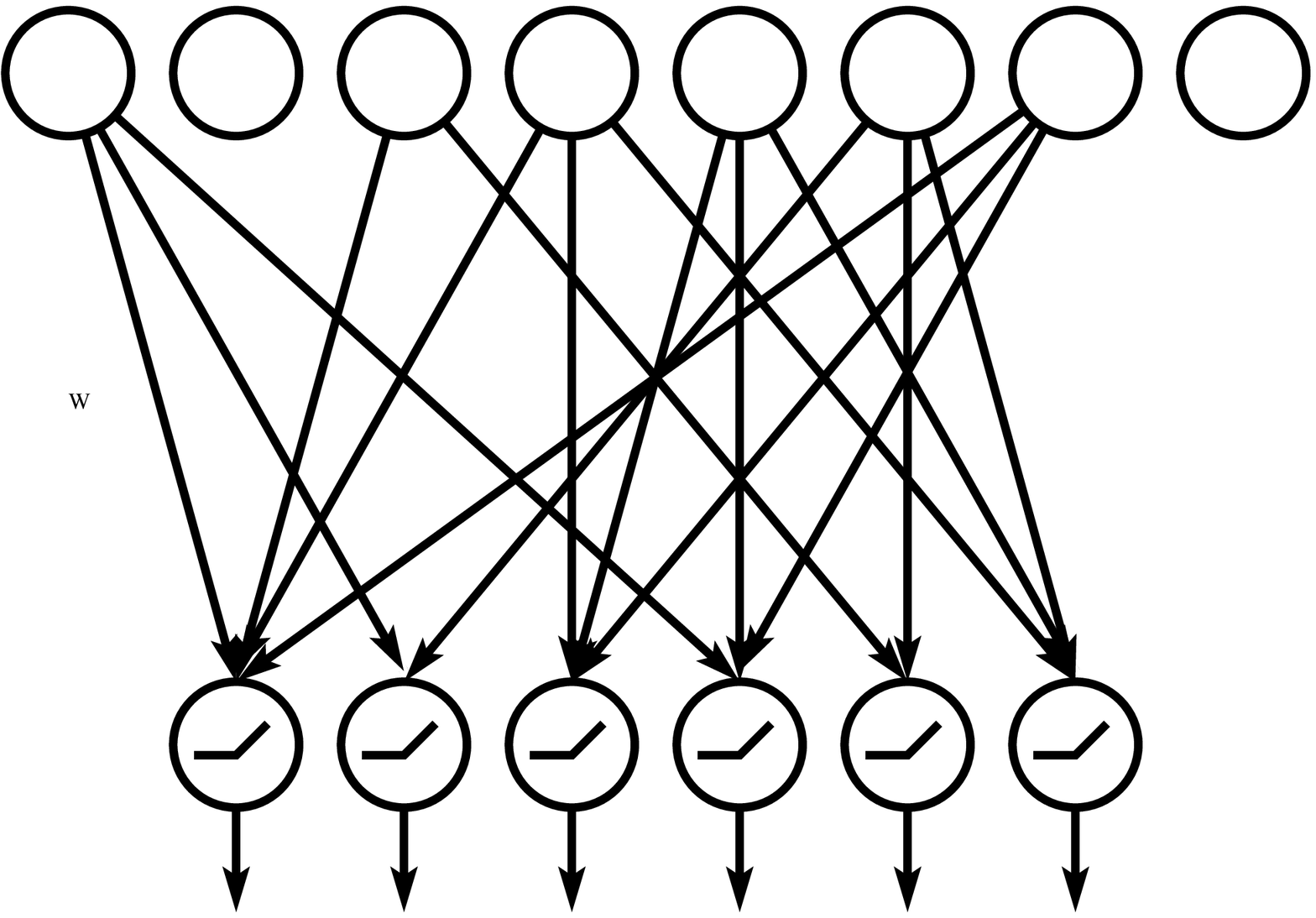} \\[2ex]
    binarization (B) & L+B & \makebox[0pt][c]{P + B} \\
    \psfrag{W}[r][r]{$\W_3$\hspace{-0.8ex}}
    \includegraphics[width=0.30\linewidth]{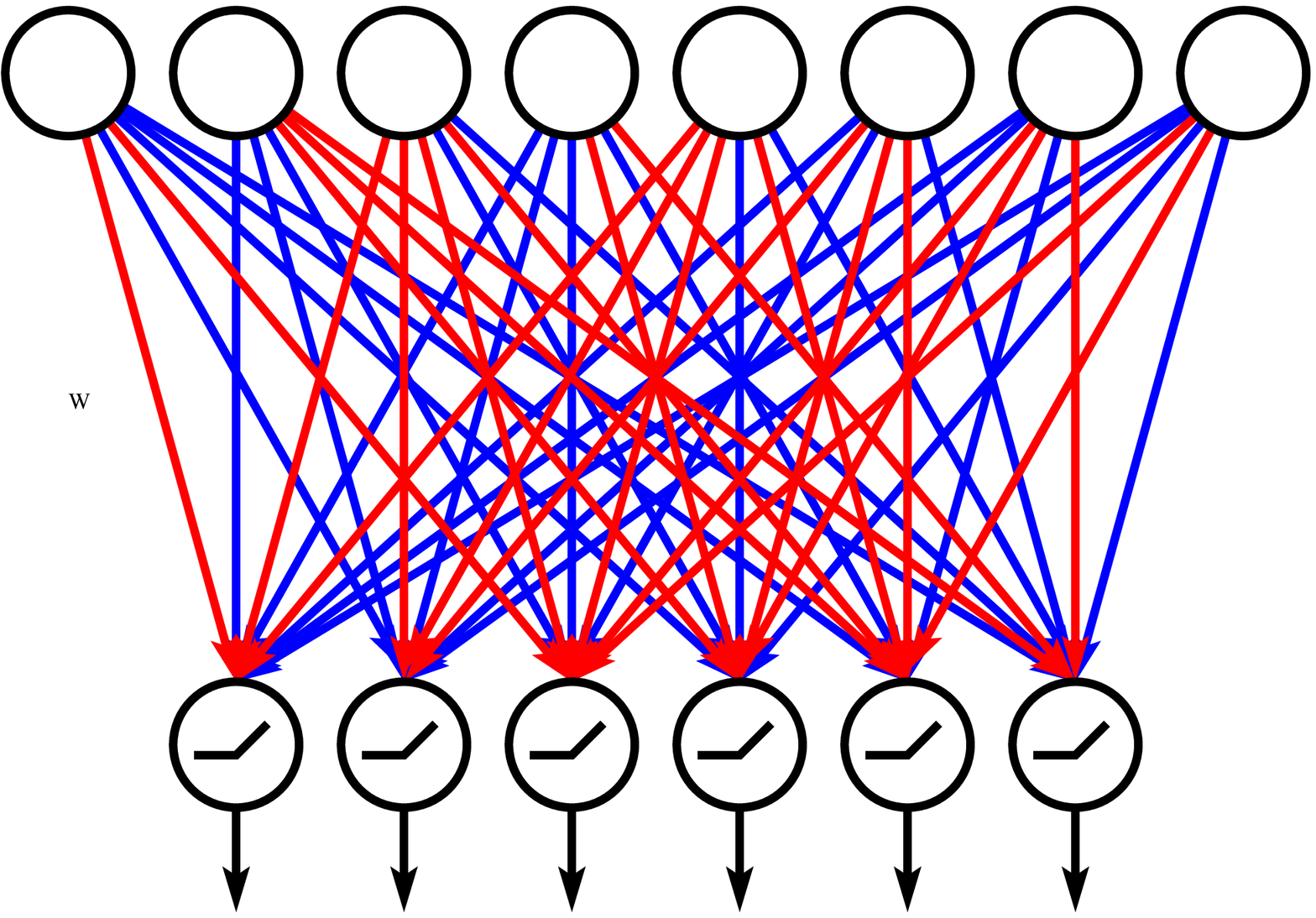} &
    \psfrag{W}[r][r]{\caja[0.7]{c}{c}{\\ $\W_1$ \\ $+$ \\ $\W_2$}\hspace{-0.8ex}}
    \includegraphics[width=0.30\linewidth]{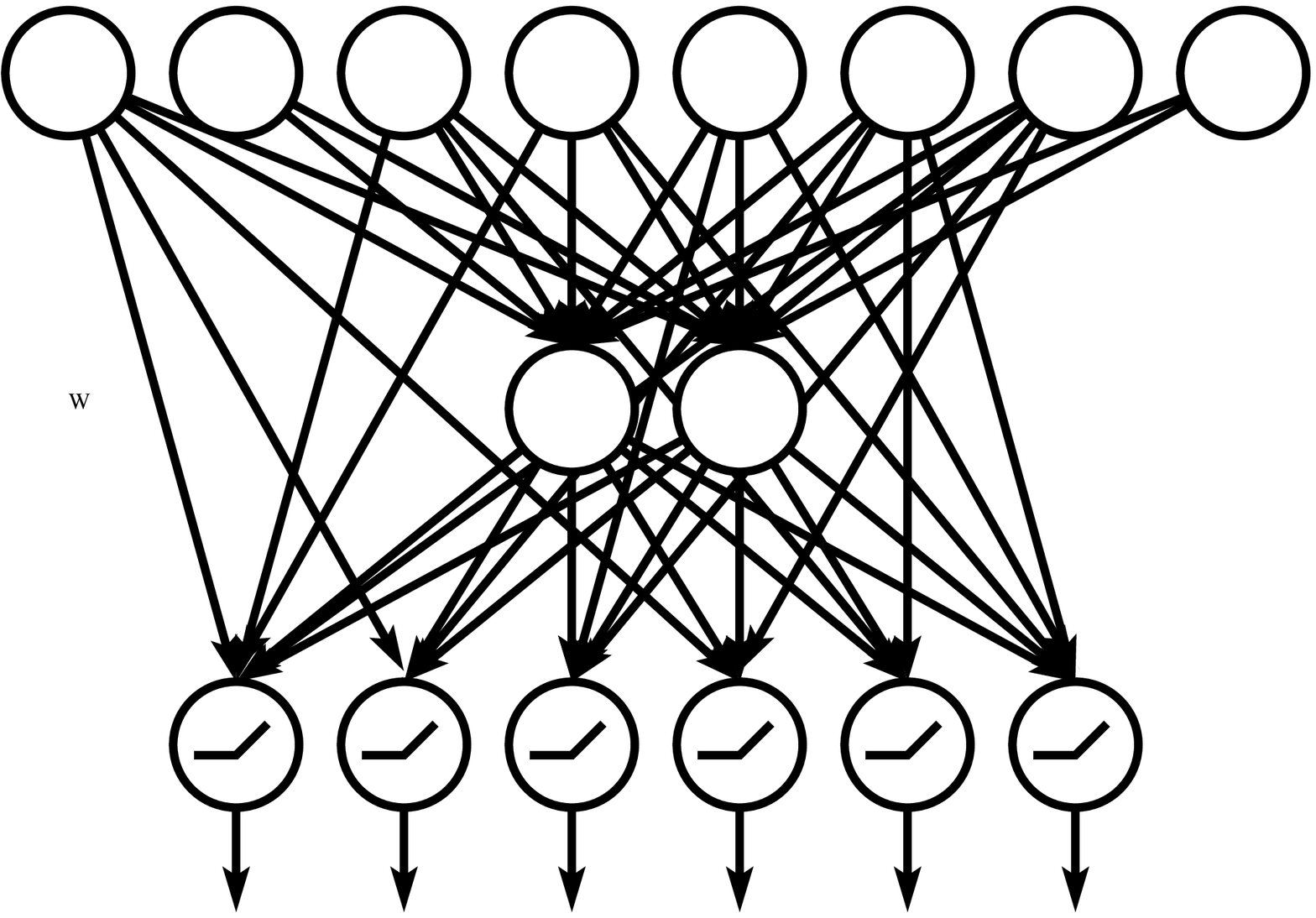} &
    \psfrag{W}[r][r]{\caja[0.7]{c}{c}{\\ $\W_2$ \\ $+$ \\ $\W_3$}\hspace{-0.8ex}}
    \includegraphics[width=0.30\linewidth]{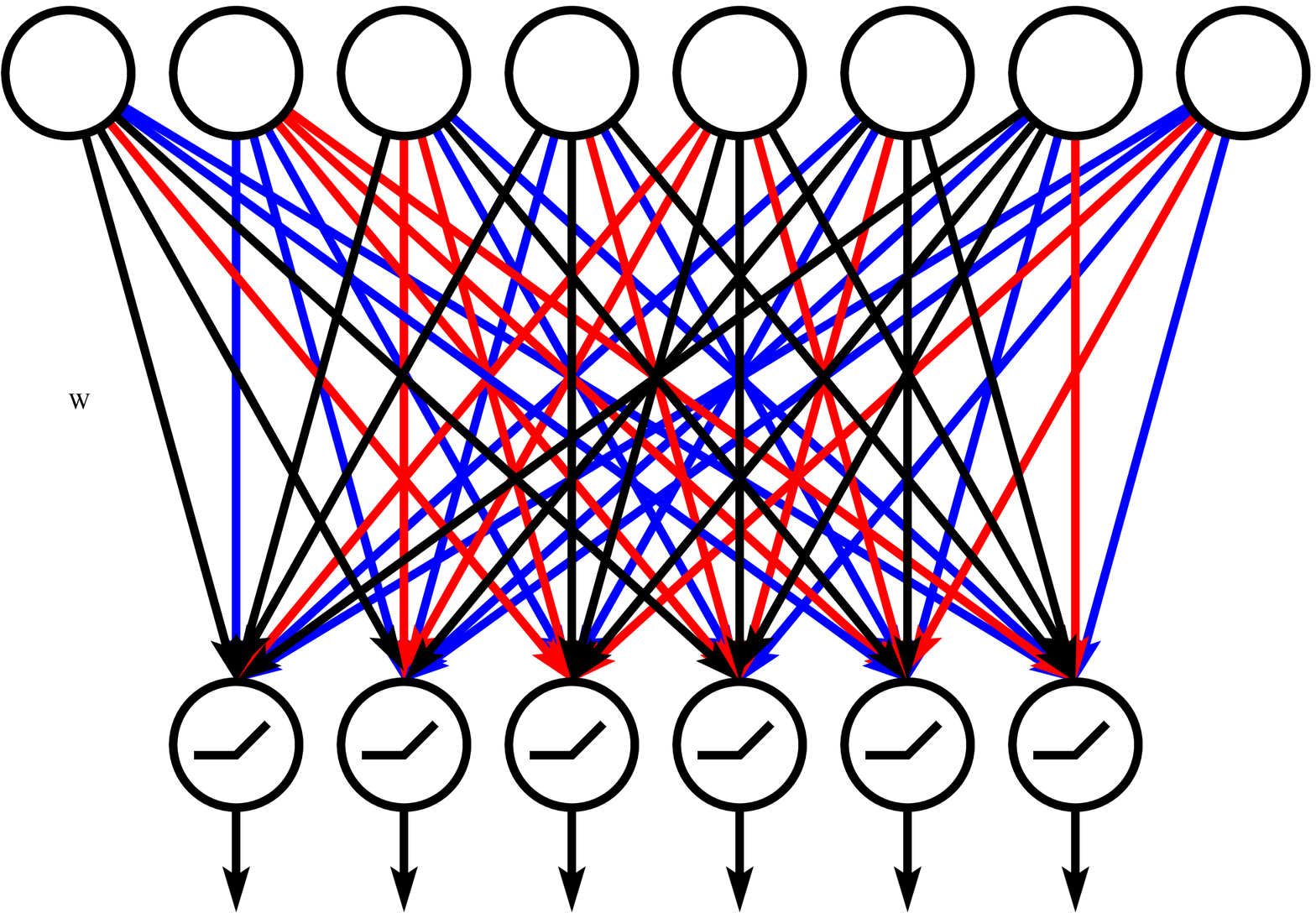}
  \end{tabular}
  \caption{Illustration of compression by additive combination $\W = \W_1 + \W_2 + \W_3$. Black weights are real, red weights are $-1$ and blue weights are $+1$.}
  \label{f:arch}
\end{figure*}

\section{Introduction}
\label{s:intro}

In machine learning, model compression is the problem of taking a neural net or some other model, which has been trained to perform (near)-optimal prediction in a given task and dataset, and transforming it into a model that is smaller (in size, runtime, energy or other factors) while maintaining as good a prediction performance as possible. This problem has recently become important and actively researched because of the large size of state-of-the-art neural nets, trained on large-scale GPU clusters without constraints on computational resources, but which cannot be directly deployed in IoT devices with much more limited capabilities.

The last few years have seen much work on the topic, mostly focusing on specific forms of compression, such as quantization, low-rank matrix approximation and weight pruning, as well as variations of these. These papers typically propose a specific compression technique and a specific algorithm to compress a neural net with it. The performance of these techniques individually varies considerably from case to case, depending on the algorithm (some are better than others) but more importantly on the compression technique. This is to be expected, because (just as happens with image or audio compression) some techniques achieve more compression for certain types of signals.

A basic issue is the representation ability of the compression: given an optimal point in model space (the weight parameters for a neural net), which manifold or subset of this space can be compressed exactly, and is that subset likely to be close to the optimal model for a given machine learning task? For example, for low-rank compression the subset contains all matrices of a given rank or less. Is that a good subset to model weight matrices arising from, say, deep neural nets for object classification?

One way to understand this question is to try many techniques in a given task and gain experience about what works in which case. This is made difficult by the multiplicity of existing algorithms, the heuristics often used to optimize the results experimentally (which are compounded by the engineering aspects involved in training deep nets, to start with), and the lack at present of an apples-to-apples evaluation in the field of model compression.

A second way to approach the question which partly sidesteps this problem is to use a common algorithm that can handle any compression technique. While compressing a deep net in a large dataset will still involve careful selection of optimization parameters (such as SGD learning rates), having a common algorithmic framework should put different compression techniques in the same footing. Yet a third approach, which we propose in this paper, is to \emph{combine} several techniques (rather than try each in isolation) while jointly optimizing over the parameters of each (codebook and assignments for quantization, component matrices for low-rank, subset and value of nonzero weights for pruning, etc.).

There are multiple ways to define a combination of compression techniques. One that is simple to achieve is by applying compression techniques sequentially, such as first pruning the weights, then quantizing the remaining nonzero weights and finally encoding them with Huffman codes \citep{Han_16a}. This is suboptimal in that the global problem is solved greedily, one compression at a time. The way we propose here is very different: an \emph{additive combination} of compression techniques. For example, we may want to compress a given weight matrix \W\ as the sum (or linear combination) $\W = \W_1 + \W_2 + \W_3$ of a low-rank matrix $\W_1$, a sparse matrix $\W_2$ and a quantized matrix $\W_3$. This introduces several important advantages. First, it contains as a particular case each technique in isolation (e.g., quantization by making $\W_1 = \0$ a zero-rank matrix and $\W_2 = \0$ a matrix with no nonzeros). Second, and critically, it allows techniques to help each other because of having complementary strengths. For example, pruning can be seen as adding a few elementwise real-valued corrections to a quantized or low-rank weight matrix. This could result (and does in some cases) in using fewer bits, lower rank and fewer nonzeros and a resulting higher compression ratio (in memory or runtime). Third, the additive combination vastly enlarges the subset of parameter space that can be compressed without loss compared to the individual compressions. This can be seen intuitively by noting that a fixed vector times a scalar generates a 1D space, but the additive combination of two such vectors generates a 2D space rather than two 1D spaces).

One more thing remains to make this possible: a formulation and corresponding algorithm of the compression problem that can handle such additive combinations of arbitrary compression techniques. We rely on the previously proposed ``learning-compression (LC)'' algorithm \citep{Carreir17a}. This explicitly defines the model weights as a function (called \emph{decompression mapping}) of low-dimensional compression parameters; for example, the low-rank matrix above would be written as $\W_1 = \U \V^T$. It then iteratively optimizes the loss but constraining the weights to take the desired form (an additive combination in our case). This alternates \emph{learning (L)} steps that train a regularized loss over the original model weights with \emph{compression (C)} steps that compress the current weights, in our case according to the additive compression form.

Next, we review related work (section~\ref{s:related}), describe our problem formulation (section~\ref{s:MCCO}) and corresponding LC algorithm (section~\ref{s:LC}), and demonstrate the claimed advantages with deep neural nets (sections \ref{s:expts}, \ref{s:expts_imagenet}). A shorter version of this paper appears as \cite{IdelbayCarreir21e}.

\section{Related work}
\label{s:related}

\subsection{General approaches}

In the literature of model and particularly neural net compression, various approaches have been studied, including most prominently weight quantization, weight pruning and low-rank matrix or tensor decompositions. There are other approaches as well, which can be potentially used in combination with. We briefly discuss the individual techniques first. Quantization is a process of representing each weight with an item from a codebook. This can be achieved through fixed codebook schemes, i.e., with predetermined codebook values that are not learned (where only the assignments should be optimized). Examples of this compression are binarization, ternarization, low-precision, fixed-point or other number representations \citep{Rasteg_16a,Li_16b}. Quantization can also be achieved through adaptive codebook schemes, where the codebook values are learned together with the assignment variables, with algorithms based on soft quantization \citep{NowlanHinton92a,Agusts_17a} or hard quantization \citep{Han_16a}. Pruning is a process of removal of weights (unstructured) or filters and neurons (structured). It can be achieved by salience ranking \citep{Lecun_90a,Han_16a} in one go or over multiple refinements, or by using sparsifying norms \citep{CarreirIdelbay18a,Ye_18a}. Low-rank approximation is a process of replacing weights with low-rank \citep{Wen_17a,Kim_19a,Xu_20a,IdelbayCarreir21b} or tensors-decomposed versions \citep{Novikov_15a}.

\subsection{Usage of combinations}
One of the most used combinations is to apply compressions sequentially, most notably first to prune weights and then to quantize the remaining ones \citep{Han_16a,Han_16a,Choi_17a,TungMori18a,Yang_20b}, which may possibly be further compressed via lossless coding algorithms (e.g., Huffman coding). Additive combination of quantizations \citep{BabenkLempit14a,Zhou_17a,Xu_18b}, where weights are the sum of quantized values, as well as low-rank + sparse combination \cite{AlvarezSalzman17a, Yu_17a} has been used to compress neural networks. However, these methods rely on optimization algorithms highly specialized to a problem, limiting its application to new combinations (e.g., quantization + low-rank).

\section{Compression via an additive combination as constrained optimization}
\label{s:MCCO}

Our basic formulation is that we define the weights as an additive combination of weights, where each term in the sum is individually compressed in some way. Consider for simplicity the case of adding just two compressions for a given matrix%
\footnote{The general case of multiple compressions, possibly applied separately to each layer of a net and not necessarily in matrix form, follows in an obvious way. Throughout the paper we use \W\ or \w\ or $w$ to notate matrix, vector or scalar weights as appropriate (e.g., \W\ is more appropriate for low-rank decomposition).}.
We then write a matrix of weights as $\W = \bDelta_1(\btheta_1) + \bDelta_2(\btheta_2)$, where $\btheta_i$ is the low-dimensional parameters of the $i$th compression and $\bDelta_i$ is the corresponding decompression mapping. Formally, the $\bDelta$ maps a compressed representation of the weight matrix \btheta{} to the real-valued, uncompressed weight matrix \W. Its intent is to represent the space of matrices that can be compressed via a constraint subject to which we optimize the loss of the model in the desired task (e.g., classification). That is, a constraint $\W = \bDelta(\btheta)$ defines a feasible set of compressed models. For example:
\begin{itemize}
\item Low-rank: $\W = \U \V^T$ with \U\ and \V\ of rank $r$, so $\btheta = (\U, \V)$.
\item Pruning: $\w = \btheta$ s.t.\ $\norm{\btheta}_0 \le \kappa$, so \btheta\ is the indices of its nonzeros and their values.
\item Scalar quantization: $w = \sum^K_{k=1}{z_k c_k}$ with assignment variables $\z \in \{0,1\}^K$, $\1^T\z = 1$ and codebook $\calC = \{c_1,\dots,c_K\} \subset \bbR$, so $\btheta = (\z, \calC)$.
\item Binarization: $w \in \{-1,+1\}$ or equivalently a scalar quantization with $\calC = \{-1,+1\}$.
\end{itemize}
Note how the mapping $\bDelta(\btheta)$ and the low-dimensional parameters \btheta\ can take many forms (involving scalars, matrices or other objects of continuous or discrete type) and include constraints on \btheta. Then, our problem formulation takes the form of \emph{model compression as constrained optimization} \citep{Carreir17a} and given as:
\begin{equation}
  \label{e:MCCO}
  \min_{\w}{ L(\w) } \quad \text{s.t.} \quad \w = \bDelta_1(\btheta_1) + \bDelta_2(\btheta_2).
\end{equation}
This expresses in a mathematical way our desire that 1) we want a model with minimum loss on the task at hand ($L(\w)$ represents, say, the cross-entropy of the original deep net architecture on a training set); 2) the model parameters \w\ must take a special form that allows them to be compactly represented in terms of low-dimensional parameters $\btheta = (\btheta_1, \btheta_2)$; and 3) the latter takes the form of an additive combination (over two compressions, in the example). Problem~\eqref{e:MCCO} has the advantage that it is amenable to modern techniques of numerical optimization, as we show in section~\ref{s:LC}.

Although the expression ``$\w = \bDelta_1(\btheta_1) + \bDelta_2(\btheta_2)$'' is an addition, it implicitly is a linear combination because the coefficients can typically be absorbed inside each $\bDelta_i$. For example, writing $\alpha \U \V^T$ (for low-rank compression) is the same as, say, $\U' \V^T$ with $\U' = \alpha \U$. In particular, any compression member may be implicitly removed by becoming zero. Some additive combinations are redundant, such as having both $\W_1$ and $\W_2$ be of rank at most $r$ (since $\rank{\W_1 + \W_2} \le 2r$) or having each contain at most $\kappa$ nonzeros (since $\norm{\smash{\W_1 + \W_2}}_0 \le 2\kappa$).

The additive combination formulation has some interesting consequences. First, \emph{an additive combination of compression forms can be equivalently seen as a new, learned deep net architecture}. For example (see fig.~\ref{f:arch}), low-rank plus pruning can be seen as a layer with a linear bottleneck and some skip connections which are learned (i.e., which connections to have and their weight value). It is possible that such architectures may be of independent interest in deep learning beyond compression. Second, while \emph{pruning in isolation means (as is usually understood) the removal of weights from the model, pruning in an additive combination means the addition of a few elementwise real-valued corrections}. This can potentially bring large benefits. As an extreme case, consider binarizing both the multiplicative and additive (bias) weights in a deep net. It is known that the model's loss is far more sensitive to binarizing the biases, and indeed compression approaches generally do not compress the biases (which also account for a small proportion of weights in total). In binarization plus pruning, all weights are quantized but we learn which ones need a real-valued correction for an optimal loss. Indeed, our algorithm is able to learn that the biases need such corrections more than other weights (see corresponding experiment in appendix~\ref{section:biases}).

\subsection{Well known combinations} Our motivation is to combine generically existing compressions in the
  context of model compression. However, some of the combinations are well known and extensively studied. Particularly, low-rank + sparse combination has been used in its own right in the fields of compressed sensing \cite{Candes_11a}, matrix decomposition \cite{ZhouTao11a}, and image processing \cite{BouwmanZahzah16a}. This combination enjoys certain theoretical guarantees \cite{Candes_11a,Chandr_10a}, yet it is unclear whether similar results can be stated over more general additive combinations (e.g., with non-differentiable scheme like quantization) or when applied to non-convex models as deep nets.

\subsection{Hardware implementation}
The goal of model compression is to implement in practice the compressed model based on the \btheta\ parameters, not the original weights \W. With an additive combination, the implementation is straightforward and efficient by applying the individual compressions sequentially and cumulatively. For example, say $\W = \W_1 + \W_2$ is a weight matrix in a layer of a deep net and we want to compute the layer's output activations $\sigma(\W\x)$ for a given input vector of activations \x\ (where $\sigma(\cdot)$ is a nonlinearity, such as a ReLU). By the properties of linearity, $\W\x = \W_1\x + \W_2\x$, so we first compute $\y = \W_1\x$ according to an efficient implementation of the first compression, and then we accumulate $\y = \y + \W_2\x$ computed according to an efficient implementation of the second compression. This is particularly beneficial because some compression techniques are less hardware-friendly than others. For example, quantization is very efficient and cache-friendly, since it can store the codebook in registers, access the individual weights with high memory locality, use mostly floating-point additions (and nearly no multiplications), and process rows of $\W_1$ in parallel. However, pruning has a complex, nonlocal pattern of nonzeros whose locations must be stored. Combining quantization plus pruning not only can achieve higher compression ratios than either just quantization or just pruning, as seen in our experiments; it can also reduce the number of bits per weight and (drastically) the number of nonzeros, thus resulting in fewer memory accesses and hence lower runtime and energy consumption.

\section{Optimization via a learning-compression algorithm}
\label{s:LC}

Although optimizing~\eqref{e:MCCO} may be done in different ways for specific forms of the loss $L$ or the decompression mapping constraint $\bDelta$, it is critical to be able to do this in as generic way as possible, so it applies to any combination of forms of the compressions, loss and model. Following Carreira-Perpinan \cite{Carreir17a}, we apply a penalty method and then alternating optimization. We give the algorithm for the quadratic-penalty method \citep{NocedalWright06a}, but we implement the augmented Lagrangian one (which works in a similar way but with the introduction of a Lagrange multiplier vector \blambda\ of the same dimension as \w). We then optimize the following while driving a penalty parameter $\mu \to \infty$:
\begin{equation}
  \label{e:QP}
  Q(\w,\btheta;\mu) = L(\w) + \frac{\mu}{2} \norm{\w - \bDelta_1(\btheta_1) - \bDelta_2(\btheta_2)}^2  
\end{equation}
by using alternating optimization over \w\ and \btheta. The step over \w\ (``learning (L)'' step) has the form of a standard loss minimization but with a quadratic regularizer on \w\ (since $\bDelta_1(\btheta_1) + \bDelta_2(\btheta_2)$ is fixed), and can be done using a standard algorithm to optimize the loss, e.g., SGD with deep nets. The step over \btheta\ (``compression (C)'' step) has the following form:
\begin{equation}
  \label{e:Cstep}
  \min_{\btheta}{ \norm{\w - \bDelta(\btheta)}^2 } \, \Leftrightarrow \, \min_{\btheta_1,\btheta_2}{ \norm{\w - \bDelta_1(\btheta_1) - \bDelta_2(\btheta_2)}^2 }.
\end{equation}

In the original LC algorithm \cite{Carreir17a}, this step (over just a single compression $\bDelta(\btheta)$) typically corresponds to a well-known compression problem in signal processing and can be solved with existing algorithms. This gives the LC algorithm a major advantage: in order to change the compression form, we simply call the corresponding subroutine in this step (regardless of the form of the loss and model). For example, for low-rank compression the solution is given by a truncated SVD, for pruning by thresholding the largest weights, and for quantization by $k$-means. It is critical to preserve that advantage here so that we can handle in a generic way an arbitrary additive combination of compressions. Fortunately, we can achieve this by applying alternating optimization again but now to~\eqref{e:Cstep} over $\btheta_1$ and $\btheta_2$, as follows%
\footnote{This form of iterated ``fitting'' (here, compression) by a ``model'' (here, $\bDelta_1$ or $\bDelta_2$) of a ``residual'' (here, $\w - \bDelta_2(\btheta_2)$ or $\w - \bDelta_1(\btheta_1)$) is called backfitting in statistics, and is widely used with additive models \citep{HastieTibshir90a}.}:
\begin{equation}
  \begin{split}
    \label{e:Cstep1}
    \hspace{-1ex}\btheta_1 = \argmin_{\btheta}{ \norm{ (\w - \bDelta_2(\btheta_2)) - \bDelta_1(\btheta)}^2 } \\ \btheta_2 = \argmin_{\btheta}{ \norm{(\w - \bDelta_1(\btheta_1)) - \bDelta_2(\btheta)}^2 }
  \end{split}
\end{equation}
Each problem in~\eqref{e:Cstep1} now does have the standard compression form of the original LC algorithm and can again be solved by an existing algorithm to compress optimally according to $\bDelta_1$ or $\bDelta_2$. At the beginning of each C step, we initialize $\btheta$ from the previous C step's result (see Alg.~\ref{alg:psuedocode}).

\begin{algorithm}[t]
\begin{center}
\small
    \begin{minipage}[c]{0.99\linewidth}
      \begin{tabbing}
        n \= n \= n \= n \= n \= \kill
        \underline{\textbf{input}} \caja{t}{l}{training data, neural net architecture with weights \w } \\
        $\w \leftarrow \argmin_{\w}{ L(\w) }$ \` {\small\textsf{reference net}} \\
        $\btheta_1, \btheta_2 \leftarrow \argmin_{\btheta_1, \btheta_2} \norm{\w - \bDelta_1(\btheta_1) - \bDelta_2(\btheta_2)}^2$ \` {\small\textsf{init}} \\
        \underline{\textbf{for}} $\mu = \mu_0 < \mu_1 < \dots < \infty$ \+ \\
        $\w \leftarrow \argmin_{\w}{ L(\w) + \smash{\frac{\mu}{2} \norm{\smash{\w - \bDelta_1(\btheta_1) -\bDelta_2(\btheta_2)}}^2} }$ \` {\small\textsf{L step}} \\
        \underline{\textbf{while}} alternation does not converge \+ \\
        $\btheta_1 \leftarrow \argmin_{\btheta_1} \norm{\left(\w-\bDelta_2(\btheta_2)\right) - \bDelta_1(\btheta_1)}^2$ \` {\small\raisebox{0pt}[0pt][0pt]{$\left.\rule{0cm}{0.7cm}\right\}\quad$}\textsf{C step}} \\
        $\btheta_2 \leftarrow \argmin_{\btheta_2} \norm{\left(\w - \bDelta_1(\btheta_1)\right) -\bDelta_2(\btheta_2)}^2$ \-\\
        \textbf{if} $\norm{\w - \bDelta_1(\btheta_1) -\bDelta_2(\btheta_2)}$ is small enough \textbf{then} exit the loop \- \\
        \underline{\textbf{return}} $\w, \btheta_1, \btheta_2$ 
      \end{tabbing}
    \end{minipage}
    \caption{Pseudocode (quadratic-penalty version)}
    \label{alg:psuedocode}
\end{center}
\end{algorithm}

It is possible that a better algorithm exists for a specific form of additive combination compression \eqref{e:Cstep}. In such case we can employ specialized version during the C step. But our proposed alternating optimization~\eqref{e:Cstep1} provides a generic, efficient solution as long as we have a good algorithm for each individual compression.

Convergence of the alternating steps~\eqref{e:Cstep1} to a global optimum of~\eqref{e:Cstep} over $(\btheta_1,\btheta_2)$ can be proven in some cases, e.g., low-rank + sparse \cite{ZhouTao11a}, but not in general, as one would expect since some of the compression problems involve discrete and continuous variables and can be NP-hard (such as quantization with an adaptive codebook). Convergence can be established quite generally for convex functions \citep{BeckTetruas13a,Wright16a}. For nonconvex functions, convergence results are complex and more restrictive \citep{Tseng01a}. One simple case where convergence occurs is if the objective in~\eqref{e:Cstep} (i.e., each $\bDelta_i$) is continuously differentiable and it has a unique minimizer over each $\btheta_i$ \citep[Proposition 2.7.1]{Bertsek99a}. However, in certain cases the optimization can be solved exactly without any alternation. We give a specific result next.

\subsection{Exactly solvable C step}
\label{s:exact_c_step}
Solution of the C step (eq.~\ref{e:Cstep}) does not need to be an alternating optimization. Below we give an exact algorithm for the additive combination of fixed codebook quantization (e.g., $\{-1,+1\}$, $\{-1,0,+1\}$, etc.) and sparse corrections.

\begin{thm}[Exactly solvable C step for combination of fixed codebook quantization + sparse corrections]
  \label{th:fixed_codebook}
  Given a fixed codebook $\calC$ consider compression of the weights $w_i$ with an additive combinations of quantized values $q_i \in \calC$ and sparse corrections $s_i$:
\begin{equation}
\label{eq:fixed_sparse}
{
\min_{q,s}{ \sum_i{ (w_i - (q_i+s_i))^2 } } \quad \text{s.t.} \quad \norm{s}_0 \leq \kappa,}
\end{equation}  
  Then the following provides one optimal solution $(q^*,s^*)$: first set
  $q_i^* = closest(w_i)$ in codebook for each $i$, then solve for $s$:
    $\min_s \,{ \sum_i { (w_i-q_i^* - s_i))^2 } } \text{ s.t. } \norm{s}_0 \leq \kappa$.
\end{thm}
\begin{proof}
  Imagine we know the optimal set of nonzeros of the vector $s$, which we denote as $\calN$. Then, for the
  elements not in $\calN$, the optimal solution is $s_i^*=0$ and $q_i^* = closest(w_i)$.
  For the elements in $\calN$, we can find their optimal solution by solving
  independently for each $i$:  
  \begin{equation*}
    \min_{q_i,s_i} \, { (w_i - (q_i+s_i))^2 } \quad \text{ s.t. } \quad q_i \in \calC. 
  \end{equation*}    
  The solution is $s_i^* = w_i-q_i$ for arbitrary chosen $q_i \in \calC$. Using this, we can rewrite the
  eq.~\ref{eq:fixed_sparse} as $\sum_{i \notin \calN}{ (w_i - q_i^*)^2 }$. 
  
  This is minimized by taking as set $\calN$ the $\kappa$ largest in magnitude elements of $w_i-q_i^*$ (indexed over $i$). Hence, the final solution is: 1) Set the elements of \calN{} to be the $\kappa$ largest in magnitude elements of $w_i-q_i^*$ (there may be multiple such sets, any one is valid). 2) For each $i$ in \calN: set $s_i^* = w_i-q_i^*$, and $q_i^* = $ any element in $\calC$. For each $i$ not in $\calN$: set $s_i^*=0$, $q_i^* = closest(w_i)$ (there may be 2 closest values, any one is valid).
  This contains multiple solutions. One particular one is as given in the
  theorem statement, where we set $q_i^* = closest(w_i)$ for every $i$, which is
  practically more desirable because it leads to a smaller $\ell_1$-norm of $s$.
\end{proof}
\section{Experiments on CIFAR10}
\label{s:expts}

\begin{figure*}[t!]
  \centering
  \begin{tabular}{@{}cc@{}}
    \begin{tabular}[c]{@{}l@{\hspace{0.5em}}c@{\hspace{0.7em}}c@{\hspace{0.7em}}r@{\hspace{0.7em}}r@{\hspace{0.7em}}r@{\hspace{0.7em}}r@{}}
      \toprule
      & 1bit Q + \% P & $\log L$ & \makebox[1pt][c]{$E_\text{test}$(\%)} & $\rho_{\text{s}}$ & $\rho_{+}$& $\rho_{\times}$ \\
      \midrule
      &{\textbf{R}}  &-0.80& 8.35                 & 1.00           & \textbf{1.00} & 1.00 \\
      & Q + 1.0\% P  &-0.84& 9.16                 & \textbf{22.67} & 0.97          & \textbf{30.74}\\ 
      \raisebox{0pt}[0pt][0pt]{\rotatebox{90}{\makebox[0pt][c]{  ResNet20 }}}
      & Q + 2.0\% P  &-0.92& 8.92                 & 19.44         & 0.96           & 19.74\\
      & Q + 3.0\% P  &-0.93& 8.31                 & 17.08         & 0.94           & 15.80 \\
      & Q + 5.0\% P  &-0.99& \textbf{8.26}        & 13.84         & 0.92           & 11.54 \\
      \midrule
      &{\textbf{R}}  &-0.82& 7.14        & 1.00           & \textbf{1.00} & 1.00 \\
      & Q + 1.0\% P  &-1.03& 7.57                 & \textbf{22.81} & 0.97          & \textbf{30.52}\\ 
      \raisebox{0pt}[0pt][0pt]{\rotatebox{90}{\makebox[0pt][c]{  ResNet32 }}}
      & Q + 2.0\% P  &-1.07& 7.61                 & 19.54         & 0.96           & 19.85\\
      & Q + 3.0\% P  &-1.10& 7.29                 & 17.14         & 0.94           & 15.80 \\
      & Q + 5.0\% P  &-1.14& \textbf{7.09}        & 13.84         & 0.92           & 11.56 \\
      \midrule
      &{\textbf{R}}  &-0.81& 6.58                 & 1.00           & \textbf{1.00} & 1.00 \\
      & Q + 0.5\% P  &-1.08& 6.77                 & \textbf{25.04} & 0.98          & \textbf{49.79}\\ 
      \raisebox{0pt}[0pt][0pt]{\rotatebox{90}{\makebox[0pt][c]{ResNet56 }}}
      & Q + 1.0\% P  &-1.13& 6.73                 & 22.87         & 0.97           & 32.04\\
      & Q + 2.0\% P  &-1.17& 6.70                 & 19.55         & 0.96           & 20.46\\
      & Q + 3.0\% P  &-1.18& \textbf{6.23}        & 17.11         & 0.94           & 15.98 \\
      \midrule
      &{\textbf{R}}  &-0.77& 6.02                 & 1.00           & \textbf{1.00} & 1.00 \\
      & Q + 0.5\% P  &-1.16& 6.20                 & \textbf{25.03} & 0.99          & \textbf{55.63}\\ 
      \raisebox{0pt}[0pt][0pt]{\rotatebox{90}{\makebox[0pt][c]{ResNet110 }}}
      & Q + 1.0\% P  &-1.20& 5.80                 & 22.80         & 0.98           & 35.94\\
      & Q + 2.0\% P  &-1.23& 5.66                 & 19.47         & 0.96           & 27.27 \\
      & Q + 3.0\% P  &-1.25& \textbf{5.58}         & 17.04         & 0.95           & 17.84 \\
      \bottomrule
    \end{tabular} &
    \psfrag{resnet20}[l][l]{\small ResNet20}
    \psfrag{resnet32}[l][l]{\small ResNet32}
    \psfrag{resnet56}[l][l]{\small ResNet56}
    \psfrag{resnet110}[l][l]{\small ResNet110}
    \psfrag{qpchoi}[l][l]{~\tiny P$\rightarrow$Q$\rightarrow$HC\cite{Choi_17a}}
    \psfrag{agusts17}[l][l]{~\tiny Q$\rightarrow$HC\cite{Agusts_17a}}
    \psfrag{qzhu17}[l][l]{~\tiny $\text{Q}$  \cite{Zhu_17a}}
    \psfrag{yin18b}[l][bl]{~\tiny $\text{Q}$ \cite{Yin_18b}}
    \psfrag{xu18a}[bl][tl]{\tiny $\text{Q}$ \cite{Xu_18b}}
    \psfrag{cp}[l][l]{~\tiny $\text{P}$ \cite{CarreirIdelbay18a}}
    \psfrag{liu19a}[bl][tl]{~\tiny $\text{P}$ \cite{Liu_19a}}
    \psfrag{err}[B][]{test error $E_{\text{test}}$ (\%)}
    \psfrag{loss}[][]{loss L}
    \psfrag{CR}[c][b]{storage ratio $\rho_{\text{s}}$}
    \begin{tabular}[c]{@{}r@{}}
      \psfrag{CR}[B][]{}\\[0.1em]
      \includegraphics[clip,height=0.37\linewidth]{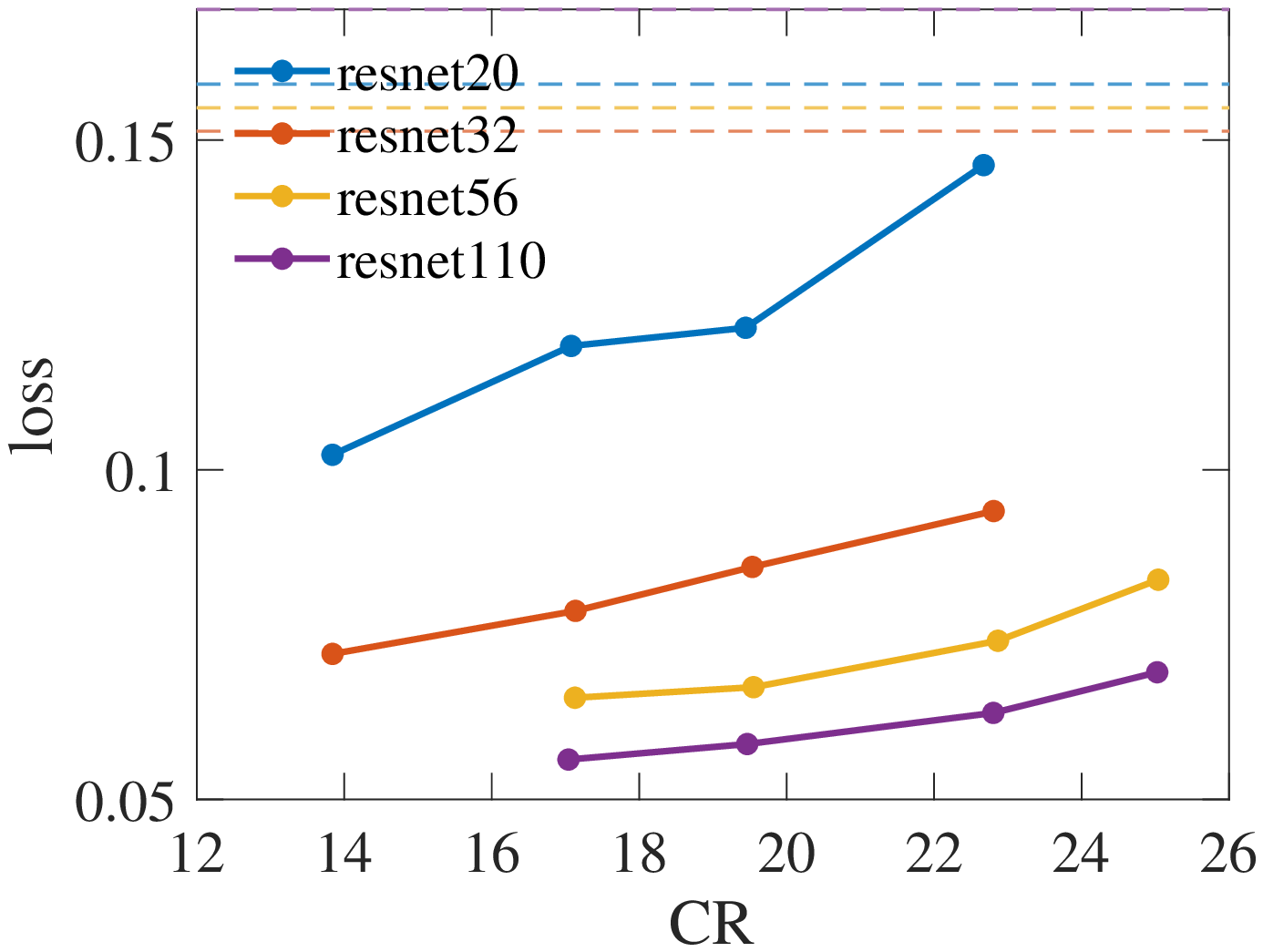} \\[1em]
      \includegraphics[clip,height=0.37\linewidth]{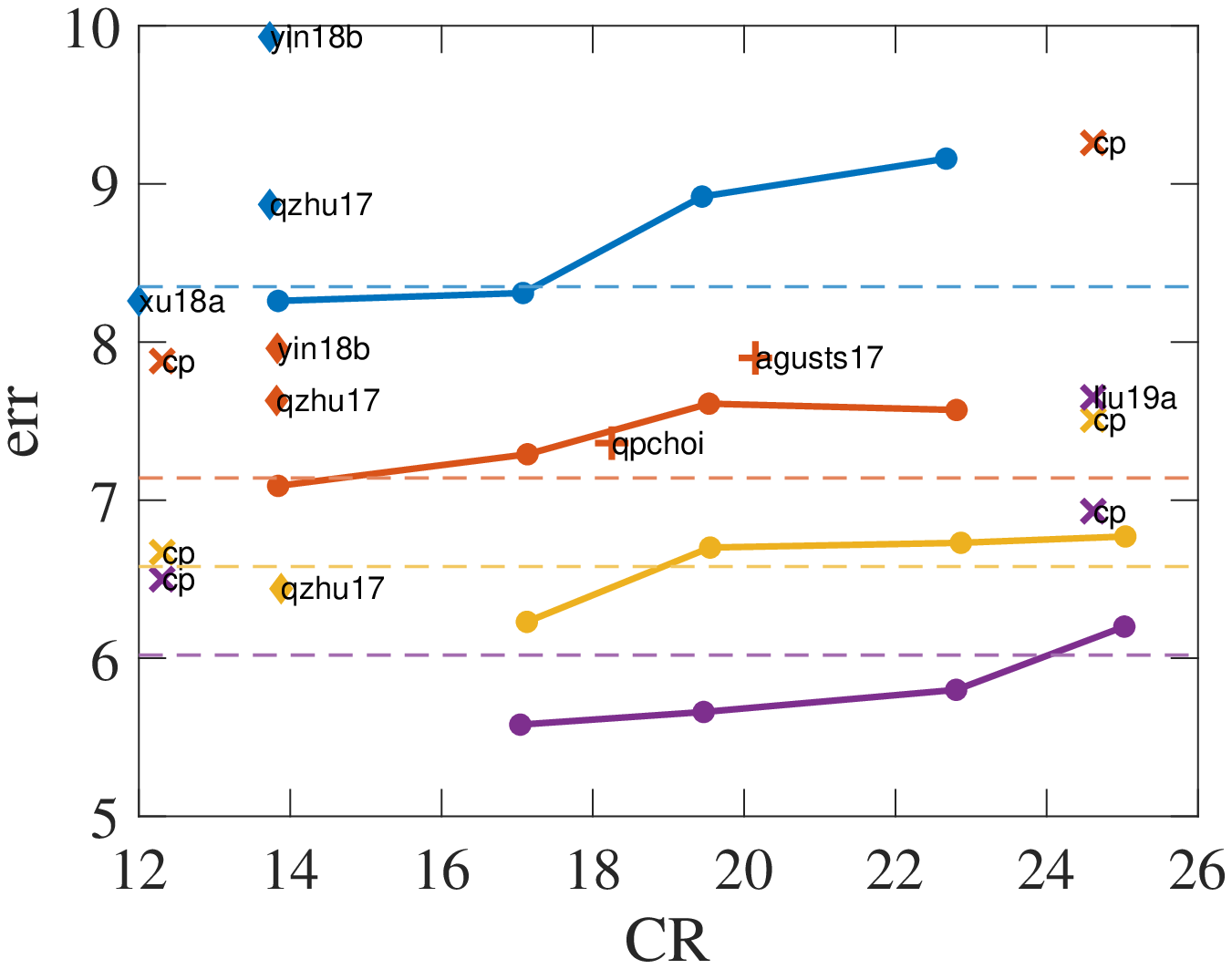}
    \end{tabular}
  \end{tabular}
  \caption{Q+P. \emph{Left}: results of running 1-bit quantization with varying amounts of additive pruning (corrections) on ResNets of different depth on CIFAR-10 (with reference nets denoted \textbf{R}). We report training loss (logarithms are base 10), test error $E_\text{test}$ (\%), and ratios of storage $\rho_{\text{s}}$ and floating point additions ($\rho_{+}$) and multiplications ($\rho_{\times}$). Boldfaced results are the best for each ResNet depth. \emph{Right}: training loss (\emph{top}) and test error (\emph{bottom}) as a function of the storage ratio. For each net, we give our algorithm's compression over several values of P, thus tracing a line in the error-compression space (reference nets: horizontal dashed lines). We also report results from the literature as isolated markers with a citation: quantization Q, pruning P, Huffman coding HC, and their sequential combination using arrows (e.g., Q$\rightarrow$HC). Point ``Q \cite{Xu_18b}'' on the left border is outside the plot ($\rho_{\text{s}} < 12$). }
  \label{fig:q_p}
\end{figure*}

\begin{figure*}
  \centering
  \begin{tabular}{@{}cc@{}}
    \begin{tabular}{@{}l@{\hspace{0.1em}}c@{\hspace{0.5em}}c@{\hspace{0.5em}}c@{\hspace{0.0em}}r@{\hspace{0.5em}}r@{\hspace{0.5em}}r@{}}
      \toprule
      & 1bit Q + rank $r$  & $\log L$ & $E_\text{test}$ (\%) & $\rho_{\text{s}}$ & $\rho_{+}$& $\rho_{\times}$ \\
      \midrule
      & {\textbf{R}}      &-0.80& \textbf{8.35}        & 1.00           & \textbf{1.00} & 1.00 \\
      &Q + rank 1   & -0.77 & 9.71                 & \textbf{20.71} & 0.96          & \textbf{21.45}\\ 
      \raisebox{0pt}[0pt][0pt]{\rotatebox{90}{\makebox[0pt][c]{\hspace{1.5em}\scriptsize  ResNet20 }}}
      &Q + rank 2   & -0.84 & 9.30                 & 16.62         & 0.92           & 11.26\\
      &Q + rank 3   & -0.89 & 8.64                 & 13.88         & 0.89           & 7.64 \\
      \midrule
      &{\textbf{R}}        &-0.82& \textbf{7.14}        & 1.00           & \textbf{1.00} & 1.00 \\
      &Q + rank 1   & -0.99 & 7.90                 & \textbf{20.94} & 0.97          & \textbf{21.89}\\ 
      \raisebox{0pt}[0pt][0pt]{\rotatebox{90}{\makebox[0pt][c]{\hspace{1.5em} \scriptsize  ResNet32 }}}
      &Q + rank 2   & -1.04 & 8.06                 & 16.81         & 0.92           & 11.47\\
      &Q + rank 3   & -1.10 & 7.52                 & 14.04         & 0.89           & 7.77 \\
      \midrule
      &{\textbf{R}}        &-0.81& 6.58                 & 1.00           & \textbf{1.00} & 1.00 \\
      &Q + rank 1   & -1.13 & 7.19                 & \textbf{21.04} & 0.96          & \textbf{22.19}\\ 
      \raisebox{0pt}[0pt][0pt]{\rotatebox{90}{\makebox[0pt][c]{  \hspace{1.5em} \scriptsize ResNet56 }}}
      &Q + rank 2   & -1.19 & 6.51                 & 16.91         & 0.92           & 11.61\\
      &Q + rank 3   & -1.22 & \textbf{6.29}        & 14.10         & 0.89           &  7.87 \\
      \midrule
      &{\textbf{R}}        &-0.77& 6.02                 & 1.00           & \textbf{1.00} & 1.00 \\
      &Q + rank 1   & -1.19 & 5.98                 & \textbf{21.11} & 0.96          & \textbf{22.38}\\ 
      \raisebox{0pt}[0pt][0pt]{\rotatebox{90}{\makebox[0pt][c]{  \hspace{1.5em} \scriptsize ResNet110 }}}
      &Q + rank 2   & -1.24 & 5.93                 & 16.96         & 0.92           & 11.70\\
      &Q + rank 3   & -1.27 & \textbf{5.50}        & 14.18         & 0.89           & 7.92 \\
      \bottomrule
    \end{tabular} &
    \psfrag{resnet20}[l][l]{\small ResNet20}
    \psfrag{resnet32}[l][l]{\small ResNet32}
    \psfrag{resnet56}[l][l]{\small ResNet56}
    \psfrag{resnet110}[l][l]{\small ResNet110}
    \psfrag{qzhu17}[l][l]{~\tiny $\text{Q}$ \cite{Zhu_17a}}
    \psfrag{yin18b}[l][bl]{~\tiny $\text{Q}$ \cite{Yin_18b}}
    \psfrag{xu18a}[bl][tl]{\tiny $\text{Q}$ \cite{Xu_18b}}
    \psfrag{wen17}[l][l]{~\tiny $\text{L}_1$ \cite{Wen_17a}}
    \psfrag{xu18b}[bl][tl]{\tiny $\text{L}_2$ \cite{Xu_20a}}
    \psfrag{err}[B][]{test error $E_{\text{test}}$ (\%)}
    \psfrag{loss}[][]{loss L}
    \psfrag{CR}[c][b]{storage ratio $\rho_{\text{s}}$}
    \begin{tabular}{@{}r@{}}
      \psfrag{CR}[B][]{}
      \includegraphics[clip,height=0.31\linewidth]{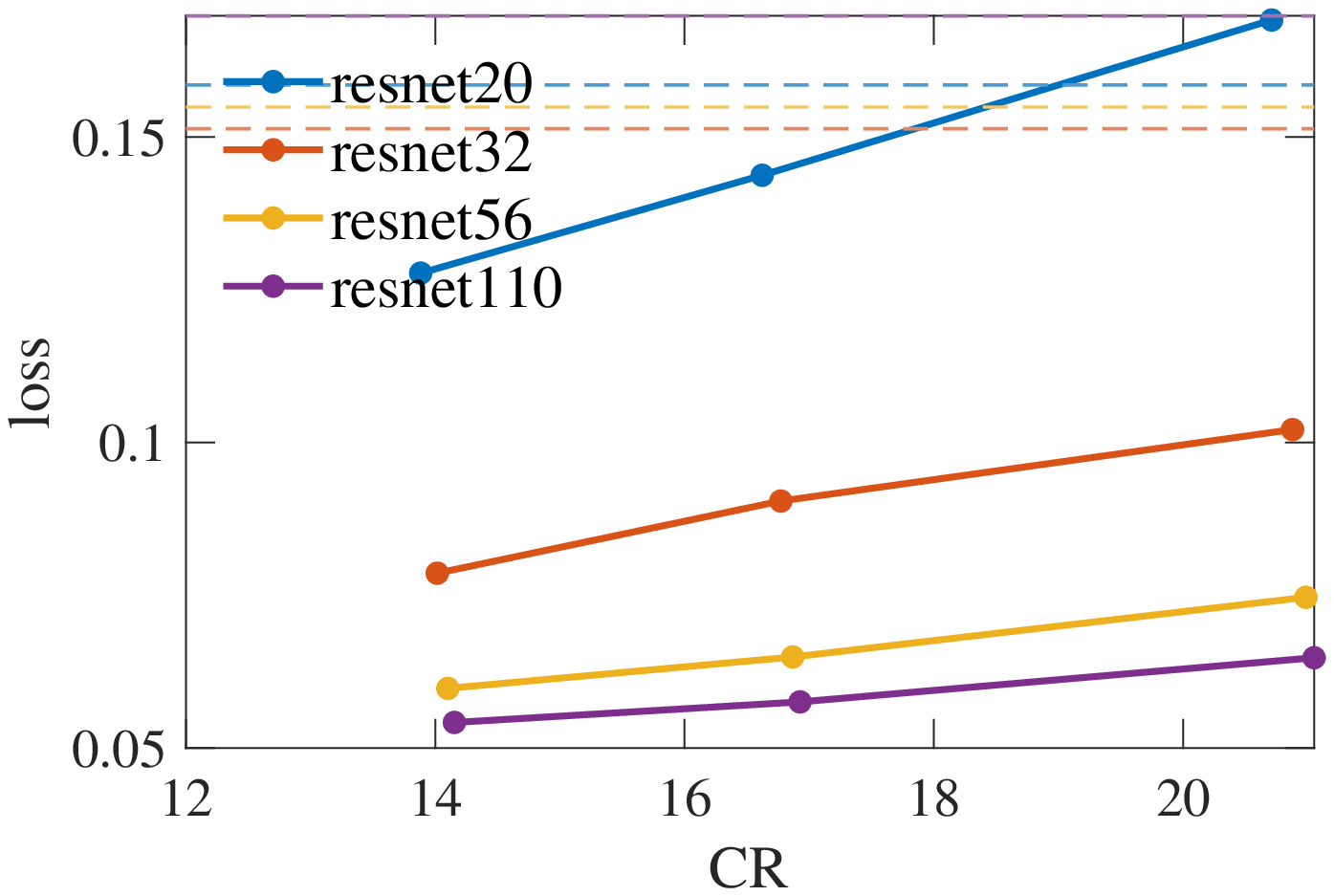} \\
      \includegraphics[clip,height=0.32\linewidth]{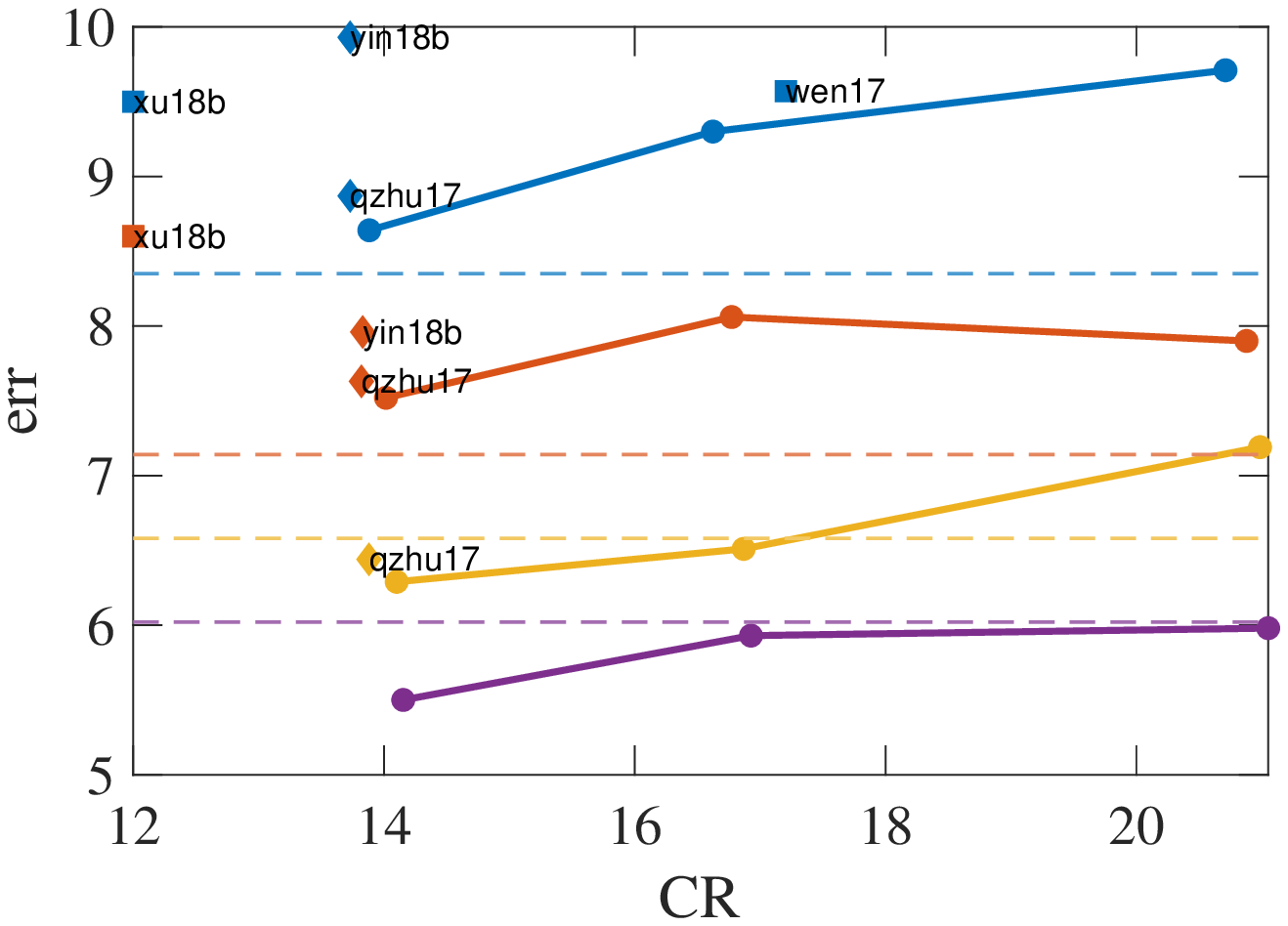}
    \end{tabular}
  \end{tabular}
  \caption{Q+L.~\emph{Left}: results of running 1-bit quantization with addition of a low-rank matrix of different rank on ResNets on CIFAR10. The organization is as for fig.~\ref{fig:q_p}. In the right-bottom plot we also show results from the literature for single compressions (Q: quantization, L: low-rank). Points ``$\text{L}$ \cite{Xu_20a}'' on the left border are outside the plot ($\rho_{\text{s}} < 12$). }
  \label{fig:q_l}
\end{figure*}

We evaluate the effectiveness of additively combining compressions on deep nets of different sizes on the CIFAR10 (VGG16 and ResNets). We systematically study each combination of two or three compressions out of quantization, low-rank and pruning. We demonstrate that the additive combination improves over any single compression contained in the combination (as expected), and is comparable or better than sequentially engineered combinations such as first pruning some weights and then quantizing the rest. We sometimes achieve models that not only compress the reference significantly but also reduce its error. Notably, this happens with ResNet110 (using quantization plus either pruning or low-rank), even though our reference ResNets were already well trained and achieved a lower test error than in the original paper \citep{He_16a}.

We initialize our experiments from reasonably well-trained reference models. We train reference ResNets of depth 20, 32, 56, and 110 following the procedure of the original paper \citep{He_16a} (although we achieve lower errors). The models have 0.26M, 0.46M, 0.85M, and 1.7M parameters and test errors of 8.35\%, 7.14\%, 6.58\% and 6.02\%, respectively. We adapt VGG16 \citep{SimonyZisser15a} to the CIFAR10 dataset (see details in appendix~\ref{section:exp_cifar10}) and train it using the same data augmentation as for ResNets. The reference VGG16 has 15.2M parameters and achieves a test error of 6.45\%.

The optimization protocol of our algorithm is as follows throughout all experiments with minor changes (see appendices \ref{section:exp_cifar10} and \ref{section:exp_imagenet}). To optimize the L step we use Nesterov's accelerated gradient method \citep{Nester83a} with momentum of 0.9 on minibatches of size 128, with a decayed learning rate schedule of $\eta_0 \cdot a^m$ at the $m$th epoch. The initial learning rate $\eta_0$ is one of \{0.0007,0.007,0.01\}, and the learning rate decay one of \{0.94,0.98\}. Each L step is run for 20 epochs. Our LC algorithm (we use augmented Lagrangian version) runs for $j$ steps where $j\leq 50$, and has a penalty parameter schedule $\mu_j = \mu_0 \cdot 1.1^j$; we choose $\mu_0$ to be one of $\{5\cdot 10^{-4},10^{-3}\}$. The solution of the C step requires alternating optimization over individual compressions, which we perform 30 times per each step.

We report the training loss and test error as measures of the model classification performance; and the ratios of storage (memory occupied by the parameters) $\rho_{\text{s}}$, number of multiplications $\rho_{\times}$ and number of additions $\rho_{+}$ as measures of model compression. Although the number of multiplications and additions is about the same in a deep net's inference pass, we report them separately because different compression techniques (if efficiently implemented) can affect quite differently their costs. We store low-rank matrices and sparse correction values using 16-bit precision floating point values. See our appendix~\ref{section:metrics} for precise definitions and details of these metrics.

\subsection{Q+P: quantization plus pruning}
\label{s:expts:quant+pruning}
We compress ResNets with a combination of quantization plus pruning. Every layer is quantized separately with a codebook of size 2 (1 bit). For pruning we employ the constrained $\ell_0$ formulation \cite{CarreirIdelbay18a}, which allows us to specify a single number of nonzero weights $\kappa$ for the entire net ($\kappa$ is the ``\% P'' value in fig.~\ref{fig:q_p}). The C step of eq.~\eqref{e:Cstep} for this combination alternates between running $k$-means (for quantization) and a closed-form solution based on thresholding (for pruning).

Fig.~\ref{fig:q_p} shows our results; published results from the literature are at the bottom-right part, which shows the error-compression space. We are able to achieve considerable compression ratios $\rho_{\text{s}}$ of up to 20$\times$ without any degradation in accuracy, and even higher ratios with minor degradation. These results beat single quantization or pruning schemes reported in the literature for these models. The best 2-bit quantization approaches for ResNets we know about \citep{Zhu_17a,Yin_18b} have $\rho_{\text{s}} \approx 14\times$ and lose up to 1\% in test error comparing to the reference; the best unstructured pruning \citep{CarreirIdelbay18a,Liu_19a} achieves $\rho_{\text{s}} \approx 12\times$ and loses 0.8\%. 

ResNet20 is the smallest and hardest to compress out of all ResNets. With 1-bit quantization plus 3\% pruning corrections we achieve an error of 8.26\% with $\rho_{\text{s}} =$ 13.84$\times$. To the best of our knowledge, the highest compression ratio of comparable accuracy using only quantization is 6.22$\times$ and has an error of 8.25\% \citep{Xu_18b}. On  ResNet110 with 1-bit quantization plus 3\% corrections, we achieve 5.58\% error while still compressing 17$\times$.

Our results are comparable or better than published results where multiple compressions are applied sequentially (Q$\rightarrow$HC and P$\rightarrow$Q$\rightarrow$HC in fig.~\ref{fig:q_p}). For example, quantizing and then applying Huffman coding to ResNet32 \citep{Agusts_17a} achieves $\rho_{\text{s}} =$ 20.15$\times$ with 7.9\% error, while we achieve $\rho_{\text{s}} =$ 22.81$\times$ with 7.57\% error. We re-emphasize that unlike the ``prune then quantize'' schemes, our additive combination is different: we quantize all weights and apply a pointwise correction. 

\begin{table}[t]
  \centering
  \begin{tabular}{@{}@{}llcrrrr@{}}
    \toprule
    & Model & $E_\text{test}$ (\%)  & $\rho_{\text{s}}$  \\
    \midrule
    &{\textbf{R}}  VGG16      & \textbf{6.45}                 & 1.00            \\
    \midrule
    & rank 2 + 2\% P  &  6.66                 & \textbf{60.99} \\ 
    \raisebox{0pt}[0pt][0pt]{\rotatebox{90}{\makebox[0pt][c]{\hspace*{1.5em}ours}}}
    & rank 3 + 2\% P &  6.65                 & 56.58         \\
    \midrule
    & pruning \citep{Liu_19a} &  6.66 &  $\approx$ 24.53  \\
    & filter pruning \citep{Li_17b} & 6.60 & 5.55   \\
    & quantization \citep{Qu_20a} & 8.00 & 43.48   \\
    \bottomrule
  \end{tabular}
  \caption{L+P. Compressing VGG16 with low-rank and pruning using our algorithm (top) and by recent works on structured and unstructured pruning. Metrics as in Fig.~\ref{fig:q_p}.} \label{table:vgg16}
\end{table}

\begin{figure*}[t]
  \centering
  \begin{tabular}{@{}cc@{}}
    \begin{tabular}{@{}l@{\hspace{0.3em}}l@{\hspace{0.1em}}c@{\hspace{0.4em}}r@{}l@{\hspace{0.3em}}c@{}}
      \toprule
       &\multicolumn{1}{c}{Model} & top-1 &  MBs&& MFLOPs  \\
         \midrule
       &Caffe-AlexNet \cite{Jia_14a}  & 42.70               & 243.5&  & 724           \\
       \midrule
       & AlexNet-QNN \cite{Wu_16a} & 44.24 & 13.0& & 175 \\
       &P$\rightarrow_1$Q \cite{Han_16a}  & 42.78   & 6.9&  & 724    \\
       &P$\rightarrow_2$Q \cite{Choi_17a}  & 43.80   & 5.9&  & 724    \\
       & P$\rightarrow_3$Q \cite{TungMori18a} & 42.10 & 4.8& & 724 \\
       & P$\rightarrow_4$Q \cite{Yang_20b} & 42.48 & 4.7& & 724\\
       & P$\rightarrow_5$Q \cite{Yang_20b} & 43.40 & 3.1& & 724 \\
       & filter pruning \cite{Li_19a} & 43.17 & 232.0 & & 334\\
       \midrule
       &\textbf{R} Low-rank AlexNet ($\text{L}_1$) & 39.61 & 100.5 & & 227           \\
       &$\text{L}_1$ $\rightarrow$ Q + P (0.25M)  & 39.67  & 3.7 & & 227 \\
       \raisebox{0pt}[0pt][0pt]{\rotatebox{90}{\makebox[0pt][c]{\hspace*{2em}ours }}}
       &$\text{L}_1$ $\rightarrow$ Q  + P (0.50M)  & \textbf{39.25}  & 4.3 &  & 227 \\
       \midrule 
       &\textbf{R} Low-rank AlexNet ($\text{L}_2$) & 39.93     & 69.8 & & 185           \\
       &$\text{L}_2$ $\rightarrow$ Q  + P (0.25M)  &   40.19 & 2.8 & & 185 \\
       \raisebox{0pt}[0pt][0pt]{\rotatebox{90}{\makebox[0pt][c]{\hspace*{2em}ours }}}
       &$\text{L}_2$ $\rightarrow$ Q  + P (0.50M)  & 39.97  & 3.4 &  & 185 \\ 
       \midrule
       &\textbf{R} Low-rank AlexNet ($\text{L}_3$) & 41.02     & 45.9 & & \textbf{152} \\
       &$\text{L}_3$ $\rightarrow$ Q  + P (0.25M)   & 41.27  & \textbf{2.1}& & \textbf{152} \\
       \raisebox{0pt}[0pt][0pt]{\rotatebox{90}{\makebox[0pt][c]{\hspace*{2em}ours }}}
       &$\text{L}_3$ $\rightarrow$ Q  + P (0.50M)    & 40.88  & 2.7 &  & \textbf{152} \\ 
      \bottomrule
    \end{tabular} &
    \begin{tabular}[c]{@{}c@{}}
    \psfrag{caffe}[l][l]{\scriptsize Caffe-AlexNet } 
    \psfrag{qnn}[l][l]{~\scriptsize QNN }
    \psfrag{han16}[l][l]{~\scriptsize P$\rightarrow_1$Q}
    \psfrag{choi17}[l][l]{~\scriptsize P$\rightarrow_2$Q} 
    \psfrag{tung18}[l][l]{~\scriptsize P$\rightarrow_3$Q}
    \psfrag{yang20v1}[l][l]{~\scriptsize P$\rightarrow_4$Q}
    \psfrag{yang20v2}[l][l]{~\scriptsize P$\rightarrow_5$Q} 
    \psfrag{wen17}[l][l]{~\scriptsize $\text{L}_1$ } 
    \psfrag{xu18b}[bl][tl]{\scriptsize $\text{L}_2$} 
    \psfrag{err}[B][]{top-1 test error (\%)}
    \psfrag{loss}[][]{loss L}
    \psfrag{CR}[c][b]{storage ratio $\rho_{\text{s}}$}
    \psfrag{l1alexnet}[l][l]{\scriptsize$\text{L}_1\rightarrow$ Q+P }
    \psfrag{l2alexnet}[l][l]{\scriptsize$\text{L}_2\rightarrow$ Q+P }
    \psfrag{l3alexnet}[l][l]{\scriptsize$\text{L}_3\rightarrow$ Q+P }
    \psfrag{L1}[l][l]{~\scriptsize $\text{L}_1$}
    \psfrag{L2}[l][l]{~\scriptsize $\text{L}_2$}
    \psfrag{L3}[l][l]{~\scriptsize $\text{L}_3$}
    \includegraphics[clip,height=0.39\linewidth]{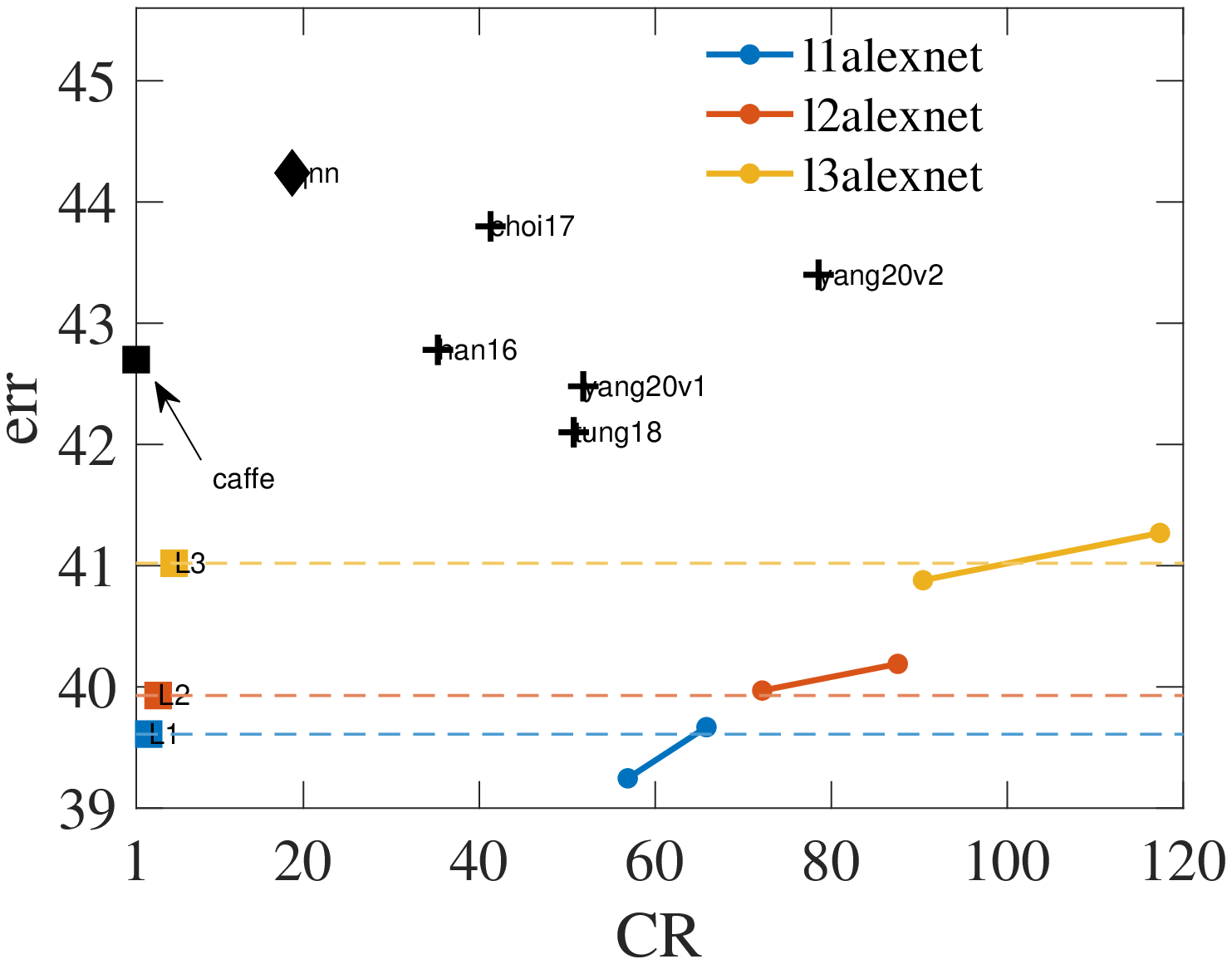} \\[1em]
    Inference time and speed-up using 1-bit Q + 0.25M P\\
    \begin{tabular}{@{}lc@{\hspace{1em}}c@{}}
      \toprule
      Model & time, ms & speed-up \\
      \midrule
      Caffe-AlexNet & 23.27 & 1.00 \\
      $\text{L}_1$ $\rightarrow$ Q + P   & 11.32 & 2.06 \\
      $\text{L}_2$ $\rightarrow$ Q + P   & \phantom{0}8.75 & 2.66 \\
      $\text{L}_3$ $\rightarrow$ Q + P   & \phantom{0}6.72 & 3.46 \\
      \bottomrule
    \end{tabular} 
    \end{tabular} 
  \end{tabular}
  \caption{Q+P scheme is powerful enough to further compress already downsized models, here, it is used to further compress the low-rank AlexNets \cite{IdelbayCarreir20a}. In all our experiments reported here, we use 1-bit quantization with varying amount of pruning. \emph{Left:} We report top-1 validation error, size of the final model in MB when saved to disk, and resulting FLOPs. P---pruning, Q--- quantization, L---low-rank. \emph{Top right:} same as the table on the left, but in graphical form. Our compressed models are given as solid connected lines.  \emph{Bottom right:} The delay (in milliseconds) and corresponding speed-ups of our compressed models on Jetson Nano Edge GPU. }
  \label{fig:qp_alexnet_low_rank}
\end{figure*}

\subsection{Q+L: quantization plus low-rank}
\label{s:expts:quant+low-rank}
We compress the ResNets with the additive combination of 1-bit quantized weights (as in section~\ref{s:expts:quant+pruning}) and rank-$r$ matrices, where the rank is fixed and has the same value for all layers. The convoloutional layers are parameterized by low-rank as in Wen et al.~\cite{Wen_17a}. The solution of the C step \eqref{e:Cstep1} for this combination is an alternation between $k$-means and truncated SVD. 

Fig.~\ref{fig:q_l} shows our results and at bottom-right of Fig.~\ref{fig:q_l} we see that our additive combination (lines traced by different values of the rank $r$ in the error-compression space) consistently improve over individual compression techniques reported in the literature (quantization or low-rank, shown by markers Q or L). Notably, the low-rank approximation is not a popular choice for compression of ResNets: fig.~\ref{fig:q_l} shows only two markers, for the only two papers we know  \citep{Wen_17a,Xu_20a}. Assuming storage with 16-bit precision on ResNet20, Wen et al.~\cite{Wen_17a} achieve 17.20$\times$ storage compression (with 9.57\% error) and Xu et al.~\cite{Xu_20a} respectively 5.39$\times$ (with 9.5\% error), while our combination of 1-bit quantization plus rank-2 achieves 16.62$\times$ (9.3\% error).

\subsection{L+P: low-rank plus pruning}
\label{s:expts:low-rank+pruning}
We compress VGG16 trained on CIFAR10 using the additive combination of low-rank matrices and pruned weights. The reference model has 15.2M parameters, uses 58.17 MB of storage and achieves 6.45\% test error. 
When compressing with L+P scheme of rank 2 and 3\% point-wise corrections (Table~\ref{table:vgg16}), we achieve a compression ratio of to 60.99$\times$ (0.95 MB storage), and the test error of 6.66\%.
\section{Experiments on ImageNet}
\label{s:expts_imagenet}
To demonstrate the power and complementary benefits of additive combinations, we proceed by applying the Q+P combination to \emph{already downsized} models trained on the ILSVRC2012 dataset. We obtain low-rank AlexNets following the work of \cite{IdelbayCarreir20a}, and compress them further with Q+P scheme. The hyperparameters of the experiments are almost identical to CIFAR10 experiments (sec.~\ref{s:expts}) with minor changes, see appendix~\ref{section:exp_imagenet}. 

In Fig.~\ref{fig:qp_alexnet_low_rank} (left) we report our results: the achieved top-1 error, the size in megabytes when a compressed model is saved to disk (we use the sparse index compression procedure of Han et al.~\cite{Han_16a}), and floating point operations required to perform the forward pass through a network. Additionally, we include prominent results from the literature to put our models in perspective. Our Q+P models achieve \emph{significant compression} of AlexNet: we get $117\times$ compression (2.075MB) without degradation in accuracy and $87\times$ compression with more than 2\% improvement in the \mbox{top-1} accuracy when compared to the Caffe-AlexNet. Recently, Yang et al.~\cite{Yang_20b} (essentially using our Learning-Compression framework) reported $118\times$ and $205\times$ compression on AlexNet with none to small reduction of accuracy. However, as can be found by inspecting the code of Yang et al.~\cite{Yang_20b}, these numbers are artificially inflated in that they do not account for the storage of the element indices (for a sparse matrix), for the storage of the codebook, and use a fractional number of bits per element instead of rounding it up to an integer. If these are taken into account, the actual compression ratios become much smaller ($52\times$ and $79\times$), with models of sizes 4.7MB and 3.1MB respectively (see left of Fig.~\ref{fig:qp_alexnet_low_rank}). Our models outperform those and other results not only in terms of size, but also in terms of inference speed. We provide the runtime evaluation (when processing a single image) of our compressed models on a small edge device (NVIDIA's Jetson Nano) on the right bottom of Fig.~\ref{fig:qp_alexnet_low_rank}.

\section{Conclusion}

We have argued for and experimentally demonstrated the benefits of applying multiple compressions as an additive combination. We achieve this via a general, intuitive formulation of the optimization problem via constraints characterizing the additive combination, and an algorithm that can handle \emph{any choice of compression combination} as long as each individual compression can be solved on its own. In this context, pruning takes the meaning of adding a few elementwise corrections where they are needed most. This can not only complement existing compressions such as quantization or low-rank, but also be an interesting way to learn skip connections in deep net architectures. With deep neural nets, we observe that we can find significantly better models in the error-compression space, indicating that \emph{different compression types have complementary benefits}, and that the best type of combination depends exquisitely on the type of neural net. The resulting compressed nets may also make better use of the available hardware. Our codes and models are available at \url{https://github.com/UCMerced-ML/LC-model-compression} as part of LC Toolkit \cite{IdelbayCarreir20b}.

Our work opens up possibly new and interesting mathematical problems regarding the best approximation of a matrix by X, such as when X is the sum of a quantized matrix and a sparse matrix. Also, we do not claim that additive combination is the only or the best way to combine compressions, and future work may explore other ways.

\paragraph{Acknowledgments}
We thank NVIDIA Corporation for multiple GPU donations.

\clearpage
\appendix
\appendixpage
This appendix contains the following material: a) the pseudocode for augmented Lagrangian version of our algorithm (page~\pageref{section:code})
b) description of how we define the reduction ratios of storage, additions, and multiplications (page~\pageref{section:metrics})
c) description and results of the experiment compressing biases (page~\pageref{section:biases})
d) full details of all reported experiments with extended tables and figures for the CIFAR10 (page~\pageref{section:exp_cifar10}) and ImageNet (page~\pageref{section:exp_imagenet}) datasets with description of Jetson Nano measurements (page~\pageref{sec:runtime})

\section{Pseudocode using augmented Lagrangian}
\label{section:code}
In the main paper, we gave pseudocode for our algorithm using Quadratic Penalty formulation, here we give the pseudocode using the augmented Lagrangian version when there are two additive combinations.

\begin{center}
\small
  \fbox{
    \begin{minipage}[c]{0.6\linewidth}
      \begin{tabbing}
        n \= n \= n \= n \= n \= \kill
        \underline{\textbf{input:}} \caja{t}{l}{training data, neural net architecture with weights \w } \\
        $\w \leftarrow \argmin_{\w}{ L(\w) }$ \` {\small\textsf{reference net}} \\
        $\btheta_1, \btheta_2 \leftarrow \argmin_{\btheta_1, \btheta_2} \norm{\w - \bDelta_1(\btheta_1) - \bDelta_2(\btheta_2)}^2$ \` {\small\textsf{init as in C step}} \\
        $\blambda  \leftarrow \0$ \` {\small\textsf{Lagrange multipliers}} \\
        \underline{\textbf{for}} $\mu = \mu_0 < \mu_1 < \dots < \infty$ \+ \\
        $\w \leftarrow \argmin_{\w}{ L(\w) + \smash{\frac{\mu}{2} \norm{\smash{\w - \bDelta_1(\btheta_1) -\bDelta_2(\btheta_2)} - \frac{1}{\mu} \blambda }^2} }$ \` {\small\textsf{L step}} \\[0.2em]
        \underline{\textbf{while}} alternation does not converge \+ \\
        $\btheta_1 \leftarrow \argmin_{\btheta_1} \norm{\left(\w-\bDelta_2(\btheta_2) - \frac{1}{\mu} \blambda \right) - \bDelta_1(\btheta_1)}^2$ \` {\small\raisebox{0pt}[0pt][0pt]{$\left.\rule{0cm}{0.98cm}\right\}\quad\,\,$}\textsf{C step}} \\
        $\btheta_2 \leftarrow \argmin_{\btheta_2} \norm{\left(\w - \bDelta_1(\btheta_1) - \frac{1}{\mu} \blambda \right) -\bDelta_2(\btheta_2)}^2$ \-\\
        
        $\blambda \leftarrow \blambda - \mu \left(\W - \bDelta_1(\btheta_1) -\bDelta_2(\btheta_2)\right)$ \` {\small\textsf{multipliers step}} \\
        \textbf{if} $\norm{\w - \bDelta_1(\btheta_1) -\bDelta_2(\btheta_2)}$ is small enough \textbf{then} exit the loop \- \\
        
        \underline{\textbf{return}} $\w, \btheta_1, \btheta_2$ 
      \end{tabbing}
    \end{minipage}
     }
\end{center}

\begin{section}{Metrics}
\label{section:metrics}
\subsection{Compression ratio, storage}
\label{section:rho_storage}
We define a storage \emph{compression ratio} ($\rho_\text{s}$) as a ratio between bits required to store a reference model over bits of a compressed model:
\begin{equation}
\rho_\text{s} = \frac{\bits (\text{reference})}{\bits (\text{compressed})}
\end{equation}
In this paper, we assume that the weights of a reference model are stored as 32-bit IEEE floating point numbers, therefore, the total bits for storing the reference model is
\begin{equation}
\bits (\text{reference}) = \params \times 32,
\end{equation}
here $\params$ is the total number of parameters in the reference network. We discuss how we compute storage bits of compressed models, $\bits (\text{compressed})$, next.

\subsubsection{Quantization}

When we quantize the weights with a codebook of size $k$, we need to store codebook itself and $\ceil{\log k}$ bits for every weight to index them from the codebook. In our experimental setup, we define a separate codebook for every layer with the same size $k$. The total compressed storage then:
\begin{equation}
\texttt{bits}_Q (\text{compressed}) = \underbrace{\layers \times \, k  \cdot 32}_{\text{codebook bits}} \quad + \quad  \underbrace{\params \times \ceil{\log k}}_{\text{index bits}}
\end{equation}

\subsubsection{Sparse corrections (Pruning)}

In additive combination sense, un-pruned weights are point-wise corrections. We store such corrections separately for each layer as a list of index-value pairs, where index is the location of a correction in a vector of flattened weights. Instead of storing indexes directly, we adopt storing differences between subsequent indexes, as in \citep{Han_16a}, using unsigned $p$-bit integers. If a difference is larger than $2^p-1$, we add a dummy pair of zeros, i.e., $(0,0)$; with such scheme storing some corrections would require multiple index-value pairs. We choose to store the values of the corrections as 16-bit IEEE floating point numbers. Then, the compressed storage for a layer is
\begin{equation}
\bits_P(\text{layer}) = \texttt{diffs} \times (p + 16).
\end{equation}

\subsubsection{Low-rank}

Rank $r$ matrix of shape $m\times n$, is stored with two matrices of shapes $m\times r$ and $r \times m$ using 16-bit floating point numbers. The compressed storage for such layer is
\begin{equation}
\bits_L(\text{layer}) = 16 \times r \times (m+n),
\end{equation}
and the total bits for compressed storage is the sum over all layers.
\end{section}

\subsubsection{Uncompressed parts}

Since we only compress weights of a layer, some additional structures like biases and batch-normalization parameters would not be compressed. Therefore, we store them in full precision and add their storage cost to the total of compressed bits. In some cases, we can reduce the storage by fusing elements, e.g.,\ biases can be fused with batch-normalization layer preceding it, which we take advantage of.

\subsubsection{Additive combination}

When compression combined additively, the total compressed storage is the sum of compression parts over all layers, e.g.,\ for quantization with point-wise corrections the total compressed storage is:
\begin{equation}
\bits(\text{compressed})=\sum_{ l} \,\Big(\bits_P(\text{layer } l) + \bits_Q(\text{layer } l)\Big)
\end{equation} 

\subsection{Multiplications and additions, efficient implementation}
\label{section:efficient_impl}
We define reduction ratio of the number of floating point additions ($\rho_+$) and floating point multiplications ($\rho_\times$) as a ratio between the number of additions (and respectively multiplications) in reference model over the number of additions (multiplications) in a compressed model:
\begin{align}
\rho_+ &= \frac{\text{\texttt{\num add} in reference}}{\text{\texttt{\num add} in compressed}} \\
\rho_\times &= \frac{\text{\texttt{\num mult} in reference}}{\text{\texttt{\num mult} in compressed}}
\end{align}

\subsubsection{FLOPs for uncompressed model}

A fully connected layer with weights \W\ of shape $n\times m$ and biases \b\ of shape $n\times 1$ requires the following number of multiplications and additions to compute a result of $\W\x + \b$:
\begin{align*}
\texttt{\num add} & = n\times(m-1) + n = nm\\
\texttt{\num mult} & = nm.
\end{align*}
For a convolutional layer with shape $n\times c\times d\times d$ (here $n$ filters with $c$ channels and spatial resolution $d\times d$) and biases of size $n$, the total number of multiplications and additions are:
\begin{align*}
\texttt{\num add} & = ncd^2 \times M,\\
\texttt{\num mult} & = ncd^2 \times M,
\end{align*}
where $M$ is the total number of one convolutional filter being applied to the input image.

\subsubsection{Quantization}

To perform efficient inference with quantized weights, for each neuron with weight vector $\w$ we need to maintain an accumulator corresponding to each value of the codebook $\calC = \{c_1 \dots c_k \}$: \begin{equation}
\w^T \x = \sum_i w_i x_i = \sum_i c(w_i) x_i = \sum_k c_k \underbrace{\sum_{c(w_i)=c_k} x_i}_{\text{accumulate}}
\end{equation} 
Therefore for each neuron, there are $k$ multiplications and the number of additions is equal to the number of the weights connected to the neuron. In the case of fully connected layers with $m$ inputs and $n$ outputs, there are following number of add./mult.: \begin{align*}
  &\texttt{\num mult} = k\times n \\
  &\texttt{\num add} = m\times n
\end{align*} For convolutional layers with $n$ filters of shape $c\times d \times d$, we respectively have:
\begin{align*}
  &\texttt{\num mult} = k\times n \times M, \\
  &\texttt{\num add} = ncd^2 \times M,
\end{align*} where $M$ is the total number of the applications of conv.\ filters to the input.

\subsubsection{Pruning}

When the weight matrix $\W$ has only $p$ non-zero items, it will require $p$ multiplications and $p-1$ additions for matrix-vector product; for convolutional layer these numbers should be multiplied by number of applications of conv filter, $M$ (see above).

\subsubsection{Low-rank}

A rank $r$ matrix of shape $n\times m$ can be represented by two matrices of shapes $n\times r$ and $r\times m$, therefore, during the inference of $\W\x + \b$  we have $\texttt{\num mult} = \texttt{\num add} = r(n+m)$.

\section{Experiments on CIFAR10}
\label{section:exp_cifar10}
In this section, we give full details of our experimental setup for the CIFAR10 networks reported in the paper, as well as extended analysis of the results.
\subsection{Quantization plus pruning, Q+P}
\label{section:q+p}
\subsubsection{Compressing biases together with weights}
\label{section:biases}
We would like to verify whether point-wise corrections are going to recover bad compression decisions. One such decision is to quantize both weights and biases with a single codebook. We report the results of such compression in Table~\ref{t:compressed_biases}. As you can see with few corrections (say, 0.65\%) most of the biases are ``recovered'' from the bad quantization decision. Below we give a full experimental setup.
\begin{description}
\item[Reference model] We train a multinomial logistic regression classifier on the CIFAR10 dataset (60k color images of $32\times32$ pixels, 10 classes). We use Nesterov's SGD with a batch size of 1024 and learning rate of 0.05 which decayed after every epoch by 0.98; we use momentum of 0.9 and run the training for 300 epochs. We preprocess images by standardizing pixel-wise means and variances. The model has 30730 weights ($3072\times10$ weights and 10 biases), and achieves train loss of 1.5253 and test accuracy of 38.79\%. 
\item[Compression setup] We run our algorithm with quantization and point-wise corrections on the weights and biases jointly. The codebook has $k=2$ entries and we vary the number of corrections. The algorithm runs for 50 LC steps with $\mu=\times 5 \times 10^{-4}\times 1.1^k$ at $k$-th step. Each  L-step is performed by Nesterov's SGD with momentum 0.9 and run for 20 epochs with the initial learning rate of $0.05$ decayed by 0.98 after each epoch. We do not perform any finetuning. Running the LC algorithm is 3.33 times longer comparing to the training of the reference model. We report results (loss and accuracies) corresponding to the model with the smallest train loss seen during the training, which is usually on the last iteration. For each C-step, we alternate between Q and P compressions 30 times.
\end{description}

\begin{table}[t]
\centering
\begin{tabular}{@{}ll@{ + }lcrcrrr@{}}
\toprule
 &\multicolumn{2}{c}{Model} & $L$ & test acc, \% &corrected biases (out of 10)\\
   \midrule
 &\multicolumn{2}{l}{\textbf{R} logistic regression}& \textbf{1.5253} & \textbf{38.79}                 & &\\
 &1-bit quant. & 100 corrections & 1.6830 & 36.06                 & 9 \\
 &1-bit quant. & 200 corrections & 1.6716 & 37.29                 & 9 \\
\bottomrule
\end{tabular}
\caption{Results of running 1-bit quantization with corrections on simple multinomial logistic regression model (\textbf{R}) trained on the CIFAR10, where both weights and biases are compressed jointly with a single codebook. We report train loss $L$, test accuracy, and the number of corrections acting on a total of 10 biases. When we allow correcting only 100 (0.16\%) values 9 out of 10 biases are already corrected. This indeed confirms that a) biases are important, and compressing them requires higher precision b) point-wise corrections are able to ``fix" bad compression decisions. }
\label{t:compressed_biases}
\end{table}

\subsubsection{ResNet experiments}
\label{section:resnet_qp}

We train ResNets \cite{He_16a} of depth 20, 32, 56, and 110 layers (0.27M, 0.46M, 0.85M, and 1.7M parameters, respectively) on the CIFAR10 dataset using the same augmentation setup as in \cite{He_16a}. Images in the dataset are normalized to have channel-wise zero mean, variance 1. For training, we use a random horizontal flip, zero pad the image with 4 pixels on each side and randomly crop $32\times32$ image out of it. For test, we use normalized images without augmentation. We report results obtained at the end of the training. The loss is average cross entropy with weight decay (as in the original paper).
\begin{description}
\item[Training reference nets] The models are trained with Nesterov's SGD \citep{Nester83a} with a momentum of 0.9 on the minibatches of size 128. The loss is average cross entropy with weight decay of $10^{-4}$; weights initialized following \cite{He_15a}. Each reference network is trained for 200 epochs with a learning rate of 0.1 which is divided by 10 after 100 and 150 epochs.
\item[Training compressed models] We run our algorithm for 50 LC iterations on the ResNet-20/32, and for 45 iterations on the ResNet-56/110, with $\mu=\times 10^{-3}\times 1.1^k$ at $k$-th step. Each L-step is performed by Nesterov's SGD with a momentum of 0.9 and runs for 20 epochs with a learning rate of $0.01$ at the beginning of the step and decayed by 0.94 after each epoch. We do not perform any finetuning. Running the LC algorithm is 5 times longer comparing to the training of the reference network. We report results (train loss and test error) corresponding to the networks with the smallest train loss seen during the training, which is usually the last iteration. For each C-step, we alternate between Q and P compressions 30 times.
\end{description}

In our compression setup, each layer is quantized with its own codebook of size 2, however, corrections are applied throughout, to all weights with a predefined $\kappa$. We quantize only weights, and not the biases. For ResNet 20 and 32, we chose $\kappa$ values to be 1, 2, 3, 5 \% of the total number of parameters. For ResNet 56 and 110, we chose $\kappa$ values to be 0.5, 1, 2, 3 \% accordingly. The results are given in Table~\ref{table:q_p}, and we compare our results to others in Fig.\ref{fig:qp_plots}. We additionally plot how pointwise corrections affect the weight of each layer in Fig.~\ref{fig:qp_sprs}.

\begin{table}[t]
\centering
\begin{tabular}{@{}ll@{ + }lccrrr@{}}
\toprule
 &\multicolumn{2}{c}{Model} & $\log L$ & $E_\text{test}$, \% & $\rho_\text{s}$ & $\rho_+$ & $\rho_\times$ \\
   \midrule
 &\multicolumn{2}{l}{\textbf{R}}  &-0.80& 8.35                 & 1.00           & \textbf{1.00} & 1.00 \\
 &1-bit quant. & 1.0\% correction  &-0.84& 9.16                 & \textbf{22.67} & 0.97          & \textbf{30.74}\\ 
 \raisebox{0pt}[0pt][0pt]{\rotatebox{90}{\makebox[0pt][c]{  ResNet20 }}}
 &1-bit quant. & 2.0\% correction  &-0.92& 8.92                 & 19.44         & 0.96           & 19.74\\
 &1-bit quant. & 3.0\% correction  &-0.93& 8.31                 & 17.08         & 0.94           & 15.80 \\
 &1-bit quant. & 5.0\% correction  &-0.99& \textbf{8.26}        & 13.84         & 0.92           & 11.54 \\
  \midrule
 &\multicolumn{2}{l}{\textbf{R}}  &-0.82& 7.14        & 1.00           & \textbf{1.00} & 1.00 \\
 &1-bit quant. & 1.0\% correction  &-1.03& 7.57                 & \textbf{22.81} & 0.97          & \textbf{30.52}\\ 
 \raisebox{0pt}[0pt][0pt]{\rotatebox{90}{\makebox[0pt][c]{  ResNet32 }}}
 &1-bit quant. & 2.0\% correction  &-1.07& 7.61                 & 19.54         & 0.96           & 19.85\\
 &1-bit quant. & 3.0\% correction  &-1.10& 7.29                 & 17.14         & 0.94           & 15.80 \\
 &1-bit quant. & 5.0\% correction  &-1.14& \textbf{7.09}        & 13.84         & 0.92           & 11.56 \\
 \midrule
 &\multicolumn{2}{l}{\textbf{R}}  &-0.81& 6.58                 & 1.00           & \textbf{1.00} & 1.00 \\
 &1-bit quant. & 0.5\% correction  &-1.08& 6.77                 & \textbf{25.04} & 0.98          & \textbf{49.79}\\ 
 \raisebox{0pt}[0pt][0pt]{\rotatebox{90}{\makebox[0pt][c]{  \hspace{1em} ResNet56 }}}
 &1-bit quant. & 1.0\% correction  &-1.13& 6.73                 & 22.87         & 0.97           & 32.04\\
 &1-bit quant. & 2.0\% correction  &-1.17& 6.70                 & 19.55         & 0.96           & 20.46\\
 &1-bit quant. & 3.0\% correction  &-1.18& \textbf{6.23}        & 17.11         & 0.94           & 15.98 \\
 \midrule
 &\multicolumn{2}{l}{\textbf{R}}  &-0.77& 6.02                 & 1.00           & \textbf{1.00} & 1.00 \\
 &1-bit quant. & 0.5\% correction  &-1.16& 6.20                 & \textbf{25.03} & 0.99          & \textbf{55.63}\\ 
 \raisebox{0pt}[0pt][0pt]{\rotatebox{90}{\makebox[0pt][c]{  \hspace{1em} ResNet110 }}}
 &1-bit quant. & 1.0\% correction  &-1.20& 5.80                 & 22.80         & 0.98           & 33.94\\
 &1-bit quant. & 2.0\% correction  &-1.23& 5.66                 & 19.47         & 0.96           & 23.27 \\
 &1-bit quant. & 3.0\% correction  &-1.25& \textbf{5.58}       & 17.04         & 0.95           & 17.84 \\
\bottomrule
\end{tabular}
\caption{Results of running 1-bit quantization with corrections on the ResNets (\textbf{R}-s) of different depth. We report test error, the reduction ratio of storage($\rho_\text{s}$), floating point additions ($\rho_+$), and multiplications ($\rho_\times$), under assumption of efficient implementation (section \ref{section:efficient_impl}); logarithms are base 10. We trade-off additions to multiplications, e.g., for ResNet110 with 3\% corrections, the number of multiplications is decreased by factor of $17$ with only a 5\% increase in the number of additions. Also, notice that with a higher percentage of the corrections the compressed model achieves better than reference test error.}
\label{table:q_p}
\end{table}

\begin{figure}
\psfrag{resnet20}[l][l]{ResNet20}
\psfrag{resnet32}[l][l]{ResNet32}
\psfrag{resnet56}[l][l]{ResNet56}
\psfrag{resnet110}[l][l]{ResNet110}
\psfrag{qpchoi}[l][l]{~\small P$\rightarrow$Q$\rightarrow$H.C \cite{Choi_17a}}
\psfrag{agusts17}[l][l]{~\small Q$\rightarrow$H.C \cite{Agusts_17a}}
\psfrag{qzhu17}[l][l]{~\small Q \cite{Zhu_17a}}
\psfrag{yin18b}[l][l]{~\small Q \cite{Yin_18b}}
\psfrag{cp}[l][l]{~\small P \cite{CarreirIdelbay18a}}
\psfrag{err}[B][]{test error $E_{\text{test}}$, \%}
\psfrag{loss}[B][]{loss L}
\psfrag{CR}[c][b]{compression ratio}
\begin{tabular}{@{}r@{\hspace{0.04\linewidth}}r@{}}
\includegraphics[height=0.41\linewidth]{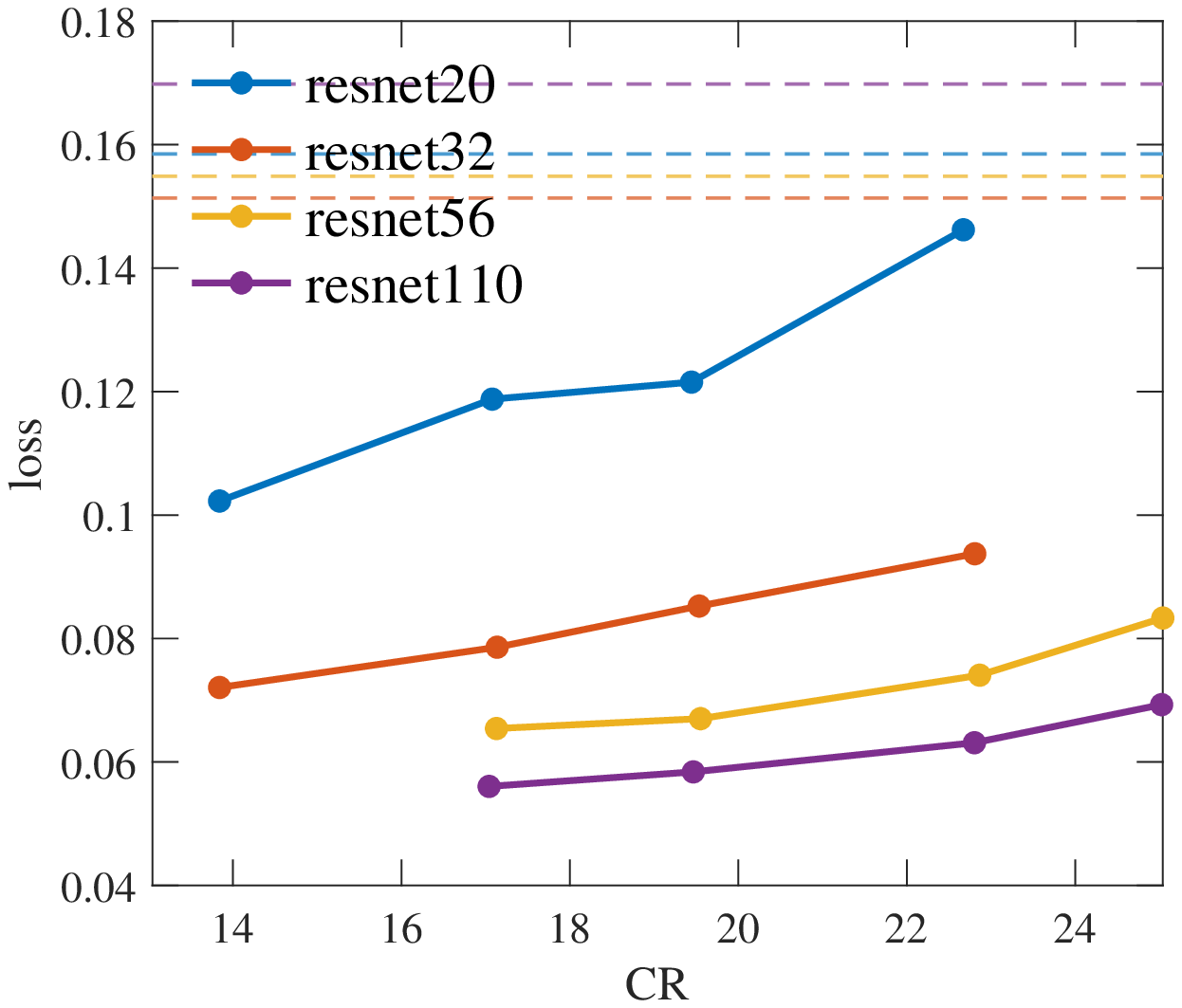} &
\includegraphics[height=0.41\linewidth]{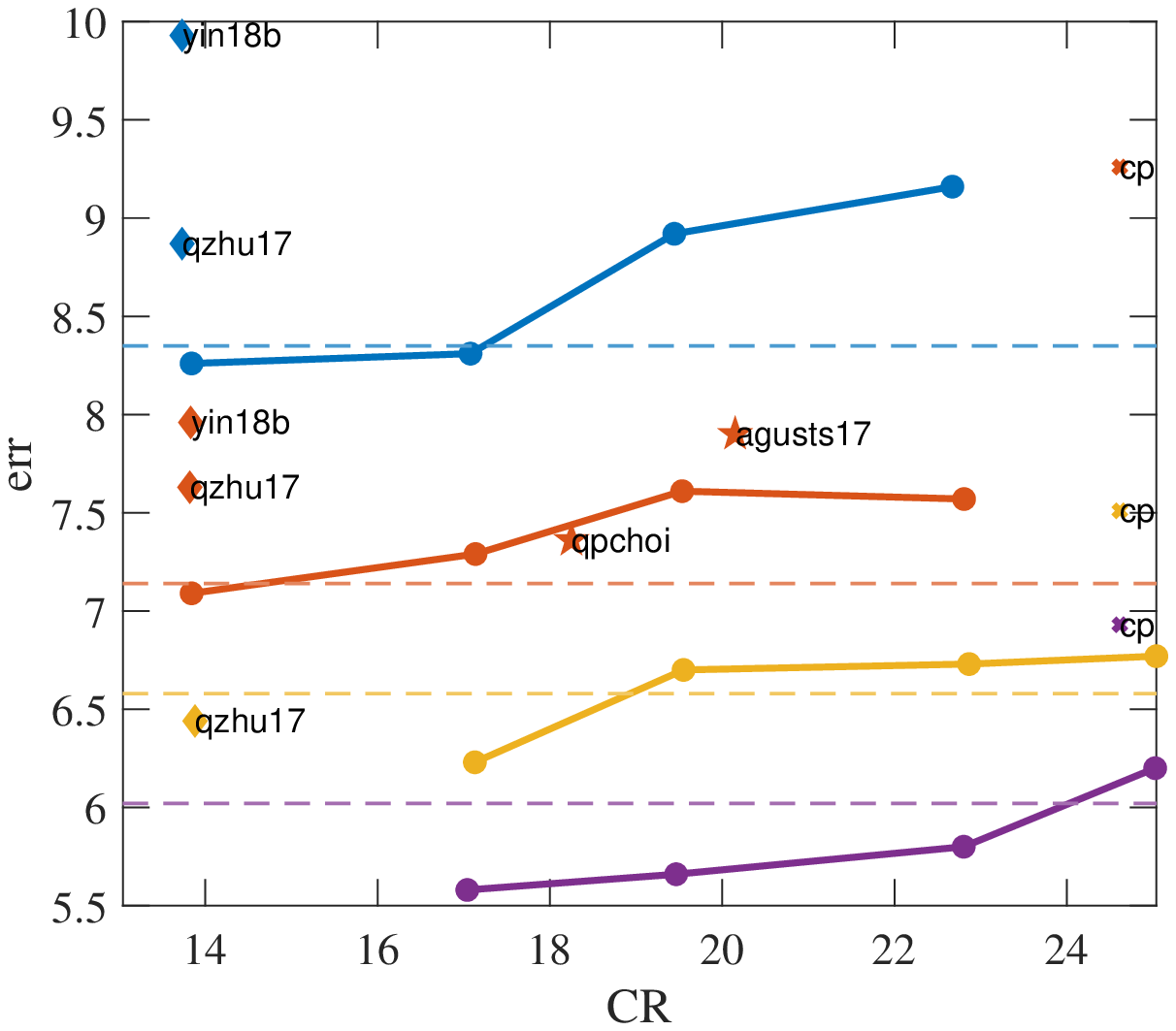}
\end{tabular}
\caption{Results of running 1-bit quantization with corrections on ResNets of different depth, we plot train loss and test error as a function of compression (See Table~\ref{table:q_p} for more detail). Here, thick lines --- our results, horizontal thin lines --- reference nets. We also additionally indicate (via markers) results from the literature: Q --- quantization, P---pruning, H.C --- Huffman coding, right arrow indicates nesting the compression via ``add'' combination.}
\label{fig:qp_plots}
\end{figure}

\begin{figure}
\centering
\psfrag{corrections}[B][]{correction, \%}
\psfrag{layers}[t][t]{layers}
\begin{tabular}{@{}c@{\hspace{2em}}c@{}}
ResNet20 & ResNet32 \\
\psfrag{low}[l][l]{1\% corr.}
\psfrag{high}[l][l]{5\% corr.}
\includegraphics[clip,height=0.42\linewidth]{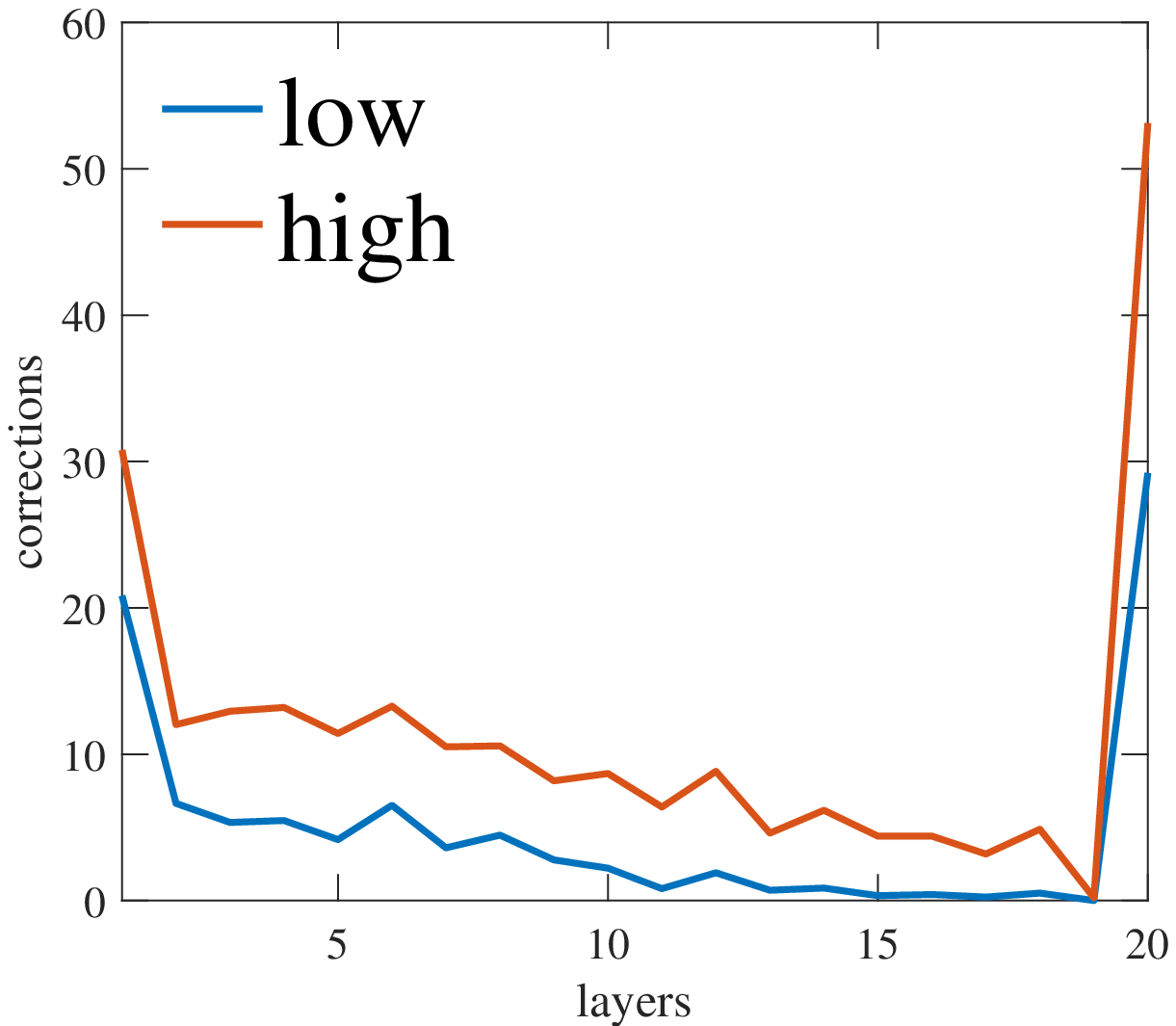} &
\psfrag{low}[l][l]{1\% corr.}
\psfrag{high}[l][l]{5\% corr.}
\includegraphics[clip,height=0.42\linewidth]{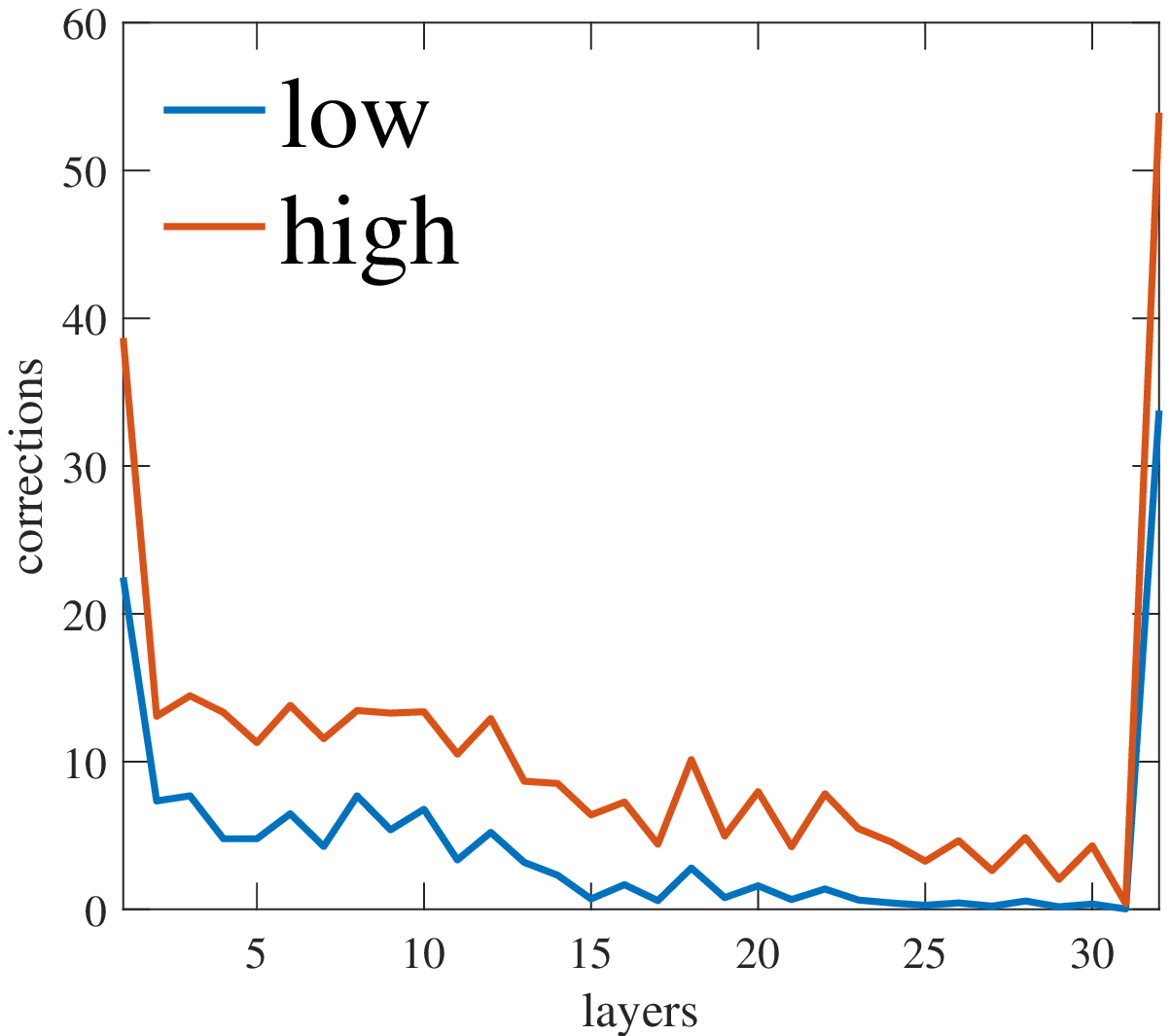} \\[2ex]
ResNet56 & ResNet110 \\[0.5em]
\psfrag{low}[l][l]{0.5\% corr.}
\psfrag{high}[l][l]{3\% corr.}
\includegraphics[clip,height=0.42\linewidth]{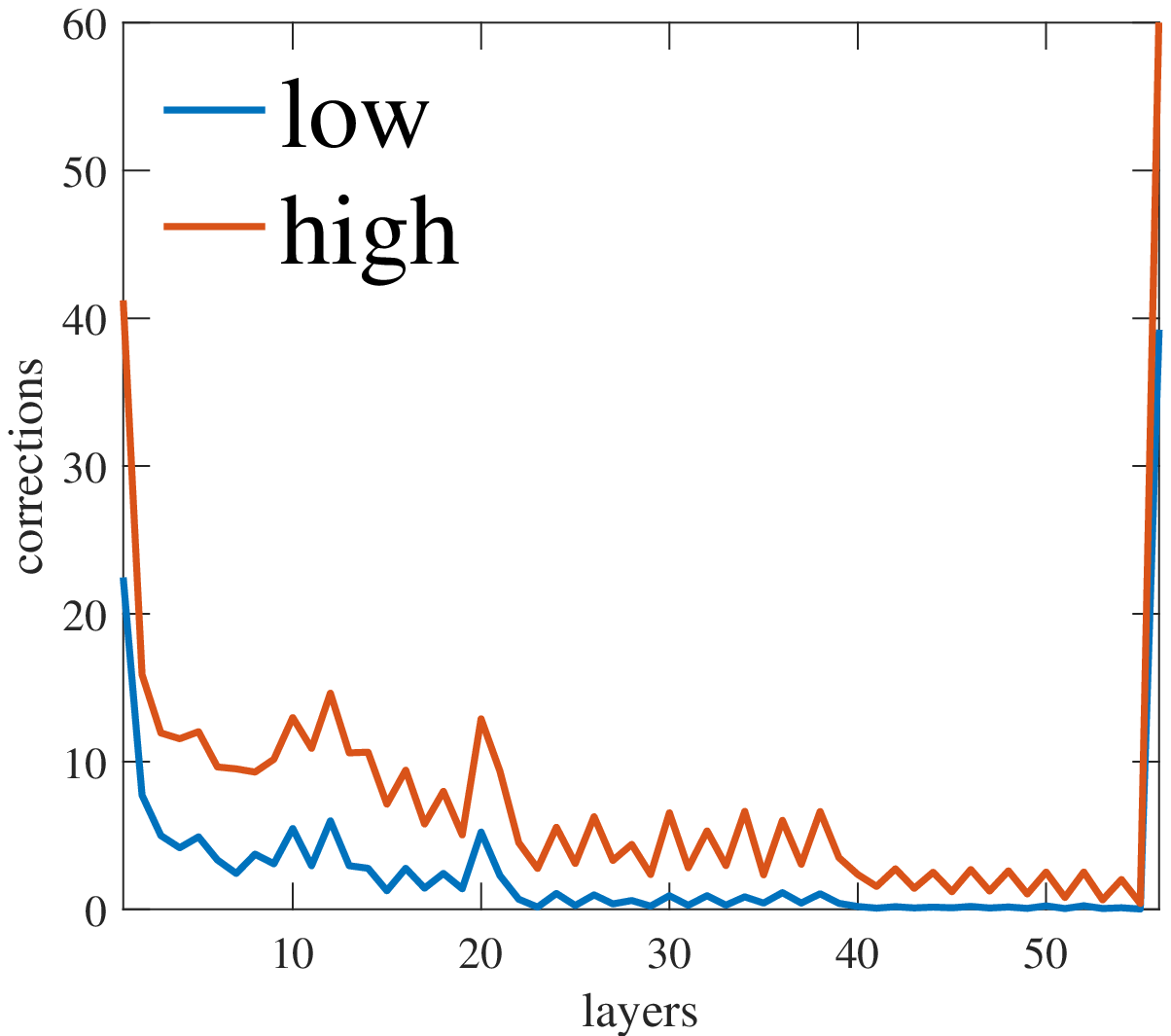} &
\psfrag{low}[l][l]{0.5\% corr.}
\psfrag{high}[l][l]{3\% corr.}
\includegraphics[clip,height=0.42\linewidth]{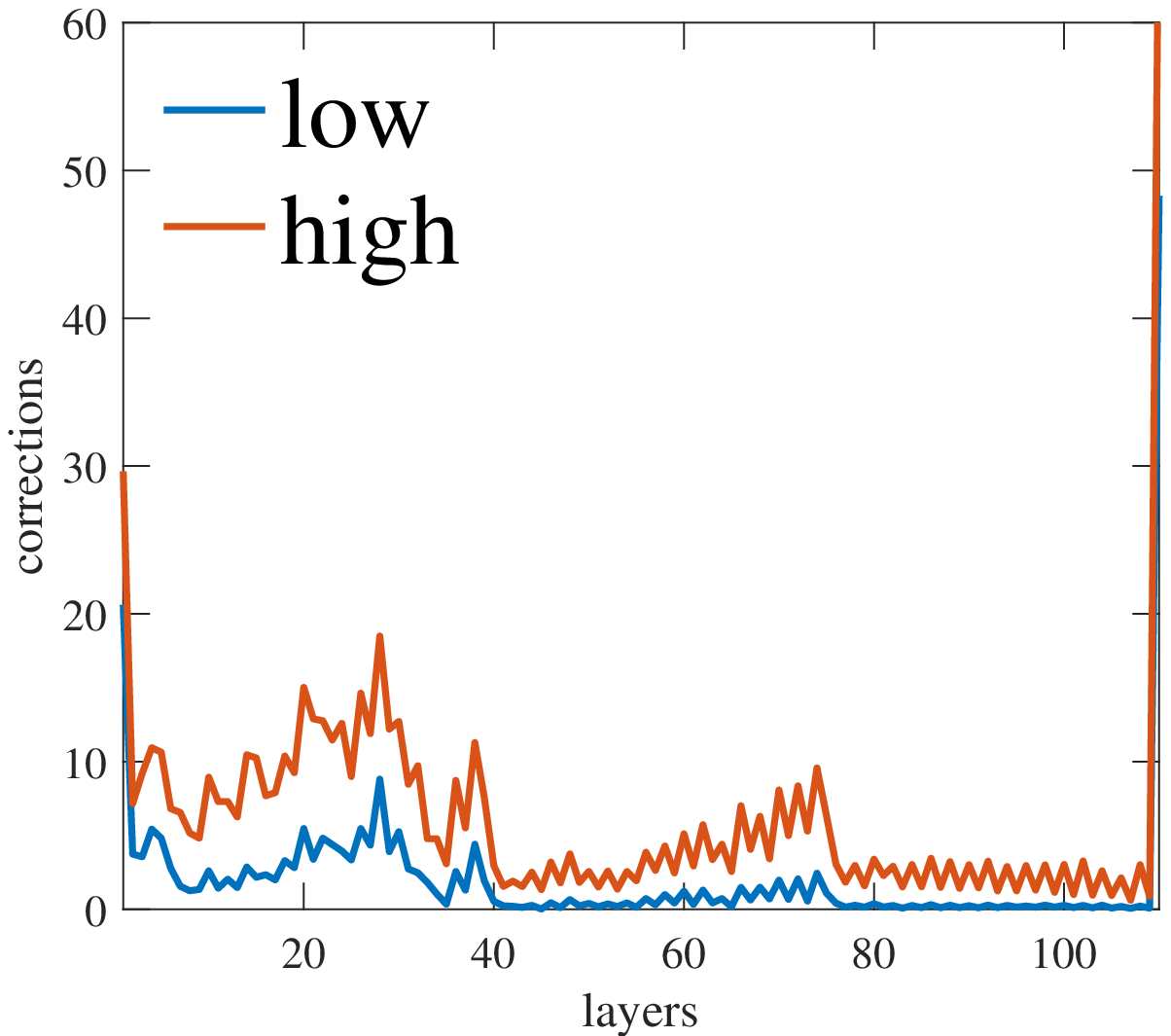} 
\end{tabular}
\caption{For every ResNet compressed with Q+P, we plol the proportion of weight being corrected at every level when the different amount of corrections are allowed. We see that the most corrected layers are the first and the last ones, which coincide with heuristics in the literature of not compressing the first and last layers.}
\label{fig:qp_sprs}
\end{figure}

\subsection{Quantization plus low-rank, Q+L}
\label{section:q+l}
We run quantization with low-rank corrections on ResNet-s of depth 20, 32, 56, and 110. Each layer is quantized with its own codebook of size 2 applied to weights and corrected with $r$-rank matrices. We set $r$ to have the same value throughout all layers.

We follow the procedure for the reference and compressed net training as described in section~\ref{section:resnet_qp} with minor changes. We use 45 LC iterations for all networks, however, learning rates are different: for ResNet20 we set learning rate at the beginning of each step to be 0.01 decayed by 0.94 after every epoch, and for other ResNet-s the initial learning rate for every L-step is 0.007 decayed by 0.94 after every epoch. The results are given in Table~\ref{table:q_lr}. We describe the low-rank parameterization of every convolutional layer next.

\subsubsection{Low-rank parametrization of convolution}

A convolutional layer having $n$ filters of shape $c\times d\times d$ (here $c$ is a number of channels and $d\times d$ is spatial resolution), can be seen as a linear layer with shape $n \times cd^2$ applied to the reshaped volumes of input. Its rank-$r$ parametrization will have two linear mappings with weights $n\times r$ and $r\times cd^2$, which can be efficiently implemented as a sequence of two convolutional layers: the first with $r$ filters of shape $c\times d \times d$ and the second with $n$ filters of shape $r\times 1\times 1$.

\begin{table}[t]
\centering
\begin{tabular}{@{}ll@{ + }lccrrr@{}}
\toprule
 &\multicolumn{2}{c}{Model} & $\log L$& $E_\text{test}$, \% & $\rho_\text{s}$ & $\rho_+$ & $\rho_\times$ \\
   \midrule 
   &\multicolumn{2}{l}{\textbf{R}}      &-0.80& \textbf{8.35}        & 1.00           & \textbf{1.00} & 1.00 \\
 &1-bit quant. & $\rankop=1$   & -0.77 & 9.71                 & \textbf{20.71} & 0.96          & \textbf{21.45}\\ 
 \raisebox{0pt}[0pt][0pt]{\rotatebox{90}{\makebox[0pt][c]{\hspace{1.5em}  ResNet20 }}}
 &1-bit quant. & $\rankop=2$   & -0.84 & 9.30                 & 16.62         & 0.92           & 11.26\\
 &1-bit quant. & $\rankop=3$   & -0.89 & 8.64                 & 13.88         & 0.89           & 7.64 \\
 \midrule
 &\multicolumn{2}{l}{\textbf{R}}        &-0.82& \textbf{7.14}        & 1.00           & \textbf{1.00} & 1.00 \\
 &1-bit quant. & $\rankop=1$   & -0.99 & 7.90                 & \textbf{20.94} & 0.97          & \textbf{21.89}\\ 
 \raisebox{0pt}[0pt][0pt]{\rotatebox{90}{\makebox[0pt][c]{\hspace{1.5em}  ResNet32 }}}
 &1-bit quant. & $\rankop=2$   & -1.04 & 8.06                 & 16.81         & 0.92           & 11.47\\
 &1-bit quant. & $\rankop=3$   & -1.10 & 7.52                 & 14.04         & 0.89           & 7.44 \\
 \midrule
 &\multicolumn{2}{l}{\textbf{R}}        &-0.81& 6.58                 & 1.00           & \textbf{1.00} & 1.00 \\
 &1-bit quant. & $\rankop=1$   & -1.13 & 7.19                 & \textbf{21.04} & 0.96          & \textbf{22.19}\\ 
 \raisebox{0pt}[0pt][0pt]{\rotatebox{90}{\makebox[0pt][c]{  \hspace{1.5em} ResNet56 }}}
 &1-bit quant. & $\rankop=2$   & -1.19 & 6.51                 & 16.91         & 0.92           & 11.61\\
 &1-bit quant. & $\rankop=3$   & -1.22 & \textbf{6.29}        & 14.10         & 0.89           &  7.87 \\
 \midrule
 &\multicolumn{2}{l}{\textbf{R}}        &-0.77& 6.02                 & 1.00           & \textbf{1.00} & 1.00 \\
 &1-bit quant. & $\rankop=1$   & -1.19 & 5.98                 & \textbf{21.11} & 0.96          & \textbf{22.38}\\ 
 \raisebox{0pt}[0pt][0pt]{\rotatebox{90}{\makebox[0pt][c]{  \hspace{1.5em} ResNet110 }}}
 &1-bit quant. & $\rankop=2$   & -1.24 & 5.93                 & 16.96         & 0.92           & 11.70\\
 &1-bit quant. & $\rankop=3$   & -1.27 & \textbf{5.50}        & 14.18         & 0.89           & 7.92 \\
\bottomrule
\end{tabular}
\caption{Results of running 1-bit quantization plus low-rank on ResNet-s (\textbf{R}-s) of different depth.  We report test error, the reduction ratio of storage ($\rho_\text{s}$), floating point additions ($\rho_+$) and multiplications ($\rho_\times$), under assumption of efficient implementation (section~\ref{section:efficient_impl}); logarithms are base 10. Notice how on the 110-layers ResNet, all compressed models achieve better than the reference test error.}
\label{table:q_lr}
\end{table}

\begin{figure}
\psfrag{resnet20}[l][l]{ResNet20}
\psfrag{resnet32}[l][l]{ResNet32}
\psfrag{resnet56}[l][l]{ResNet56}
\psfrag{resnet110}[l][l]{ResNet110}
\psfrag{qpchoi}[l][l]{~\small P$\rightarrow$Q$\rightarrow$H.C \cite{Choi_17a}}
\psfrag{agusts17}[l][l]{~\small Q$\rightarrow$H.C \cite{Agusts_17a}}
\psfrag{qzhu17}[l][l]{~\small Q \cite{Zhu_17a}}
\psfrag{wen17}[l][l]{~\small L \cite{Wen_17a}}
\psfrag{yin18b}[l][l]{~\small Q \cite{Yin_18b}}
\psfrag{err}[B][]{test error $E_{\text{test}}$, \%}
\psfrag{loss}[B][]{loss L}
\psfrag{CR}[c][b]{compression ratio}
\begin{tabular}{@{}r@{\hspace{0.04\linewidth}}r@{}}
\includegraphics[height=0.41\linewidth]{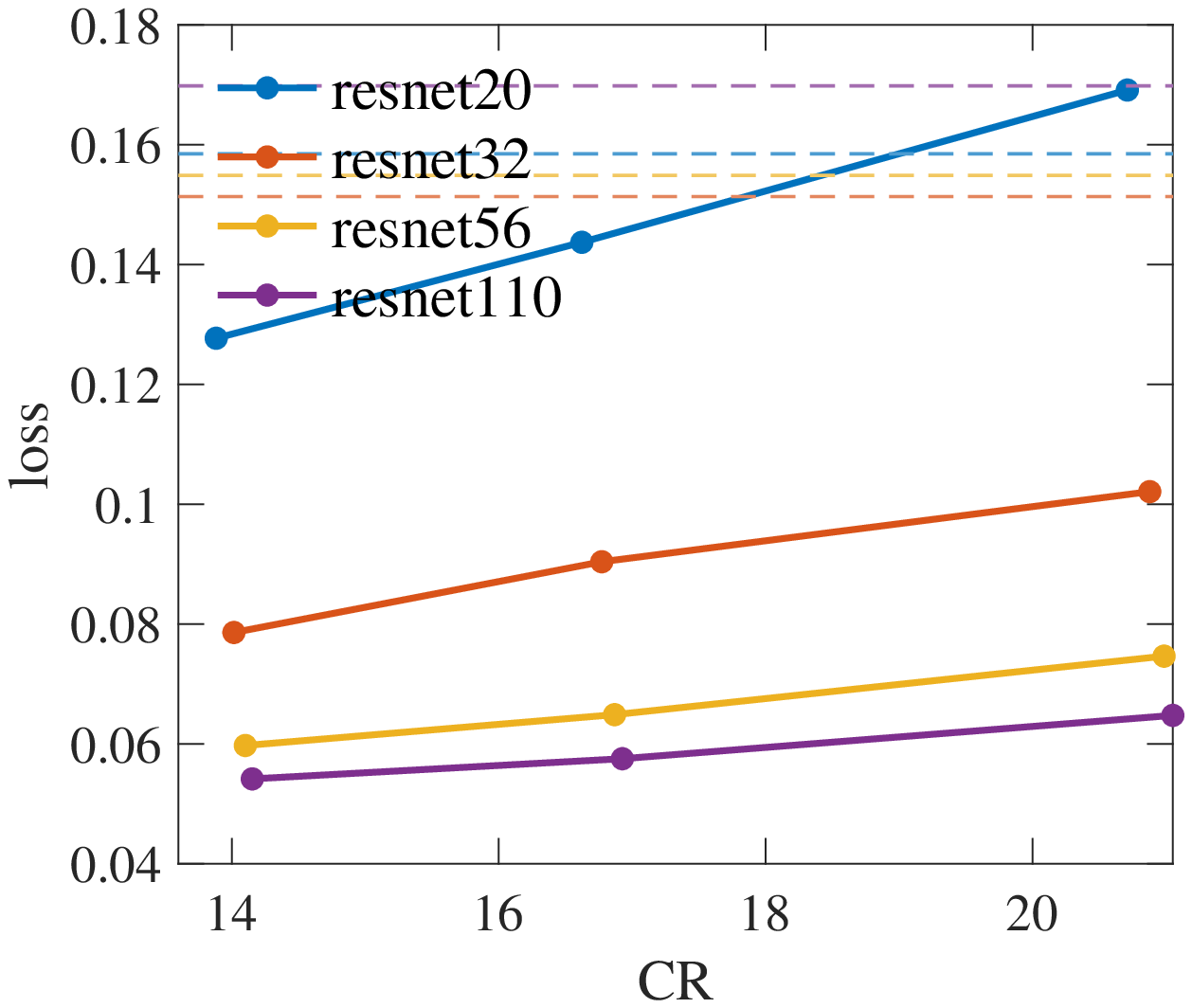} &
\includegraphics[height=0.41\linewidth]{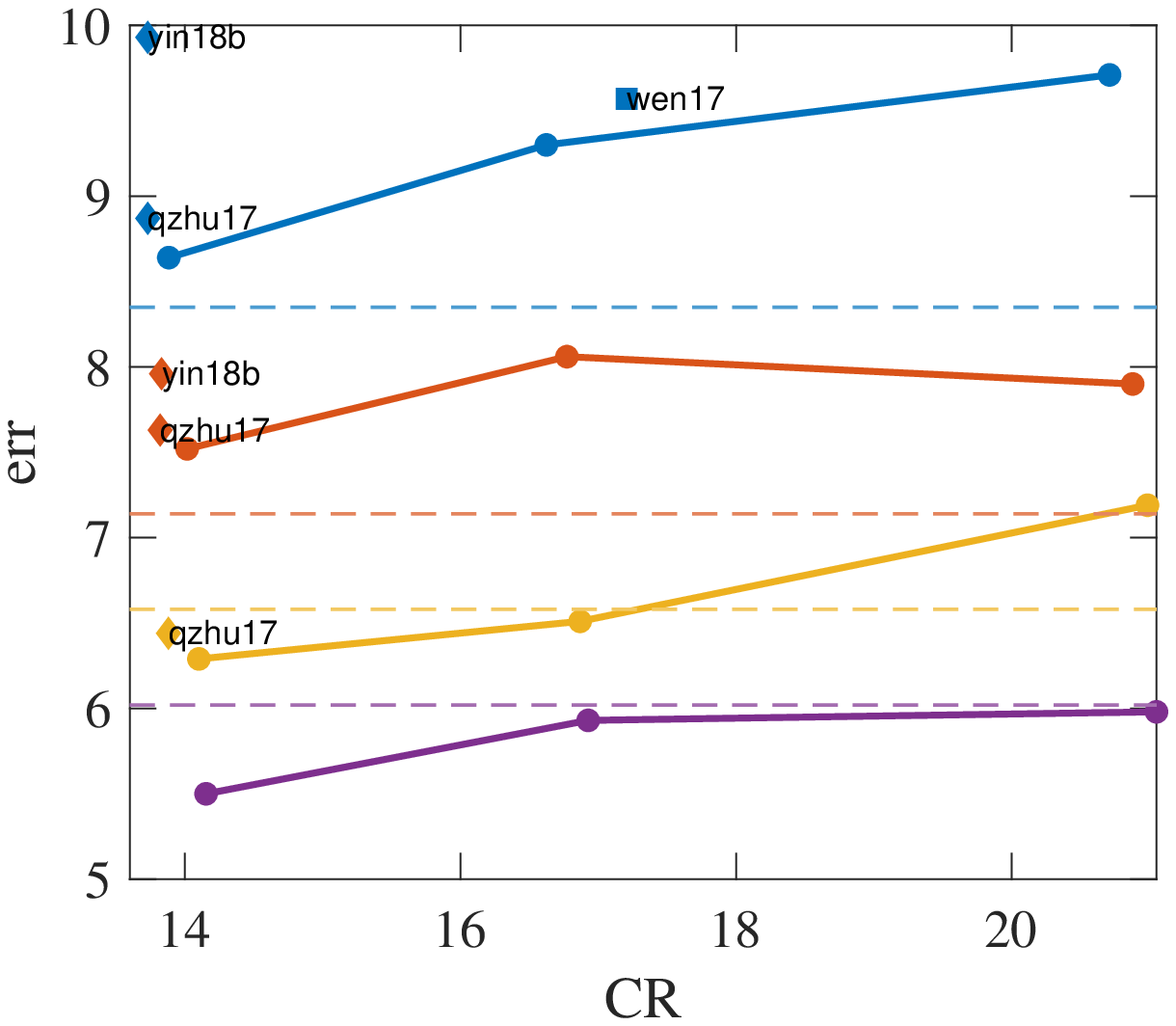}
\end{tabular}
\caption{Results of running 1-bit quantization with corrections on ResNets of different depth, we plot train loss and test error as a function of compression (See Table~\ref{table:q_p} for more detail). Here, thick lines are our results, and horizontal thin lines are reference nets. We also additionally indicate (via markers) results from the literature involving quantization (Q) and low-rank (L).}
\label{fig:ql_plots}
\end{figure}

\subsection{Low-rank plus pruning, L+P}
\label{section:l+p}
We compress using a low-rank plus pruning combination the ResNet-s on CIFAR10 of depth 20, 32, 56, and 110; and the VGG16 adapted for the CIFAR10. Each layer is compressed with $r$-rank matrices and point-wise corrections. We set $r$ to be the same throughout all layers; we set an amount of point-wise corrections (pruning) for the entire network in terms of percentage of a total number of weights. The results are in Table~\ref{table:p_lr}.

\subsubsection{ResNets}

We follow the procedure for reference and compressed net training as described in section~\ref{section:resnet_qp} with minor changes. We 45 LC iterations for all networks, with $\mu_k=5\times 10^{-4}\times 1.1^k$ at $k$-th iteration. The learning rate at the beginning of each L-step is 0.01 and decayed by 0.94 after every epoch; we run each L-step for 20 epochs. 

\subsubsection{VGG-16}

We train the VGG16 \citep{SimonyZisser15a} adapted for the CIFAR10 dataset. We employ batch normalization after every layer except the last, and dropout after fully connected layers (see Table~\ref{t:VGG16} for the full details). Images in the dataset are normalized channel-wise to have zero mean and variance one. For training, we use simple augmentation (random horizontal flip, zero pad with 4 pixels on each side and randomly crop $32\times32$ image). For test we use normalized images without augmentation. We report results corresponding to a model with the smallest loss. The loss is average cross entropy with $\ell_2$ weight decay. The resulting net has 15M parameters.
\begin{description}
\item[Training the reference net] The model is trained with Nesterov's accelerated SGD \citep{Nester83a} with momentum 0.9 on minibatches of size 128. The loss is average cross entropy with weight decay of $5\times 10^{-4}$. The network is trained for 300 epochs with an initial learning rate of 0.05 decayed by 0.97716 after every epoch. The resulting test error is 6.45\%.
\item[Training compressed models] Our algorithm is run for 50 LC iterations, with $\mu=5\times 10^{-4}\times 1.1^k$ at $k$-th iteration. Each L-step is performed by Nesterov's SGD with momentum of 0.9 and runs for 20 epochs with a learning rate of $0.0007
$ at the beginning of the step; and decayed by 0.98 after each epoch. We do not perform any finetuning. Running the LC algorithm with finetuning is 3.4 times longer comparing to the training of the reference network. The results are presented in Table~\ref{table:p_lr}.
\end{description}

\begin{table}[t]
\centering
  \begin{tabular}[t]{@{}ll@{}}
    \toprule
    Layer & Connectivity \\
    \midrule
    Input & $32 \times 32$ image \\
    1 & \caja{t}{l}{convolutional, 64 $3 \times 3$ filters (stride=1), 
      followed by BN and ReLU} \\
    2 & \caja{t}{l}{convolutional, 64 $3 \times 3$ filters (stride=1), 
      followed by BN and ReLU} \\
      & \caja{t}{l}{max pool, $2 \times 2$ window (stride=2)} \\
    3 & \caja{t}{l}{convolutional, 128 $3 \times 3$ filters (stride=1), 
      followed by BN and ReLU} \\
    4 & \caja{t}{l}{convolutional, 128 $3 \times 3$ filters (stride=1), 
      followed by BN and ReLU} \\
      & \caja{t}{l}{max pool, $2 \times 2$ window (stride=2)} \\
    5 & \caja{t}{l}{convolutional, 256 $3 \times 3$ filters (stride=1), 
      followed by BN and ReLU} \\
    6 & \caja{t}{l}{convolutional, 256 $3 \times 3$ filters (stride=1), 
      followed by BN and ReLU} \\
    7 & \caja{t}{l}{convolutional, 256 $3 \times 3$ filters (stride=1), 
      followed by BN and ReLU} \\
      & \caja{t}{l}{max pool, $2 \times 2$ window (stride=2)} \\
    8 & \caja{t}{l}{convolutional, 512 $3 \times 3$ filters (stride=1), 
      followed by BN and ReLU} \\
    9 & \caja{t}{l}{convolutional, 512 $3 \times 3$ filters (stride=1), 
      followed by BN and ReLU} \\
    10 & \caja{t}{l}{convolutional, 512 $3 \times 3$ filters (stride=1), 
      followed by BN and ReLU} \\
      & \caja{t}{l}{max pool, $2 \times 2$ window (stride=2)} \\
    11 & \caja{t}{l}{convolutional, 512 $3 \times 3$ filters (stride=1), 
      followed by BN and ReLU} \\
    12 & \caja{t}{l}{convolutional, 512 $3 \times 3$ filters (stride=1), 
      followed by BN and ReLU} \\
    13 & \caja{t}{l}{convolutional, 512 $3 \times 3$ filters (stride=1), 
      followed by BN and ReLU} \\
      & \caja{t}{l}{max pool, $2 \times 2$ window (stride=2)} \\
    14 & \caja{t}{l}{fully connected, 512 neurons and dropout \\ with $p=0.5$, followed by ReLU} \\
    15 & \caja{t}{l}{fully connected, 512 neurons and dropout \\ with $p=0.5$, followed by ReLU} \\
    \caja{t}{l}{16 \\ (output)} & \caja{t}{l}{fully connected, 10 neurons, followed by softmax} \\
    \midrule
    \multicolumn{2}{c}{14981952 weights, 8970 biases, 8448 running means/variances for BN} \\
    \bottomrule
  \end{tabular}
  \caption{Structure of the adapted VGG16 network for CIFAR10 dataset. BN--Batch Normalization, ReLU -- rectified linear units.}
  \label{t:VGG16}
\end{table}

\begin{table}[t]
\centering
\begin{tabular}{@{}ll@{ + }lrrrrr@{}}
\toprule
 &\multicolumn{2}{c}{Model} & $\log L$& $E_\text{test}$, \% & $\rho_\text{s}$ & $\rho_+$ & $\rho_\times$ \\
   \midrule 
   &\multicolumn{2}{l}{\textbf{R}}      &\textbf{-0.80}& \textbf{8.35}        & 1.00           & 1.00 & 1.00 \\
 &3\% correction & $\rankop=3$   & -0.51 & 12.02                 & \textbf{15.84} & \textbf{5.72}          & \textbf{5.72}\\ 
 \raisebox{0pt}[0pt][0pt]{\rotatebox{90}{\makebox[0pt][c]{\hspace*{2.5em} RN-20 }}}
 &3\% correction  & $\rankop=4$   & -0.55 & 11.44                 & 13.34         & 4.70           & 4.70\\
 
 \midrule
 &\multicolumn{2}{l}{\textbf{R}}        &\textbf{-0.82}& \textbf{7.14}        & 1.00           & {1.00} & 1.00 \\
 &3\% correction & $\rankop=3$   & -0.70 & 9.55                 & \textbf{15.98} & \textbf{5.92}          & \textbf{5.92}\\ 
 \raisebox{0pt}[0pt][0pt]{\rotatebox{90}{\makebox[0pt][c]{\hspace*{2.5em} RN-32 }}}
 &3\% correction & $\rankop=4$   & -0.79 & 9.14                 & 13.47         & 4.85           & 4.85\\
 
 \midrule
 &\multicolumn{2}{l}{\textbf{R}}        &-0.81& \textbf{6.58}                 & 1.00           & 1.00 & 1.00 \\
 &3\% correction & $\rankop=3$   & -0.90 & 8.38                 & \textbf{16.06} & \textbf{5.99}          & \textbf{5.99}\\ 
 \raisebox{0pt}[0pt][0pt]{\rotatebox{90}{\makebox[0pt][c]{  \hspace*{2em} RN-56}}}
 &3\% correction & $\rankop=4$   & \textbf{-0.98} & 8.02                 & 13.54         & 4.90           & 4.90\\

 \midrule
 &\multicolumn{2}{l}{\textbf{R}}        &-0.77& \textbf{6.02}                 & 1.00           & 1.00 & 1.00 \\
 &3\% correction & $\rankop=3$   & -1.11 & 6.63                 & \textbf{16.03} & \textbf{6.38}          & \textbf{6.38}\\ 
 \raisebox{0pt}[0pt][0pt]{\rotatebox{90}{\makebox[0pt][c]{  \hspace*{2.5em} RN-110 }}}
 &3\% correction & $\rankop=4$   & \textbf{-1.14} & 6.36                 & 13.53         & 5.17           & 5.17\\

 \midrule
 &\multicolumn{2}{l}{\textbf{R}}        &-0.89& 6.45                 & 1.00           & 1.00 & 1.00 \\
 &3\% correction & $\rankop=2$   & -1.04 & 6.66                 & \textbf{60.99} & \textbf{8.20}          & \textbf{8.20}\\ 
 \raisebox{0pt}[0pt][0pt]{\rotatebox{90}{\makebox[0pt][c]{  \hspace*{2.5em} VGG16 }}}
 &3\% correction & $\rankop=3$   & \textbf{-1.05} & \textbf{6.65}                 & 56.58         & 7.81           & 7.81\\
\bottomrule
\end{tabular}
\caption{Results of running low-rank with pointwise corrections (L+P) on ResNets of different depth and VGG16.  We report test error, the reduction ratio of storage ($\rho_\text{s}$), floating point additions ($\rho_+$) and multiplications ($\rho_\times$), under the assumption of efficient implementation (section~\ref{section:efficient_impl}). RS stands for ResNet, and \textbf{R} for reference models; logarithms are base 10. In general, this type of compression performs much better on VGG16 than on ResNets.}
\label{table:p_lr}
\end{table}

\subsection{Quantization plus low-rank plus pruning, Q+L+P }
\label{section:q+l+p}
In order to verify the complementarity benefits of additive compressions, we run experiments on ResNet-s to see whether adding another compression is going to help. As we saw in Table\ref{table:q_lr}, the  quantization with additive low-rank (Q+L) achieves good performance on the ResNet110, having compressed models outperforming the reference. However, for smaller depth ResNets the Q+L scheme does not perform as good. We choose to compress these ResNets of depth 20, 32, 56 with Q+L scheme with a small number additional point-wise corrections (pruning), turning it to Q+L+P scheme. Table~\ref{table:q+l+p} shows the results.

\begin{table}[t]
\centering
\begin{tabular}{@{}ll@{ + }lrcrrr@{}}
\toprule
 &\multicolumn{2}{c}{Model} & $\log L$ &  $E_\text{test}$, \% & $\rho_\text{s}$ & improvement $\Delta$ \\
 \midrule
   &\multicolumn{2}{l}{\textbf{R}}                  & -0.80 & \textbf{8.35}                & 1.00           &  \\
 &1-bit quant. & $\rankop=1$ + 0.37\% correction    & -0.87 & 9.55                 & \textbf{19.55} &  +0.16        \\ 
 \raisebox{0pt}[0pt][0pt]{\rotatebox{90}{\makebox[0pt][c]{\hspace{2.5em}  RN20 }}}
 &1-bit quant. & $\rankop=2$ + 0.37\% correction    & -0.93 & 9.14                 & 15.85          & +0.16          \\
   \midrule
 &\multicolumn{2}{l}{\textbf{R}}                    & -0.81 & \textbf{7.14}        & 1.00           &  \\
 &1-bit quant. & $\rankop=1$ + 0.32\% correction    & -1.08 & 7.47                   & \textbf{19.82} & +0.43        \\ 
 \raisebox{0pt}[0pt][0pt]{\rotatebox{90}{\makebox[0pt][c]{\hspace{2.5em}  RN32 }}}
 &1-bit quant. & $\rankop=2$ + 0.32\% correction    & -1.12 & 7.27                   & 16.81          & +0.79         \\
    \midrule
 &\multicolumn{2}{l}{\textbf{R}}                    & -0.77 & 6.58                   & 1.00           &              \\
 &1-bit quant. & $\rankop=1$ + 0.35\% correction    & -1.20 & 6.64                   & \textbf{19.82} & +0.55          \\ 
 \raisebox{0pt}[0pt][0pt]{\rotatebox{90}{\makebox[0pt][c]{\hspace{2.5em}  RN56 }}}
 &1-bit quant. & $\rankop=2$ + 0.35\% correction    & -1.23 &\textbf{6.47}                 & 16.10         & +0.04    \\
\bottomrule
\end{tabular}
\caption{Results of running our algorithm with three compressions combined: quantization $+$ low-rank corrections $+$ point-wise corrections (Q+L+P). We choose to add a few pointwise compression to some of the Q+L experiments in Table~\ref{table:q_lr}. We report loss $L$, test error $E_\text{test}$, storage compression ratio ($\rho_\text{s}$), and improvement $\Delta$ of test error over Q+L scheme (higher is better). \textbf{R} stands for the reference network, and RS is shorthand for ResNet. As you can see, adding few point-wise corrections improve train loss and test error in all compression schemes.}
\label{table:q+l+p}
\end{table}

\begin{figure}
\centering
\psfrag{corrections}[B][]{correction, \%}
\psfrag{layers}[t][t]{layers}
\begin{tabular}{@{}c@{\hspace{2em}}c@{}}
ResNet20 & ResNet32 \\
\psfrag{low}[l][l]{Q + r=1 + 0.37\% corr.}
\psfrag{high}[l][l]{Q + r=2 + 0.37\% corr.}
\includegraphics[clip,height=0.42\linewidth]{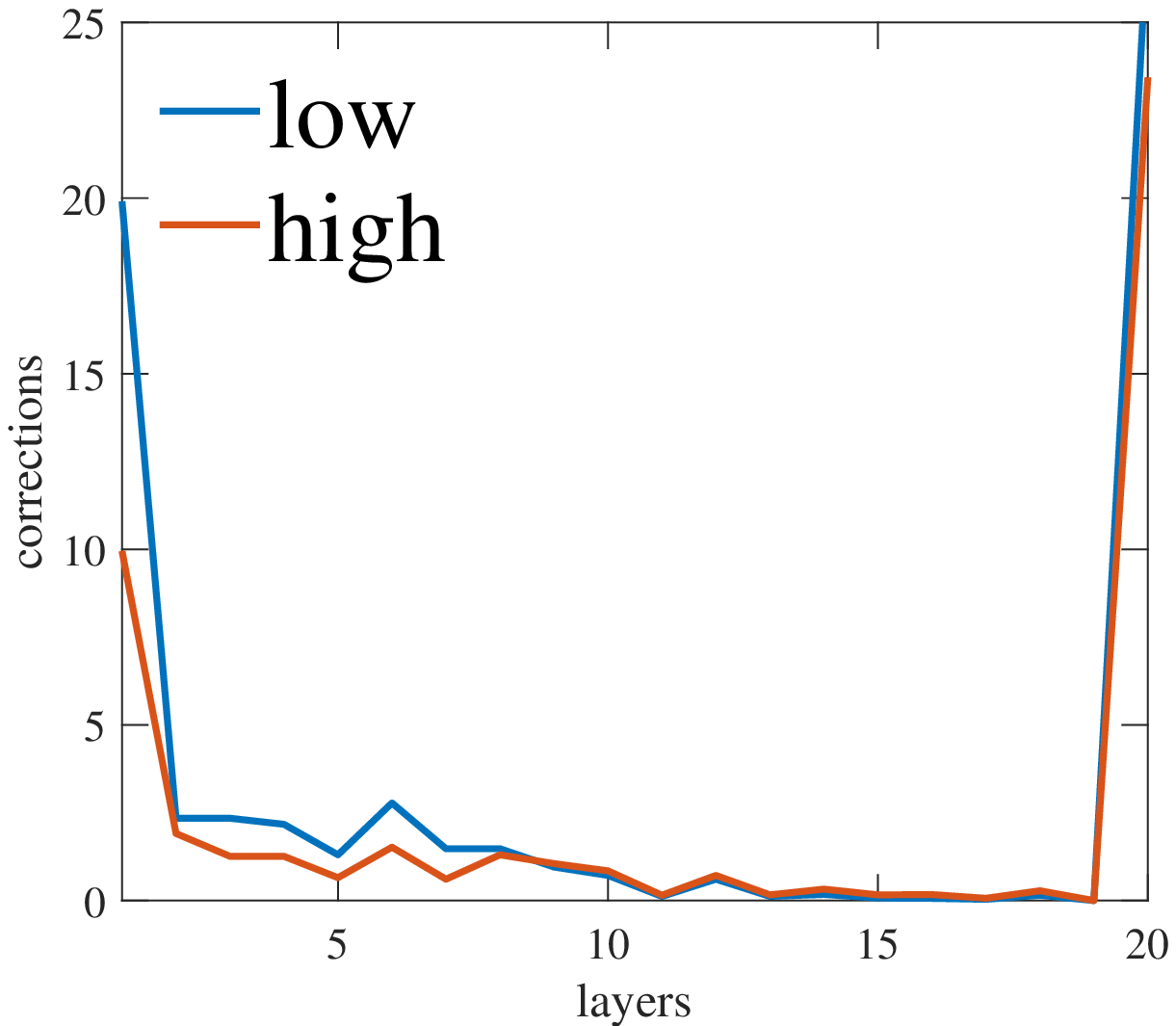} &
\psfrag{low}[l][l]{Q + r=1 + 0.32\% corr.}
\psfrag{high}[l][l]{Q + r=2 + 0.32\% corr.}
\psfrag{corrections}{}
\includegraphics[clip,height=0.42\linewidth]{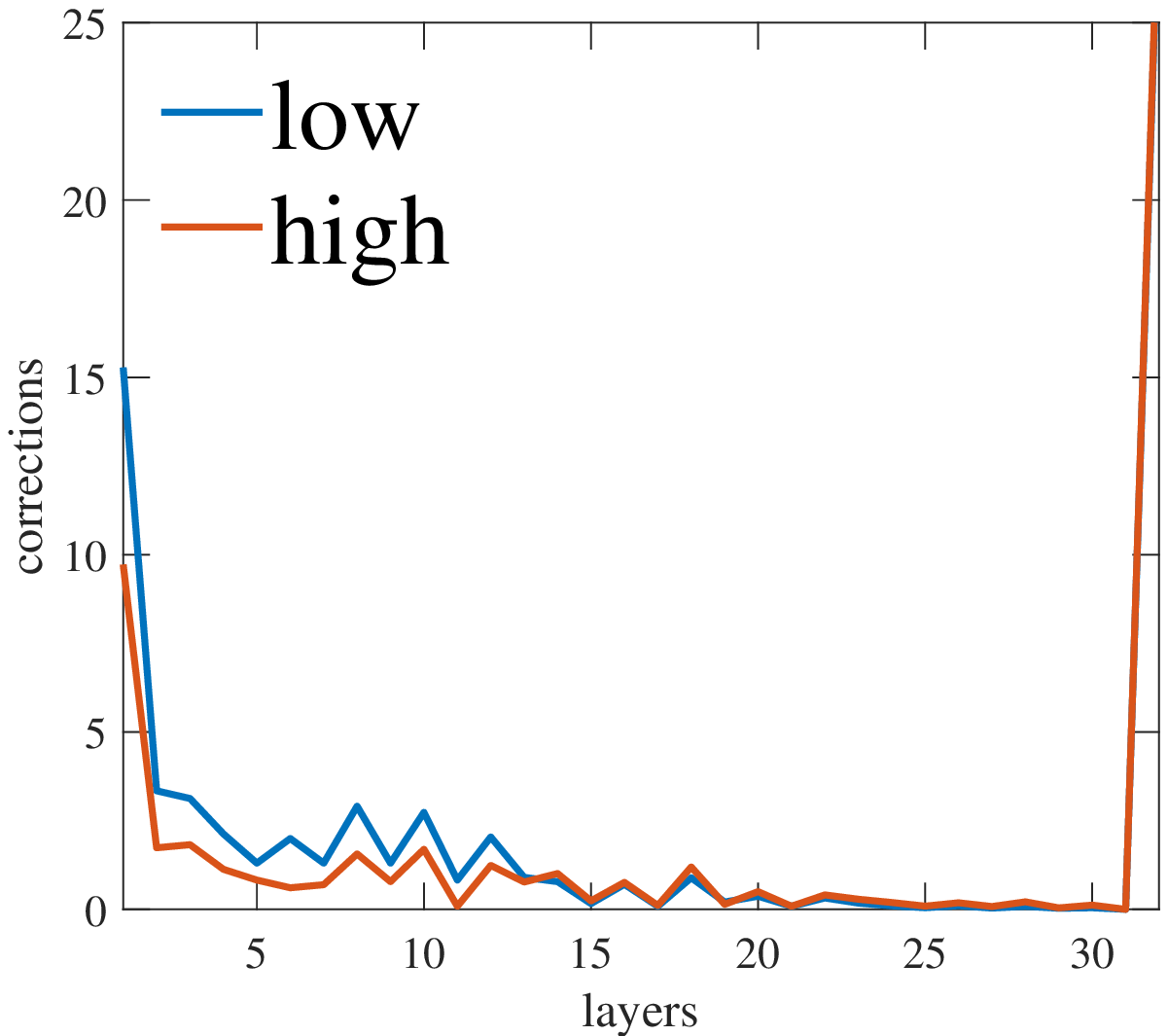} \\[2ex]
ResNet56 \\
\psfrag{low}[l][l]{Q + r=1 + 0.35\% corr.}
\psfrag{high}[l][l]{Q + r=2 + 0.35\% corr.}
\includegraphics[clip,height=0.42\linewidth]{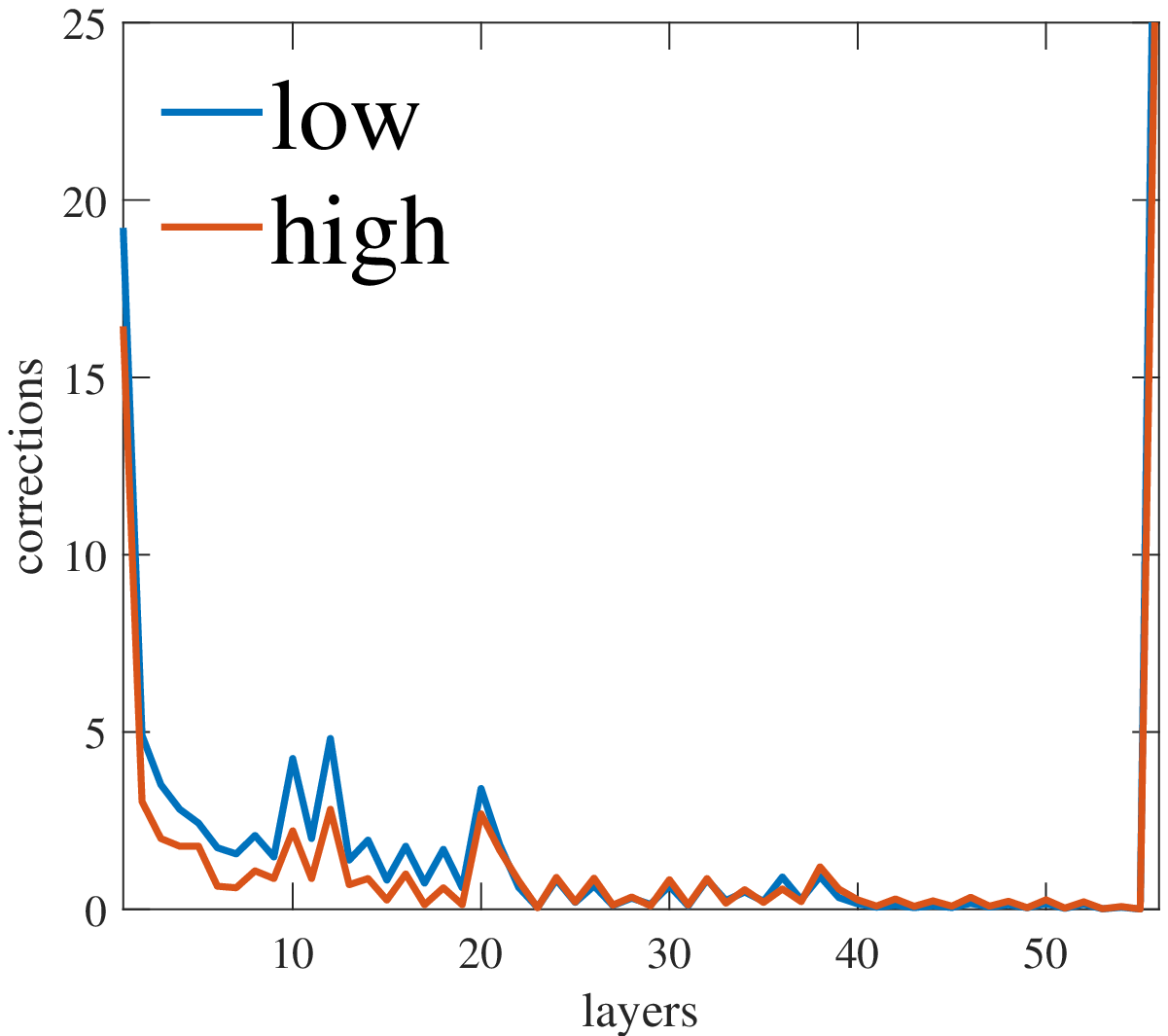} 
\end{tabular}
\caption{We plot here for every ResNet compressed with Q+L+P scheme, the proportion of weight being corrected at every layer when a different amount of rank corrections is applied. As rank increases, the corrections move to other places where it is needed most.}
\label{fig:qp_sprs2}
\end{figure}
\clearpage

\section{Experiments on ImageNet}
\label{section:exp_imagenet}
Due to the large number of possible combinations and scale of each experiment, we limit our attention to quantization with additive point-wise corrections combination (Q+P). Below we give details on training reference and compressed networks. 
\begin{description}
\item[Training the reference net] We train Batch Normalized AlexNet (having 1140 FLOPs) network using data-augmentation of the original paper \citep{Krizhev_12a} on ImageNet ILSVRC2012 dataset \cite{Russak_15a}. During training the images are resized to have the shortest side of 256 pixels length, and random $227\times 227$ part of the image is cropped. Then pixel-wise color mean subtracted and image normalized to have standard variance. During test, we use a central $227\times 227$ crop. We train with SGD on minibatches of size 256, with weight decay of $0.5\times 10^{-4}$ and an initial learning rate of 0.05, which is reduced $10\times$ after every 20 epochs. The trained model achieves Top-1 validation accuracy of 59.57\% and Top-5 validation accuracy of 82.45\%. This is slightly better than the the original paper \citep{Krizhev_12a} and Caffe re-implementation \cite{Jia_14a}.  \emph{Training time} on the NVIDIA Titan V GPU is 17 hours.
\item[Training compressed models] We train Q+P scheme with 1-bit quantization throughout (every layer having its own codebook) and total number of pointwise correction $\kappa$ to be 0.5M and 1M. These models are trained for 35 LC steps, with $\mu=5\times 10^{-4}\times 1.12^k$ at $k$-th iteration. Each L-step is performed by Nesterov's SGD with momentum 0.9 and runs for 10 epochs (with 256 images in a minibatch) with a learning rate of $0.001$ at the beginning of the step and decayed by 0.94 after each epoch. Finally, we perform fine-tuning for 10 epochs with a learning rate of $0.001$ decayed by 0.9 after every epoch, on minibatches of size 512.  Running the LC algorithm with finetuning is 3.6 times longer comparing to the training of the reference network. \\
The resulting compressed models are quantized and have sparse corrections. To efficiently save them to disk, we utilized the sparse index compression technique described in \citet{Han_16a}, and used the standard compressed save option of numpy library. 
\end{description}

\begin{table}[t]
\centering
\begin{tabular}{@{}ll@{ + }ccrrr@{}}
\toprule
 &\multicolumn{2}{c}{Model} & $E$ top-1, \% & $E$ top-5, \% & $\rho_\text{s}$  \\
   \midrule
 &\multicolumn{2}{l}{\textbf{R} BN-AlexNet}  & 40.43 & 17.55                & 1.00            \\
 \raisebox{0pt}[0pt][0pt]{\rotatebox{90}{\makebox[0pt][c]{  ours }}}
 &1-bit quant. & 0.5M corrections (0.8\%)  & 39.09  & 16.84 & \textbf{25.96}       \\ 
 &1-bit quant. & 1.0M corrections (1.6\%)  & \textbf{38.94}  & \textbf{16.69} & 22.42       \\
\midrule
& \multicolumn{2}{l}{1-bit (DoReFa) \cite{Zhou_16b}} & 46.10 &  23.70 & 10.35\\
& \multicolumn{2}{l}{1-bit (BWN) \cite{Rasteg_16a}} & 43.20 &  20.59 & 10.35\\
& \multicolumn{2}{l}{1-bit (ADMM) \cite{Leng_18a}} & 43.00 &  20.29 & 10.35\\
& \multicolumn{2}{l}{1-bit (Quantization Net) \cite{Yang_19b}} & 41.20 &  18.30 & 10.35\\
\midrule 
& \multicolumn{2}{l}{multi-bit AlexNet-QNN \cite{Wu_16a}} & 44.24 & 20.92 & 18.76\\
\bottomrule
\end{tabular}
\caption{Results of running 1-bit quantization with 0.5M and 1M corrections on AlexNet and comparison to the quantization methods in the literature. We report top-1 and top-5 errors on the validation set, the compression ratio of parameter storage ($\rho_\text{s}$). Notice that although results in the literature use 1-bit quantization, the compression ratio is rather small due to not compressing first and last layers, keeping them full precision. Also, notice that our compressed model achieves better than the reference validation error.}
\label{table:qp_alexnet}
\end{table}

Table~\ref{table:qp_alexnet} summarizes our results and compares to existing quantization methods. We would like to note that we apply quantization throughout all layers, and allow compression mechanism (sparse additions) to self-correct at necessary positions. This is in contrast to quantization results on AlexNet where first and last layers are excluded from compression to keep overall performance from degradation. Thus, we achieve 25.96$\times$ compression in parameter storage size without any accuracy loss, while best 1-bit quantization methods achieve only 10.35$\times$ compression.

Although the Q+P scheme is quite powerful in representing the weights, its compression ratio with scalar quantization is limited by at most $32\times$, as the minimal amount of bits required to represent one weight is at least 1bit. If we would like to drive it further, compressions need to be nested, and in the following set of experiments we show that we can capitalize on the power of the Q+P scheme and drive the compression further.
To demonstrate it, we decided to further compress using the Q+P scheme the low-rank versions of AlexNet obtained from \cite{IdelbayCarreir20a}, see Table~\ref{table:qp_alexnet_low_rank}. We choose multiple low-rank networks achieving $3.20\times$ to $4.78\times$ reduction in FLOPs and applied our algorithm with varying amount of corrections (P) and fixed 1-bit quantization (Q). 

\begin{table}[t]
\center
\begin{tabular}{@{}l@{\hspace{0.3em}}lccccc@{}}
\toprule
 &\multicolumn{1}{c}{Model} & top-1 & top-5 &  size & compression & MFLOPs  \\
  & & \% & \% & MB & ratio, $\rho_\text{size}$ & \\
   \midrule
 &Caffe-AlexNet\cite{Han_16a}                        & 42.70 & 19.80 & 243.5 & 1  & 724           \\
 \midrule
 & Caffe-AlexNet-QNN \cite{Wu_16a}                   & 44.24 & 20.92 &  13  & 18.7  & 175$^*$ \\
 &P$\rightarrow$Q \cite{Han_16a}                     & 42.78 & 19.70 &  6.9 & 35.2  & 724 \\
 & P$\rightarrow$Q \cite{Choi_17a}                 & 43.80 & \na   &  5.9 & 41.2  & 724 \\
 & P$\rightarrow$Q \cite{TungMori18a}                & 42.10 & \na   &  4.8 & 50.7  & 724 \\
 & P$\rightarrow$Q \cite{Yang_20b}                   & 42.48 & \na   &  4.7 & 51.8  & 724 \\
 & P$\rightarrow$Q \cite{Yang_20b}                   & 43.40 & \na   &  3.1 & 78.5  & 724 \\
 & filter pruning  \cite{Wen_16a}                    &  \na  & 21.63 &  \na & \na   & 231 \\
 & filter pruning  \cite{Yu_18a}                     & 44.13 & \na   &  \na & \na   & 232 \\ 
 & filter pruning \cite{Li_19a}                      & 43.17 & \na   &  232 & 1.04  & 334 \\
 & filter pruning \cite{Ding_19a}                    & 43.83 & 20.47 &  \na & \na   & 492 \\
 & weight pruning \cite{Xiao_19a}                    & 44.10 & \na   &   13 &  18.7 & \na \\
 \midrule
 &\textbf{R} Low-rank AlexNet ($\text{L}_1$)         & 39.61 & 17.40 & 100.9&  2.4 & 227 \\
 &$\text{L}_1$ $\rightarrow$ Q (1-bit) + P (0.25M)   & 39.67 & 17.36 & 3.7  &  65.8  & 227
 \\
 \raisebox{0pt}[0pt][0pt]{\rotatebox{90}{\makebox[0pt][c]{\hspace*{1em}our }}}
 &$\text{L}_1$ $\rightarrow$ Q (1-bit) + P (0.50M)   & 39.25 & 16.97 & 4.3  &  56.9  & 227 \\
 \midrule 
 &\textbf{R} Low-rank AlexNet ($\text{L}_2$)         & 39.61 & 17.40 & 72.4 &  3.6 & 185 \\
 &$\text{L}_2$ $\rightarrow$ Q (1-bit) + P (0.25M)   & 40.19 & 17.50  & 2.8  &  87.6 & 185 \\
 \raisebox{0pt}[0pt][0pt]{\rotatebox{90}{\makebox[0pt][c]{\hspace*{2em}our }}}
 &$\text{L}_2$ $\rightarrow$ Q (1-bit) + P (0.50M)   & 39.97 & 17.35 & 3.4  &  72.15 & 185 \\ 
 \midrule
 &\textbf{R} Low-rank AlexNet ($\text{L}_3$)         & 41.02 & 18.22 & 49.9 &  4.8  & \textbf{152}  \\
 &$\text{L}_3$ $\rightarrow$ Q (1-bit) + P (0.25M)   & 41.27 & 18.44 & \textbf{2.1
} & \textbf{117.3}  & \textbf{151} \\
 \raisebox{0pt}[0pt][0pt]{\rotatebox{90}{\makebox[0pt][c]{\hspace*{2em}our }}}
 &$\text{L}_3$ $\rightarrow$ Q (1-bit) + P (0.50M)   & 40.88  & 18.29 & 2.7 &  90.4 & \textbf{151} \\ 
\bottomrule
\end{tabular}
\vspace{1em}
\caption{Q+P scheme is powerful enough to further compress already downsized models, here (bottom of the table) obtained by low-rank compression (L). We report top-1 validation error, size of the final model in MB, and resulting FLOPs. Shorthands are as follows: P stands for pruning, Q for quantization, L for low-rank, and H.C.\ for Huffman Coding. Numbers with $^*$ assumes efficient software/hardware implementation. }
\label{table:qp_alexnet_low_rank}
\vspace{-1em}
\end{table}

\subsection{Runtime measurements on Jetson Nano}
\label{sec:runtime}
\begin{figure}
  \centering
  \begin{minipage}[t]{0.59\linewidth}
    \centering
    characteristics\\[1em]
    \begin{tabular}{@{}lp{0.75\linewidth}@{}}
    \toprule
    CPU & Quad-core ARM Cortex-A57 at 1479 MHz \\
    GPU & 128 CUDA cores at 922 MHz \\
    RAM & 4 GB 64-bit LPDDR4 at 1600MHz \\
    OS  & Ubuntu 18.04.5 LTS \\
    Kernel & GNU/Linux 4.9.140-tegra  \\
    Storage & 128 GB microSDXC memory card (class UHS-I) \\
    Software & PyTorch v1.6.0, TensorFlow v2.2.0, \newline ONNXRuntime v1.4.0, TensorRT v7.1.3.0 \\ 
    \bottomrule
    \end{tabular} \\[0.1em]
  \end{minipage}
  \begin{minipage}[t]{0.38\linewidth}
    \centering
    photo of the actual board \\[1em]
    \includegraphics[width=0.87\linewidth]{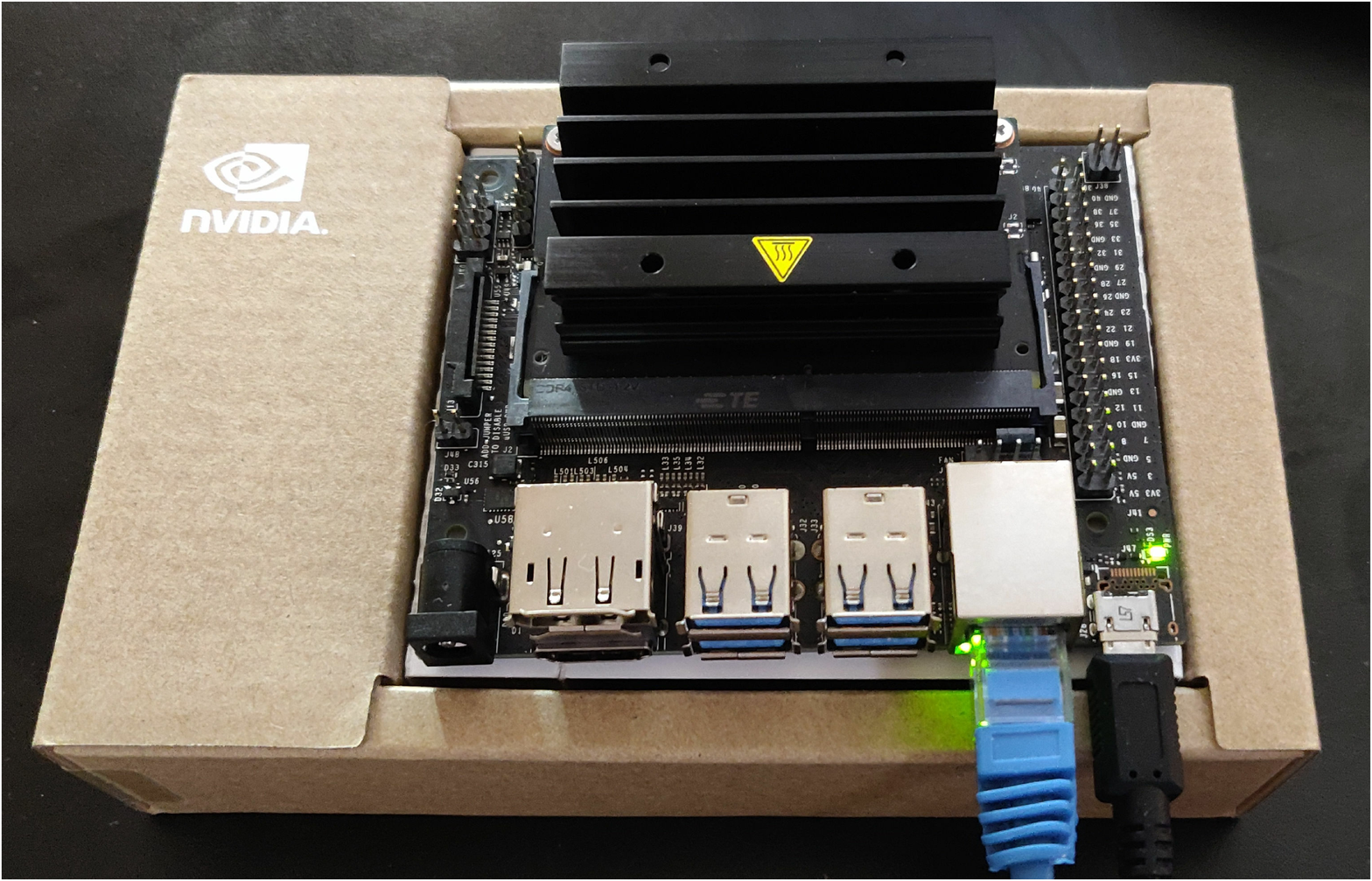}
  \end{minipage}
  \caption{The characteristics of the Jetson Nano developer kit we have used and a photo of the actual experimental setup.}
  \label{fig:jetsonnano_setup}
\end{figure}

The NVIDIA's Jetson Nano\footnote{\url{https://developer.nvidia.com/embedded/jetson-nano}} is a small yet powerful edge inference device with quad core CPU and 128 core GPU requiring only 10 watts. We use its publicly available developer kit\footnote{\url{https://developer.nvidia.com/embedded/jetson-nano-developer-kit}} to run our experiments. The full characteristics of our setup is given on Figure~\ref{fig:jetsonnano_setup}. Since device runs the Ubuntu operating system, most of the software is readily available. However, the deep-learning libraries require specific configuration and re-compilation to run on the available GPU. The pre-configured libraries are made available by NVIDIA's software teams.

\subsubsection{Runtime evaluation}

We evaluated the models on both CPU and GPU. We run the inference task (the forward pass) of tested neural networks using the batch size of 1, i.e., a single image, using 32 bit floating point values. For the CPU inference we converted all models into the ONNX inter-operable neural network format and used highly optimized ONNXRuntime. For the GPU inference we used the TensorRT library to convert the ONNX models into TensorRuntime models. To account for various operating systems scheduling issues and context switches, we repeated the measurements $N$ times and dropped $p$-proportion of the largest measurements. For the CPU measurements we used $N=100$ and for the GPU we used $N=1000$. We set $p=0.1$, i.e., the highest 10\% of the measurements will not be used in calculating the average delay. To avoid the influence of the thermal throttling, we allowed a cooldown period of at least 30 seconds between the measurements of different neural networks.

\begin{table}
  \centering
  \begin{tabular}{@{}lcr@{\hspace*{0.6em}}cr@{\hspace*{0.6em}}c@{}}
    \toprule
    \multirow{ 2}{*}{\hspace{4em}Model}  & \multirow{ 2}{*}{Top-1 err, \%} & \multicolumn{2}{c}{Edge CPU} & \multicolumn{2}{c}{Edge GPU}\\
    & & time, ms & speed-up & time, ms & speed-up  \\
    \midrule
    Caffe-AlexNet & 42.70 & 328.69 & 1.00 & 23.27 & 1.00 \\
    $\text{L}_1$ $\rightarrow$ Q (1-bit ) + P (0.25M) & 39.67 & 83.62  & 3.93 & 11.32 & 2.06 \\
    $\text{L}_2$ $\rightarrow$ Q (1-bit ) + P (0.25M) & 40.19 & 60.49  & 5.43 &  8.74 & 2.66  \\
    $\text{L}_3$ $\rightarrow$ Q (1-bit ) + P (0.25M) & 41.27 & 44.80  & 7.33 &  6.72 & 3.46  \\
    \bottomrule
  \end{tabular} 
  \caption{Runtime measurements of forward pass on Jetson Nano and corresponding speed-up when compared to Caffe-Alexnet.}
\end{table}


\begin{thebibliography}{56}
\providecommand{\natexlab}[1]{#1}
\providecommand{\url}[1]{\texttt{#1}}
\expandafter\ifx\csname urlstyle\endcsname\relax
  \providecommand{\doi}[1]{doi: #1}\else
  \providecommand{\doi}{doi: \begingroup \urlstyle{rm}\Url}\fi

\bibitem[Agustsson et~al.(2017)Agustsson, Mentzer, Tschannen, Cavigelli,
  Timofte, Benini, and Van~Gool]{Agusts_17a}
E.~Agustsson, F.~Mentzer, M.~Tschannen, L.~Cavigelli, R.~Timofte, L.~Benini,
  and L.~Van~Gool.
\newblock Soft-to-hard vector quantization for end-to-end learning compressible
  representations.
\newblock In I.~Guyon, U.~v.~Luxburg, S.~Bengio, H.~Wallach, R.~Fergus,
  S.~Vishwanathan, and R.~Garnett, editors, \emph{Advances in Neural
  Information Processing Systems (NIPS)}, volume~30, pages 1141--1151. MIT
  Press, Cambridge, MA, 2017.

\bibitem[Alvarez and Salzmann(2017)]{AlvarezSalzman17a}
J.~M. Alvarez and M.~Salzmann.
\newblock Compression-aware training of deep networks.
\newblock In I.~Guyon, U.~v.~Luxburg, S.~Bengio, H.~Wallach, R.~Fergus,
  S.~Vishwanathan, and R.~Garnett, editors, \emph{Advances in Neural
  Information Processing Systems (NIPS)}, volume~30, pages 856--867. MIT Press,
  Cambridge, MA, 2017.

\bibitem[Babenko and Lempitsky(2014)]{BabenkLempit14a}
A.~Babenko and V.~Lempitsky.
\newblock Additive quantization for extreme vector compression.
\newblock In \emph{Proc. of the 2014 IEEE Computer Society Conf. Computer
  Vision and Pattern Recognition (CVPR'14)}, pages 931--938, Columbus, OH,
  June~23--28 2014.

\bibitem[Beck and Tetruashvili(2013)]{BeckTetruas13a}
A.~Beck and L.~Tetruashvili.
\newblock On the convergence of block coordinate descent type methods.
\newblock \emph{SIAM J. Optimization}, 23\penalty0 (4):\penalty0 2037--2060,
  2013.

\bibitem[Bertsekas(1999)]{Bertsek99a}
D.~P. Bertsekas.
\newblock \emph{Nonlinear Programming}.
\newblock Athena Scientific, Nashua, NH, second edition, 1999.

\bibitem[Bouwmans and hadi Zahzah(2016)]{BouwmanZahzah16a}
T.~Bouwmans and E.~hadi Zahzah.
\newblock Robust principal component analysis via decomposition into low-rank
  and sparse matrices: An overview.
\newblock In T.~Bouwmans, N.~S. Aybat, and E.~hadi Zahzah, editors,
  \emph{Handbook of Robust Low-Rank and Sparse Matrix Decomposition.
  {Applications} in Image and Video Processing}, chapter~1, pages 1.1--1.61.
  CRC Publishers, 2016.

\bibitem[Cand{\`e}s et~al.(2011)Cand{\`e}s, Li, Ma, and Wright]{Candes_11a}
E.~J. Cand{\`e}s, X.~Li, Y.~Ma, and J.~Wright.
\newblock Robust principal component analysis?
\newblock \emph{Journal of the ACM}, 58\penalty0 (3):\penalty0 11, May 2011.

\bibitem[Carreira-Perpi{\~n}{\'a}n(2017)]{Carreir17a}
M.~{\'A}. Carreira-Perpi{\~n}{\'a}n.
\newblock Model compression as constrained optimization, with application to
  neural nets. {Part} {I}: General framework.
\newblock arXiv:1707.01209, July~5 2017.

\bibitem[Carreira-Perpi{\~n}{\'a}n and Idelbayev(2018)]{CarreirIdelbay18a}
M.~{\'A}. Carreira-Perpi{\~n}{\'a}n and Y.~Idelbayev.
\newblock ``{Learning}-compression'' algorithms for neural net pruning.
\newblock In \emph{Proc. of the 2018 IEEE Computer Society Conf. Computer
  Vision and Pattern Recognition (CVPR'18)}, pages 8532--8541, Salt Lake City,
  UT, June~18--22 2018.

\bibitem[Chandrasekaran et~al.(2010)Chandrasekaran, Sanghavi, Parrilo, and
  Willsky]{Chandr_10a}
V.~Chandrasekaran, S.~Sanghavi, P.~A. Parrilo, and A.~S. Willsky.
\newblock Rank-sparsity incoherence for matrix decomposition.
\newblock \emph{SIAM J. Optimization}, 21\penalty0 (2):\penalty0 572--596,
  2010.

\bibitem[Choi et~al.(2017)Choi, El-Khamy, and Lee]{Choi_17a}
Y.~Choi, M.~El-Khamy, and J.~Lee.
\newblock Towards the limit of network quantization.
\newblock In \emph{Proc. of the 5th Int. Conf. Learning Representations (ICLR
  2017)}, Toulon, France, Apr.~24--26 2017.

\bibitem[Ding et~al.(2019)Ding, Ding, Guo, Han, and Yan]{Ding_19a}
X.~Ding, G.~Ding, Y.~Guo, J.~Han, and C.~Yan.
\newblock Approximated oracle filter pruning for destructive {CNN} width
  optimization.
\newblock In K.~Chaudhuri and R.~Salakhutdinov, editors, \emph{Proc. of the
  36th Int. Conf. Machine Learning (ICML 2019)}, pages 1607--1616, Long Beach,
  CA, June~9--15 2019.

\bibitem[Han et~al.(2016)Han, Mao, and Dally]{Han_16a}
S.~Han, H.~Mao, and W.~J. Dally.
\newblock Deep compression: Compressing deep neural networks with pruning,
  trained quantization and {Huffman} coding.
\newblock In \emph{Proc. of the 4th Int. Conf. Learning Representations (ICLR
  2016)}, San Juan, Puerto Rico, May~2--4 2016.

\bibitem[Hastie and Tibshirani(1990)]{HastieTibshir90a}
T.~J. Hastie and R.~J. Tibshirani.
\newblock \emph{Generalized Additive Models}.
\newblock Number~43 in Monographs on Statistics and Applied Probability.
  Chapman \& Hall, London, New York, 1990.

\bibitem[He et~al.(2015)He, Zhang, Ren, and Sun]{He_15a}
K.~He, X.~Zhang, S.~Ren, and J.~Sun.
\newblock Delving deep into rectifiers: Surpassing human-level performance on
  {ImageNet} classification.
\newblock In \emph{Proc. 15th Int. Conf. Computer Vision (ICCV'15)}, pages
  1026--1034, Santiago, Chile, Dec.~11--18 2015.

\bibitem[He et~al.(2016)He, Zhang, Ren, and Sun]{He_16a}
K.~He, X.~Zhang, S.~Ren, and J.~Sun.
\newblock Deep residual learning for image recognition.
\newblock In \emph{Proc. of the 2016 IEEE Computer Society Conf. Computer
  Vision and Pattern Recognition (CVPR'16)}, pages 770--778, Las Vegas, NV,
  June~26 -- July~1 2016.

\bibitem[Idelbayev and
  Carreira-Perpi{\~n}{\'a}n(2020{\natexlab{a}})]{IdelbayCarreir20a}
Y.~Idelbayev and M.~{\'A}. Carreira-Perpi{\~n}{\'a}n.
\newblock Low-rank compression of neural nets: Learning the rank of each layer.
\newblock In \emph{Proc. of the 2020 IEEE Computer Society Conf. Computer
  Vision and Pattern Recognition (CVPR'20)}, pages 8046--8056, Seattle, WA,
  June~14--19 2020{\natexlab{a}}.

\bibitem[Idelbayev and
  Carreira-Perpi{\~n}{\'a}n(2020{\natexlab{b}})]{IdelbayCarreir20b}
Y.~Idelbayev and M.~{\'A}. Carreira-Perpi{\~n}{\'a}n.
\newblock A flexible, extensible software framework for model compression based
  on the {LC} algorithm.
\newblock arXiv:2005.07786, May~15 2020{\natexlab{b}}.

\bibitem[Idelbayev and
  Carreira-Perpi{\~n}{\'a}n(2021{\natexlab{a}})]{IdelbayCarreir21b}
Y.~Idelbayev and M.~{\'A}. Carreira-Perpi{\~n}{\'a}n.
\newblock Optimal selection of matrix shape and decomposition scheme for neural
  network compression.
\newblock In \emph{Proc. of the IEEE Int. Conf. Acoustics, Speech and Sig.
  Proc. (ICASSP'21)}, pages 3250--3254, Toronto, Canada, June~6--11
  2021{\natexlab{a}}.

\bibitem[Idelbayev and
  Carreira-Perpi{\~n}{\'a}n(2021{\natexlab{b}})]{IdelbayCarreir21e}
Y.~Idelbayev and M.~{\'A}. Carreira-Perpi{\~n}{\'a}n.
\newblock More general and effective model compression via an additive
  combination of compressions.
\newblock In \emph{Proc. of the 32nd European Conf. Machine Learning
  (ECML--21)}, Bilbao, Spain, Sept.~13--17 2021{\natexlab{b}}.

\bibitem[Jia et~al.(2014)Jia, Shelhamer, Donahue, Karayev, Long, Girshick,
  Guadarrama, and Darrell]{Jia_14a}
Y.~Jia, E.~Shelhamer, J.~Donahue, S.~Karayev, J.~Long, R.~Girshick,
  S.~Guadarrama, and T.~Darrell.
\newblock Caffe: Convolutional architecture for fast feature embedding.
\newblock arXiv:1408.5093 [cs.CV], June~20 2014.

\bibitem[Kim et~al.(2019)Kim, Khan, and Kyung]{Kim_19a}
H.~Kim, M.~U.~K. Khan, and C.-M. Kyung.
\newblock Efficient neural network compression.
\newblock In \emph{Proc. of the 2019 IEEE Computer Society Conf. Computer
  Vision and Pattern Recognition (CVPR'19)}, pages 12569--12577, Long Beach,
  CA, June~16--20 2019.

\bibitem[Krizhevsky et~al.(2012)Krizhevsky, Sutskever, and Hinton]{Krizhev_12a}
A.~Krizhevsky, I.~Sutskever, and G.~Hinton.
\newblock {ImageNet} classification with deep convolutional neural networks.
\newblock In F.~Pereira, C.~J.~C. Burges, L.~Bottou, and K.~Q. Weinberger,
  editors, \emph{Advances in Neural Information Processing Systems (NIPS)},
  volume~25, pages 1106--1114. MIT Press, Cambridge, MA, 2012.

\bibitem[{LeCun} et~al.(1990){LeCun}, Denker, and Solla]{Lecun_90a}
Y.~{LeCun}, J.~S. Denker, and S.~A. Solla.
\newblock Optimal brain damage.
\newblock In D.~S. Touretzky, editor, \emph{Advances in Neural Information
  Processing Systems (NIPS)}, volume~2, pages 598--605. Morgan Kaufmann, San
  Mateo, CA, 1990.

\bibitem[Leng et~al.(2018)Leng, Li, Zhu, and Jin]{Leng_18a}
C.~Leng, H.~Li, S.~Zhu, and R.~Jin.
\newblock Extremely low bit neural network: Squeeze the last bit out with
  {ADMM}.
\newblock In \emph{Proc. of the 32nd AAAI Conference on Artificial Intelligence
  (AAAI 2018)}, pages 3466--3473, New Orleans, LA, Feb.~2--7 2018.

\bibitem[Li et~al.(2016)Li, Zhang, and Liu]{Li_16b}
F.~Li, B.~Zhang, and B.~Liu.
\newblock Ternary weight networks.
\newblock arXiv:1605.04711, Nov.~19 2016.

\bibitem[Li et~al.(2017)Li, Kadav, Durdanovic, and Graf]{Li_17b}
H.~Li, A.~Kadav, I.~Durdanovic, and H.~P. Graf.
\newblock Pruning filters for efficient {ConvNets}.
\newblock In \emph{Proc. of the 5th Int. Conf. Learning Representations (ICLR
  2017)}, Toulon, France, Apr.~24--26 2017.

\bibitem[Li et~al.(2019)Li, Qi, Wang, Ge, Li, Yue, and Sun]{Li_19a}
J.~Li, Q.~Qi, J.~Wang, C.~Ge, Y.~Li, Z.~Yue, and H.~Sun.
\newblock {OICSR}: Out-in-channel sparsity regularization for compact deep
  neural networks.
\newblock In \emph{Proc. of the 2019 IEEE Computer Society Conf. Computer
  Vision and Pattern Recognition (CVPR'19)}, pages 7046--7055, Long Beach, CA,
  June~16--20 2019.

\bibitem[Liu et~al.(2019)Liu, Sun, Zhou, Huang, and Darrell]{Liu_19a}
Z.~Liu, M.~Sun, T.~Zhou, G.~Huang, and T.~Darrell.
\newblock Rethinking the value of network pruning.
\newblock In \emph{Proc. of the 7th Int. Conf. Learning Representations (ICLR
  2019)}, New Orleans, LA, May~6--9 2019.

\bibitem[Nesterov(1983)]{Nester83a}
Y.~Nesterov.
\newblock A method of solving a convex programming problem with convergence
  rate {$\calO(1/k^2)$}.
\newblock \emph{Soviet Math. Dokl.}, 27\penalty0 (2):\penalty0 372--376, 1983.

\bibitem[Nocedal and Wright(2006)]{NocedalWright06a}
J.~Nocedal and S.~J. Wright.
\newblock \emph{Numerical Optimization}.
\newblock Springer Series in Operations Research and Financial Engineering.
  Springer-Verlag, New York, second edition, 2006.

\bibitem[Novikov et~al.(2015)Novikov, Podoprikhin, Osokin, and
  Vetrov]{Novikov_15a}
A.~Novikov, D.~Podoprikhin, A.~Osokin, and D.~P. Vetrov.
\newblock Tensorizing neural networks.
\newblock In C.~Cortes, N.~D. Lawrence, D.~D. Lee, M.~Sugiyama, and R.~Garnett,
  editors, \emph{Advances in Neural Information Processing Systems (NIPS)},
  volume~28, pages 442--450. MIT Press, Cambridge, MA, 2015.

\bibitem[Nowlan and Hinton(1992)]{NowlanHinton92a}
S.~J. Nowlan and G.~E. Hinton.
\newblock Simplifying neural networks by soft weight-sharing.
\newblock \emph{Neural Computation}, 4\penalty0 (4):\penalty0 473--493, July
  1992.

\bibitem[Qu et~al.(2020)Qu, Zhou, Cheng, and Thiele]{Qu_20a}
Z.~Qu, Z.~Zhou, Y.~Cheng, and L.~Thiele.
\newblock Adaptive loss-aware quantization for multi-bit networks.
\newblock In \emph{Proc. of the 2020 IEEE Computer Society Conf. Computer
  Vision and Pattern Recognition (CVPR'20)}, pages 7988--7997, Seattle, WA,
  June~14--19 2020.

\bibitem[Rastegari et~al.(2016)Rastegari, Ordonez, Redmon, and
  Farhadi]{Rasteg_16a}
M.~Rastegari, V.~Ordonez, J.~Redmon, and A.~Farhadi.
\newblock {XNOR}-net: {ImageNet} classification using binary convolutional
  neural networks.
\newblock In B.~Leibe, J.~Matas, N.~Sebe, and M.~Welling, editors, \emph{Proc.
  14th European Conf. Computer Vision (ECCV'16)}, pages 525--542, Amsterdam,
  The Netherlands, Oct.~11--14 2016.

\bibitem[Russakovsky et~al.(2015)Russakovsky, Deng, Su, Krause, Satheesh, Ma,
  Huang, Karpathy, Khosla, Bernstein, Berg, and Fei-Fei]{Russak_15a}
O.~Russakovsky, J.~Deng, H.~Su, J.~Krause, S.~Satheesh, S.~Ma, Z.~Huang,
  A.~Karpathy, A.~Khosla, M.~Bernstein, A.~C. Berg, and L.~Fei-Fei.
\newblock {ImageNet} large scale visual recognition challenge.
\newblock \emph{Int. J. Computer Vision}, 115\penalty0 (3):\penalty0 211--252,
  Dec. 2015.

\bibitem[Simonyan and Zisserman(2015)]{SimonyZisser15a}
K.~Simonyan and A.~Zisserman.
\newblock Very deep convolutional networks for large-scale image recognition.
\newblock In \emph{Proc. of the 3rd Int. Conf. Learning Representations (ICLR
  2015)}, San Diego, CA, May~7--9 2015.

\bibitem[Tseng(2001)]{Tseng01a}
P.~Tseng.
\newblock Convergence of a block coordinate descent method for
  nondifferentiable minimization.
\newblock \emph{J. Optimization Theory and Applications}, 109\penalty0
  (3):\penalty0 475--494, June 2001.

\bibitem[Tung and Mori(2018)]{TungMori18a}
F.~Tung and G.~Mori.
\newblock {CLIP-Q}: Deep network compression learning by in-parallel
  pruning-quantization.
\newblock In \emph{Proc. of the 2018 IEEE Computer Society Conf. Computer
  Vision and Pattern Recognition (CVPR'18)}, pages 7873--7882, Salt Lake City,
  UT, June~18--22 2018.

\bibitem[Wen et~al.(2016)Wen, Wu, Wang, Chen, and Li]{Wen_16a}
W.~Wen, C.~Wu, Y.~Wang, Y.~Chen, and H.~Li.
\newblock Learning structured sparsity in deep neural networks.
\newblock In D.~D. Lee, M.~Sugiyama, U.~von Luxburg, I.~Guyon, and R.~Garnett,
  editors, \emph{Advances in Neural Information Processing Systems (NIPS)},
  volume~29, pages 2074--2082. MIT Press, Cambridge, MA, 2016.

\bibitem[Wen et~al.(2017)Wen, Xu, Wu, Wang, Chen, and Li]{Wen_17a}
W.~Wen, C.~Xu, C.~Wu, Y.~Wang, Y.~Chen, and H.~Li.
\newblock Coordinating filters for faster deep neural networks.
\newblock In \emph{Proc. 16th Int. Conf. Computer Vision (ICCV'17)}, Venice,
  Italy, Dec.~11--18 2017.

\bibitem[Wright(2016)]{Wright16a}
S.~J. Wright.
\newblock Coordinate descent algorithms.
\newblock \emph{Math. Prog.}, 151\penalty0 (1):\penalty0 3--34, June 2016.

\bibitem[Wu et~al.(2016)Wu, Leng, Wang, Hu, and Cheng]{Wu_16a}
J.~Wu, C.~Leng, Y.~Wang, Q.~Hu, and J.~Cheng.
\newblock Quantized convolutional neural networks for mobile devices.
\newblock In \emph{Proc. of the 2016 IEEE Computer Society Conf. Computer
  Vision and Pattern Recognition (CVPR'16)}, pages 4020--4028, Las Vegas, NV,
  June~26 -- July~1 2016.

\bibitem[Xiao et~al.(2019)Xiao, Wang, and Rajasekaran]{Xiao_19a}
X.~Xiao, Z.~Wang, and S.~Rajasekaran.
\newblock {AutoPrune}: Automatic network pruning by regularizing auxiliary
  parameters.
\newblock In H.~Wallach, H.~Larochelle, A.~Beygelzimer, F.~d'Alch{\'e} Buc,
  E.~Fox, and R.~Garnett, editors, \emph{Advances in Neural Information
  Processing Systems (NEURIPS)}, volume~32, pages 13681--13691. MIT Press,
  Cambridge, MA, 2019.

\bibitem[Xu et~al.(2018)Xu, Wang, Zhou, Lin, and Xiong]{Xu_18b}
Y.~Xu, Y.~Wang, A.~Zhou, W.~Lin, and H.~Xiong.
\newblock Deep neural network compression with single and multiple level
  quantization.
\newblock In \emph{Proc. of the 32nd AAAI Conference on Artificial Intelligence
  (AAAI 2018)}, pages 4335--4342, New Orleans, LA, Feb.~2--7 2018.

\bibitem[Xu et~al.(2020)Xu, Li, Zhang, Wen, Wang, Qi, Chen, Lin, and
  Xiong]{Xu_20a}
Y.~Xu, Y.~Li, S.~Zhang, W.~Wen, B.~Wang, Y.~Qi, Y.~Chen, W.~Lin, and H.~Xiong.
\newblock {TRP}: Trained rank pruning for efficient deep neural networks.
\newblock In \emph{Proc. of the 29th Int. Joint Conf. Artificial Intelligence
  (IJCAI'20)}, pages 977--983, Yokohama, Japan, Jan.~21--15 2020.

\bibitem[Yang et~al.(2020)Yang, Gui, Zhu, and Liu]{Yang_20b}
H.~Yang, S.~Gui, Y.~Zhu, and J.~Liu.
\newblock Automatic neural network compression by sparsity-quantization joint
  learning: A constrained optimization-based approach.
\newblock In \emph{Proc. of the 2020 IEEE Computer Society Conf. Computer
  Vision and Pattern Recognition (CVPR'20)}, pages 2175--2185, Seattle, WA,
  June~14--19 2020.

\bibitem[Yang et~al.(2019)Yang, Shen, Xing, Tian, Li, Deng, Huang, and
  Hua]{Yang_19b}
J.~Yang, X.~Shen, J.~Xing, X.~Tian, H.~Li, B.~Deng, J.~Huang, and X.-S. Hua.
\newblock Quantization networks.
\newblock In \emph{Proc. of the 2019 IEEE Computer Society Conf. Computer
  Vision and Pattern Recognition (CVPR'19)}, pages 7308--7316, Long Beach, CA,
  June~16--20 2019.

\bibitem[Ye et~al.(2018)Ye, Lu, Lin, and Wang]{Ye_18a}
J.~Ye, X.~Lu, Z.~Lin, and J.~Wang.
\newblock Rethinking the smaller-norm-less-informative assumption in channel
  pruning of convolution layers.
\newblock In \emph{Proc. of the 6th Int. Conf. Learning Representations (ICLR
  2018)}, Vancouver, Canada, Apr.~30 -- May~3 2018.

\bibitem[Yin et~al.(2018)Yin, Zhang, Lyu, Osher, Qi, and Xin]{Yin_18b}
P.~Yin, S.~Zhang, J.~Lyu, S.~Osher, Y.~Qi, and J.~Xin.
\newblock {BinaryRelax}: A relaxation approach for training deep neural
  networks with quantized weights.
\newblock \emph{SIAM J. Imaging Sciences}, 11\penalty0 (4):\penalty0
  2205--2223, 2018.

\bibitem[Yu et~al.(2018)Yu, Li, Chen, Lai, Morariu, Han, Gao, Lin, and
  Davis]{Yu_18a}
R.~Yu, A.~Li, C.-F. Chen, J.-H. Lai, V.~I. Morariu, X.~Han, M.~Gao, C.-Y. Lin,
  and L.~S. Davis.
\newblock {NISP}: Pruning networks using neuron importance score propagation.
\newblock In \emph{Proc. of the 2018 IEEE Computer Society Conf. Computer
  Vision and Pattern Recognition (CVPR'18)}, pages 9194--9203, Salt Lake City,
  UT, June~18--22 2018.

\bibitem[Yu et~al.(2017)Yu, Liu, Wang, and Tao]{Yu_17a}
X.~Yu, T.~Liu, X.~Wang, and D.~Tao.
\newblock On compressing deep models by low rank and sparse decomposition.
\newblock In \emph{Proc. of the 2017 IEEE Computer Society Conf. Computer
  Vision and Pattern Recognition (CVPR'17)}, pages 67--76, Honolulu, HI,
  July~21--26 2017.

\bibitem[Zhou et~al.(2017)Zhou, Yao, Guo, Xu, and Chen]{Zhou_17a}
A.~Zhou, A.~Yao, Y.~Guo, L.~Xu, and Y.~Chen.
\newblock Incremental network quantization: Towards lossless {CNNs} with
  low-precision weights.
\newblock In \emph{Proc. of the 5th Int. Conf. Learning Representations (ICLR
  2017)}, Toulon, France, Apr.~24--26 2017.

\bibitem[Zhou et~al.(2016)Zhou, Ni, Zhou, Wen, Wu, and Zou]{Zhou_16b}
S.~Zhou, Z.~Ni, X.~Zhou, H.~Wen, Y.~Wu, and Y.~Zou.
\newblock Dorefa-net: Training low bitwidth convolutional neural networks with
  low bitwidth gradients.
\newblock arXiv:1606.06160, July~17 2016.

\bibitem[Zhou and Tao(2011)]{ZhouTao11a}
T.~Zhou and D.~Tao.
\newblock {GoDec}: Randomized low-rank \& sparse matrix decomposition in noisy
  case.
\newblock In L.~Getoor and T.~Scheffer, editors, \emph{Proc. of the 28th Int.
  Conf. Machine Learning (ICML 2011)}, pages 33--40, Bellevue, WA, June~28 --
  July~2 2011.

\bibitem[Zhu et~al.(2017)Zhu, Han, Mao, and Dally]{Zhu_17a}
C.~Zhu, S.~Han, H.~Mao, and W.~J. Dally.
\newblock Trained ternary quantization.
\newblock In \emph{Proc. of the 5th Int. Conf. Learning Representations (ICLR
  2017)}, Toulon, France, Apr.~24--26 2017.

\end{thebibliography}
\end{document}